%% file: iclr2027_conference.tex
\documentclass{article} %
\usepackage{iclr2027_conference,times}

\input{math_commands.tex}

\usepackage{hyperref}
\usepackage{url}
\usepackage{multirow}
\usepackage{placeins}
\usepackage{graphicx}        

\usepackage{amsthm, bbm}
\usepackage{algorithm}
\usepackage{algorithmic}
\usepackage[english]{babel}
\newtheorem{thm}{Theorem}[section]

\title{The Tokens Strike Back: Safety-Aware Denoising with Directional Embedding-based Revision}

\author{%
Amman Yusuf\textsuperscript{1},
Jiwoo Shin\textsuperscript{2},
Ruizhi Deng\textsuperscript{3},
Heitor Rodrigues Guimar\~{a}es\textsuperscript{3}, \\
\textbf{Nikita Dvornik\textsuperscript{3}, Mijung Park\textsuperscript{1}} \\
\textsuperscript{1}Department of Computer Science, University of British Columbia, Canada \\
\textsuperscript{2}Korea Advanced Institute of Science and Technology (KAIST), South Korea \\
\textsuperscript{3}RBC Borealis, Canada \\
\texttt{ammany01@cs.ubc.ca}
}

\usepackage{booktabs,tabularx,array}

\iclrfinalcopy %
\begin{document}

\maketitle

\lhead{Preprint} %

\begin{abstract}
The Safety-Aware Denoiser (SAD) is a recent, training-free, safety-guidance framework for text diffusion models. During inference, SAD steers text samples toward safe regions of the token probability space.
In this paper, we propose  \textit{SADDER (Safety-Aware Denoising with Directional Embedding-based Revision)}, which improves the safety and quality trade-offs of SAD by exploiting a safety-guidance signal in the embedding (or activation) space for each token position, then updating the token positions using the SAD update rule multiplied by the guidance signal. 
This multiplication acts as a gating mechanism: tokens with a high unsafe signal are more likely to be modified during inference following the SAD update rule.  
However, such an update does not automatically preserve the conditional mean, i.e., the denoiser of a fixed target distribution. 
We therefore theoretically derive new position-wise target distributions for which the modified denoiser is exact. Furthermore, we derive conditions under which the position-wise target distributions match a global target distribution.  
We test our \textit{SADDER} method on popular Masked Diffusion Language Models (MDLMs) and Uniform-State Diffusion Models (USDMs) across unsafe continuation and jailbreak settings, and show that SADDER achieves a better balance between safety, generation quality, and inference cost than existing methods. 
\end{abstract}

\section{Introduction}
\input{sections/Introduction}

\section{Background}
\input{sections/Background_v2}

\section{Methods: SADDER}
\input{sections/Methods}
\section{Related Work}
\input{sections/Related_Work}

\section{Experiments}
\input{sections/Experiments}

\section{Summary and Discussion}

\input{sections/Summary_Discussion}

\newpage
\subsection*{AI use statement}

We used generative AI tools, to assist with refining research hypotheses and experimental designs, reviewing mathematical derivations and proof arguments, analyzing and interpreting experimental results, and implementing and debugging experimental code. We also used these tools to draft and revise portions of the manuscript, improve clarity and organization, prepare LaTeX pseudocode, and identify and discuss relevant literature. The models used for text generation and automated safety evaluation within our experiments are described in the experimental protocol.

The authors reviewed and revised the AI-assisted material included in this work. We checked mathematical statements and derivations, verified references against the cited sources, and checked reported numerical claims against experimental outputs. AI-generated code was reviewed and revised by the authors before use in the experiments. The authors made the final methodological decisions and determined the interpretation and presentation of the findings. We take responsibility for the final content of this work, including all text, claims, and artifacts produced with the assistance of generative AI.

\subsection*{Ethics statement}

This work investigates training-free safety guidance for text diffusion models. Our experiments involve toxic text, harmful instructions, and adversarial prompts from existing datasets and benchmarks. These materials serve a defensive evaluation purpose but may contain offensive or distressing content. Analysis of model vulnerabilities could also inform attempts to circumvent safeguards. Any release of associated harmful generations or attack artifacts should therefore consider risks of misuse, exposure to harmful content, and disclosure of sensitive information. Furthermore, safety guidance could be misused to suppress legitimate expression, while overly broad or biased safety criteria may restrict benign content and disproportionately affect particular communities.

Our evaluation relies on benchmark datasets and automated safety judges, which may reflect social biases and fail to capture context-dependent harms. Differences between evaluators further limit the interpretation of reported safety metrics. We report safety alongside generation quality, inference cost, and over-refusal, including settings where the proposed guidance offers limited benefit. Reduced refusal does not necessarily imply improved helpfulness, and lower measured attack success does not establish that responses are safe or usable. Our findings are restricted to the evaluated models, datasets, and sampling configurations; they do not establish safety across untested settings. Deployment would require additional assessment of robustness, fairness, and risks specific to the intended application.

\subsection*{Reproducibility statement}

The main text presents the mathematical formulation of SADDER, and Algorithm~\ref{SADDER} specifies the selective guidance procedure. Appendix~\ref{derivation_equivalence} derives the centered-score equivalence, and Appendices~\ref{proof_point_wise} and~\ref{proof_global} provide proofs of the theoretical results. Appendix~\ref{app:sadder-datasets} describes the datasets, prompt selection, and reference-set construction. Appendix~\ref{app:sadder-hyperparameters} documents shared settings, hyperparameter ranges, and configuration selection, while Appendix~\ref{app:implementation} details the sampling extensions. Appendices~\ref{app:settings} and~\ref{app:jailbreak} report additional experimental settings and results. Appendix~\ref{app:judge-views} compares evaluators and response scopes, and Appendix~\ref{app:xstest-results} describes refusal evaluation and the handling of missing judge labels. Appendix~\ref{app:metric-uncertainty} specifies metric aggregation, seed standard deviations, bootstrap intervals, and paired comparisons. These materials support independent implementation and assessment of the reported results.

\subsubsection*{Acknowledgments}
This work was supported by Mitacs through the Mitacs Accelerate program, in partnership with RBC Borealis (Royal Bank of Canada).

\bibliography{iclr2027_conference}
\bibliographystyle{iclr2027_conference}

\appendix
\section{Appendix}\label{sec:Appendix}

\input{sections/Appendix}

\end{document}

%% file: math_commands.tex
\usepackage{amsmath,amsfonts,bm}
\usepackage{amssymb}

\def\Secref#1{Section~\ref{#1}}

\def\eqref#1{equation~\ref{#1}}

\def\Algref#1{Algorithm~\ref{#1}}

\newcommand{\thmref}[1]{Thm.~\ref{thm:#1}}

\usepackage{amsthm} %
\theoremstyle{remark}
\newtheorem{remark}{Remark}

\def\1{\bm{1}}

\newcommand{\Dat}{\mathcal{D}}
\newcommand{\Em}{\mathbb{E}}

\def\vtheta{{\bm{\theta}}}

\def\vc{{\bm{c}}}

\def\vh{{\bm{h}}}

\def\vm{{\bm{m}}}

\def\vv{{\bm{v}}}

\def\vx{{\bm{x}}}

\DeclareMathAlphabet{\mathsfit}{\encodingdefault}{\sfdefault}{m}{sl}
\SetMathAlphabet{\mathsfit}{bold}{\encodingdefault}{\sfdefault}{bx}{n}

\newcommand{\pdata}{p_{\rm{data}}}

\newcommand{\E}{\mathbb{E}}

\newcommand{\Cov}{\mathrm{Cov}}

%% file: sections/Introduction.tex
Diffusion models have achieved state-of-the-art performance across a wide range of domains, including images~\citep{rombach2022high}, speech~\citep{popov2021grad}, and protein structure generation~\citep{wu2024protein}. In language modeling, however, autoregressive large language models (AR-LLMs) remain the dominant approach, demonstrating strong capabilities in language generation and understanding,
mathematical reasoning~\citep{yang2024qwen2}, and coding~\citep{li2023starcoder}. Despite their success, AR models generate text one token at a time, which inherently limits inference parallelism and makes revising previously generated content less natural. 

Recent advances in Text Diffusion Models (TDMs) offer an alternative paradigm, in which text is generated through an iterative denoising process that can update multiple token positions and progressively refine the generated sequence~\citep{MDLM, LLaDA}. Moreover, at comparable model scales, recent TDMs have also shown competitive performance with AR-LLMs across benchmarks; For instance, Stable-DiffCoder-8B~\citep{fan2026stablediffcoderpushingfrontiercode} vs. Seed-Coder-8B~\citep{seed2025seedcoderletcodemodel} on code generation benchmarks, and LLaDA-8B~\citep{LLaDA} vs. LLaMA3-8B~\citep{grattafiori2024llama} on mathematical reasoning tasks.

Despite these advantages, the safety of TDMs remains relatively underexplored compared to AR-LLMs, particularly with respect to toxic content generation and vulnerability to jailbreak attacks. Traditional safety mechanisms rely on fine-tuning and alignment~\citep{ICLR2024_dd1577af}, post-hoc filtering and rejection sampling~\citep{10.1145/3744746}, or inference-time methods that modify the decoding strategy during next-token prediction~\citep{xu-etal-2024-safedecoding}. However, these approaches either require additional training or assume AR decoding and therefore do not directly generalize to TDMs.

Recent work has begun exploring inference-time mechanisms specifically designed for TDMs, including particle-based steering~\citep{pmlr-v267-singhal25b-fk-steering}, classifier-based guidance~\citep{schiff2025simple}, and safety-oriented remasking approaches~\citep{lee2026adaptivesteeringremaskingsafe}. In particular, the Safety-Aware Denoiser (SAD)~\citep{yusuf26sad} provides a training-free framework to steer the denoising trajectory towards a safe region in the data distribution using only a small set of $500 \sim 1000$ unsafe reference samples. Nevertheless, SAD operates in token probability space, leaving open whether semantically richer model representations could enable more effective, position-specific safety guidance.

In this work, we propose SADDER (\textbf{S}afety-\textbf{A}ware \textbf{D}enoising with \textbf{D}irectional \textbf{E}mbedding-based \textbf{R}evision), a safety-guidance method that augments SAD with position-wise signals extracted from the model's embedding space. 
At each denoising step, we first identify potentially unsafe token positions based on a safety-guidance signal in the embedding (or activation) space. We then use this information in a second step to modulate the SAD update, enabling more precise, position-by-position interventions along the sampling trajectory. 
By selectively modifying fewer positions, SADDER achieves a better balance between safety, generation quality, and inference cost than existing methods, when tested on popular Masked Diffusion Language Models (MDLMs) and Uniform-State Diffusion Models (USDMs) across unsafe continuation and jailbreak settings.

%% file: sections/Background_v2.tex
We first describe two ingredients our method builds on: masked and uniform-state TDMs in \Secref{sec:tdm}, whose samplers differ in whether a committed token can still be changed in the subsequent denoising steps; and SAD in \Secref{sec:sad}, the inference-time guidance framework that our method extends.

\subsection{Text Diffusion Models}\label{sec:tdm}

A TDM corrupts a clean token sequence $\vx_0^{1:L}$ of length $L$ with a fixed forward process and learns to reverse it. Tokens come from a vocabulary $\mathcal V$ of size $K$ and are represented as one-hot vectors, so a distribution over tokens is a vector in the simplex $\Delta^K$. $\mathrm{Cat}(\vx\mid\bm{\pi})$ denotes the categorical distribution with probability vector $\bm{\pi}$. 

We write $\vx_t^l$ for the token at position $l$ and normalized diffusion time $t\in[0,1]$. Let $\alpha_t\in[0,1]$ denote a noise schedule that decreases strictly in $t$ with $\alpha_0\approx1$ and $\alpha_1\approx0$. We discretize the diffusion time interval using a decreasing grid $1=t_0 > t_1 > \cdots > t_T=0$ where $T$ is the number of reverse sampling steps. At iteration $j\in \{0, \dots, T-1\}$, the sampler transitions from $\vx_{t_j}^{1:L}$ to $\vx_{t_{j+1}}^{1:L}$.

\paragraph{Masked diffusion language models (MDLMs).}
MDLMs \citep{MDLM, LLaDA, ye2025dream7bdiffusionlarge} add a mask token with one-hot vector $\mathbb{M}$, so states live in $\{0,1\}^{K+1}$; the network below never predicts the mask, so its outputs stay in $\Delta^K$. In the forward process, each position transitions to the mask independently and, once masked, stays masked:
\begin{align}\label{forward_mdlm}
    q_t(\vx_t^l\mid\vx_0^l) = \mathrm{Cat}\big(\vx_t^l\mid \alpha_t \vx_0^l + (1-\alpha_t)\,\mathbb{M}\big),
\end{align}
i.e., the token is kept with probability $\alpha_t$ and masked with probability $1-\alpha_t$. The reverse process runs from the all-mask sequence back to $t=0$; its transitions are the forward posterior $q(\vx_s^l\mid\vx_t^l,\vx_0^l)$ with the unknown clean token replaced by the prediction of a denoising network $\vx_\vtheta$, which reads the whole noised sequence and outputs a distribution over the clean token at every position,
\begin{align}\label{denoiser_signature}
    \vx_\vtheta : (\mathcal V\cup\{\mathbb{M}\})^L\times[0,1]\longrightarrow(\Delta^K)^L,
\end{align}
with $\vx_{\vtheta,l}(\vx_t^{1:L},t)\in\Delta^K$ the output at position $l$. Given $\vx_t^{1:L}$, the transition to an earlier level $s<t$ factorizes across positions, with
\begin{align}\label{reverse_mdlm}
  p_\vtheta(\vx_s^l\mid\vx_t^{1:L})
  = \begin{cases}
      \delta_{\vx_t^l}(\vx_s^l), & \vx_t^l\neq\mathbb{M},\\
      \mathrm{Cat}\!\Big(\vx_s^l\,\Big|\,\tfrac{(1-\alpha_s)\,\mathbb{M}+(\alpha_s-\alpha_t)\,\vx_{\vtheta,l}(\vx_t^{1:L},t)}{1-\alpha_t}\Big), & \vx_t^l=\mathbb{M}.
    \end{cases}
\end{align}
Under \eqref{reverse_mdlm}, a revealed position is copied, and a masked position is revealed at random with probability $(\alpha_s-\alpha_t)/(1-\alpha_t)$, its token drawn from $\vx_{\vtheta,l}$. The samplers of LLaDA and Dream differ from \eqref{reverse_mdlm} only in how they choose which masked positions to reveal \citep{LLaDA, ye2025dream7bdiffusionlarge}. \emph{Remasking} samplers \citep{wang2025ReMDM} additionally return revealed positions to $\vm$ so that they are predicted again.

\paragraph{Uniform-state diffusion models (USDMs).}
USDMs \citep{sahoo2025duo, sahoo2026scaling, diffusiongemma2026} go a step further: they have no mask token, and any token can change at any step, which \citet{sahoo2026scaling} credit with self-correction and better guided generation. In the forward process, a token is kept with probability $\alpha_t$ and otherwise replaced by one drawn uniformly from the vocabulary,
\begin{align}\label{forward_usdm}
    q_t(\vx_t^l\mid\vx_0^l) = \mathrm{Cat}\big(\vx_t^l\mid \alpha_t \vx_0^l + (1-\alpha_t)\,\tfrac{1}{K}\mathbf 1\big),
\end{align}
with $\mathbf 1$ the all-ones vector. 
The reverse kernel is again the forward posterior with $\vx_0^l$ replaced by the prediction; with $\alpha_{t|s}=\alpha_t/\alpha_s$, it is given by
\begin{align}\label{reverse_usdm}
    & q_{s|t}(\vx_s^l|\vx_t^l, \vx_\vtheta^l) \nonumber \\
    &= \mbox{Cat}\left(\vx_s^l | \frac{K\alpha_t \vx_t^l \odot  \vx_\vtheta^l + (\alpha_{t|s}-\alpha_t)\vx_t^l}{K\alpha_t \langle\vx_t^l, \vx_\vtheta^l\rangle + 1-\alpha_t}
    +\frac{(\alpha_{s}-\alpha_t)\vx_\vtheta^l + (1-\alpha_{t|s})(1-\alpha_s) \mathbf{1}/K}{K\alpha_t \langle\vx_t^l, \vx_\vtheta^l\rangle + 1-\alpha_t}\right),
\end{align}
where $\odot$ denotes the Hadamard product between two vectors and $\vx_\vtheta^l=\vx_{\vtheta,l}(\vx_t^{1:L},t)$ is given by the pretrained model.
Every entry of this vector can be positive, which is what lets any token change at any step.

\subsection{The Safety-Aware Denoiser}\label{sec:sad}

The safety-aware denoiser (SAD) \citep{yusuf26sad} is a training-free guidance method that steers a pretrained TDM away from a \emph{negation set} of sequences that should not be generated. It adapts the safe denoiser of \citet{SafeDenoisers}, later extended to flow matching \citep{kim2026safetyguided}, by replacing the Gaussian corruption kernel with the categorical kernels \eqref{forward_mdlm} and \eqref{forward_usdm}.

\paragraph{Safe Denoiser.}
Assume the training data $\Dat$ split into safe and unsafe sets, $\Dat=\Dat_{\text{safe}}\cup\Dat_{\text{unsafe}}$ with $|\Dat_{\text{safe}}|\gg|\Dat_{\text{unsafe}}|$. 
Restricting $p_{\text{data}}$ to each set gives $p_{\text{safe}}$ and $p_{\text{unsafe}}$, and running the forward process from $p_{\text{data}}$, $p_{\text{safe}}$, and $p_{\text{unsafe}}$ defines the data, safe, and unsafe denoisers $\mathbb{E}_{\Dat}[\vx_0\mid\vx_t]$, $\mathbb{E}_{\Dat_{\text{safe}}}[\vx_0\mid\vx_t]$, and $\mathbb{E}_{\Dat_{\text{unsafe}}}[\vx_0\mid\vx_t]$, which are related as follows.
\begin{thm}[Theorem 3.2 in \citet{SafeDenoisers}]
\label{thm:safe-denoiser}
There exists a nonnegative weight $\beta^\ast(\vx_t)$, monotone in the posterior likelihood that $\vx_t$ originates from the unsafe set, such that
\begin{align}\label{safe-denoiser}
    \mathbb{E}_{\Dat_{\text{safe}}}[\vx_0\mid\vx_t]
    = \mathbb{E}_{\Dat}[\vx_0\mid\vx_t] + \beta^\ast(\vx_t)\,\big(\mathbb{E}_{\Dat}[\vx_0\mid\vx_t]-\mathbb{E}_{\Dat_{\text{unsafe}}}[\vx_0\mid\vx_t]\big).
\end{align}
\end{thm}
\thmref{safe-denoiser} states that by subtracting ``unsafe'' components from the data denoiser, with $\beta^\ast$ adapting to how unsafe the current state is, we can provably generate samples from the safe denoiser. A pretrained model provides the data denoiser, so only the unsafe denoiser and $\beta^\ast$ remain.

\paragraph{SAD.}
For a TDM, the data denoiser is simply the pre-trained model's output, $\mathbb{E}_{\Dat}[\vx_0^l\mid\vx_t^{1:L}]=\vx_{\vtheta,l}(\vx_t^{1:L},t)$, since the mean of a categorical variable is its probability vector; the relation is exact for the ideal denoiser, which $\vx_\vtheta$ approximates. The unsafe denoiser is approximated by Monte Carlo over a negation set, $\Dat_{\text{unsafe}}=\{\vx_0^{(n)}\}_{n=1}^N$:
$\hat{\mathbb{E}}_{\Dat_{\text{unsafe}}}[\vx_0\mid\vx_t]
    = \sum_{n=1}^{N}\vx_0^{(n)}\,\frac{q_t(\vx_t\mid\vx_0^{(n)})}{\sum_{m=1}^{N} q_t(\vx_t\mid\vx_0^{(m)})},$
where $q_t(\vx_t\mid\vx_0^{(n)})=\prod_{l=1}^L q_t(\vx_t^l\mid\vx_0^{(n),l})$ is the sequence kernel from \eqref{forward_mdlm} or \eqref{forward_usdm}, a product because positions are noised independently. 
The precise definition of the weight is $\beta^\ast(\vx_t)=Z_{\text{unsafe}}\,p_{\text{unsafe},t}(\vx_t)/\big(Z_{\text{safe}}\,p_{\text{safe},t}(\vx_t)\big)$ \citep[Appendix~A]{SafeDenoisers}, where $p_{\text{safe},t}$ and $p_{\text{unsafe},t}$ are the marginal distributions of the forward process at level $t$ started from the safe and the unsafe distribution, and $Z_{\text{safe}}$ and $Z_{\text{unsafe}}$ are the normalizing constants of the two restricted densities. 

Unfortunately, directly computing $\beta^\ast(\vx_t)$ is \textit{intractable}: the denominator requires $q_t(\vx_t\mid\vx_0)$ over all safe data, roughly most of the training set, at every sampling step. 
SAD approximates
\begin{align}
\beta^\ast(\vx_t) &\approx  \eta\,\beta(\vx_t), \nonumber 
\end{align} using an unbiased estimate of $\beta(\vx_t)$ and a hyperparameter $\eta$ to replace the remaining terms of $\beta^\ast$. 
\begin{align}\label{beta_hat}
    \beta(\vx_t)& =\int p_{\text{unsafe}}(\vx_0)\,q_t(\vx_t\mid\vx_0)\,\mathrm{d} \vx_0\approx\frac{1}{N}\sum_{n=1}^{N} q_t(\vx_t\mid\vx_0^{(n)}). 
\end{align}
Because SAD approximates $\beta^\ast$ this way, the SAD implementation offers \textit{no theoretical safety guarantees} on generated data, unlike what the theory in \thmref{safe-denoiser} promises.

Prior work shows that such guidance must be strong early in sampling and fade later, or it degrades sample quality \citep{SafeDenoisers, kim2026safetyguided, kirchhof2024sparse}; SAD therefore keeps $\eta$ at a nonzero constant over a window $C$ of critical timesteps at the start of the reverse process and sets $\eta=0$ outside it.

%% file: sections/Methods.tex
SAD provides safety guidance in the token probability space. We can enhance this guidance by selecting a guidance signal in the embedding (or activation) space for each token position, then modifying the selected token positions using the SAD update rule. In what follows, we introduce SADDER (Safety-Aware Denoising with Directional Embedding-based Revision), consisting of two phases.

\subsection{First Phase: \textbf{Selecting Risky Positions.}}
The goal of the first phase is to create a gating mechanism which tells us which token positions need to be modified. 
Inspired by \citep{lee2026adaptivesteeringremaskingsafe}, we propose accessing TDM's internal embedding space to obtain this information. 

For each sample $i$, we denote ${\vc}_i$ as a given prompt and $\vx_i^{\mathrm{harm}}$, $\vx_i^{\mathrm{safe}}$ as two corresponding harmful and safe responses for $\vc_i$, respectively. 
We denote $P_i$ = $|\vc_i|$ as the prompt length and $R_i^{\mathrm{harm}} = \min(|\vx_i^{harm}|, M)$ the response length after truncation of $M$ tokens. Similarly, $R_i^{\mathrm{safe}} = \min(|\vx_i^{safe}|, M)$. 
Let $\vh_{k,l}(\cdot) \in \mathbb{R}^d$ be the layer-$k$ hidden representation at position $l$. $\vh_{k,l}^{(t)}(\cdot)$ denotes the layer-$k$'s hidden state for position $l$ at timestep $t$. We extract hidden states and compute the average representations over response tokens:
\begin{align}
\bar \vh_{i,k}^{\mathrm{harm}} = \frac1{R_i^{\mathrm{harm}}}\sum_{l=P_i+1}^{P_i+R_i^{\mathrm{harm}}} \vh_{k,l}(\vc_i \oplus \vx_i^{\mathrm{harm}} ),
\quad \bar \vh_{i,k}^{\mathrm{safe}} = \frac1{R_i^{\mathrm{safe}}}\sum_{l=P_i+1}^{P_i+R_i^{\mathrm{safe}}} \vh_{k,l}(\vc_i \oplus \vx_i^{\mathrm{safe}} )
\end{align} We arrive at the (normalized) layer-wise difference in the semantic representations of harmful and safe responses, $\hat\vv_k$, by averaging over the examples $i$:
\begin{align}
   \widehat \vv_{k} &=
\frac{\bar \vh_{k}^{\mathrm{harm}}-\bar \vh_{k}^{\mathrm{safe}}}
{\|\bar \vh_{k}^{\mathrm{harm}}-\bar \vh_{k}^{\mathrm{safe}}\|_2},   
\end{align} where $   \bar \vh_{k}^{\mathrm{harm}} = \frac{1}{N_{\mathrm{harm}}}\sum_i \bar \vh_{i,k}^{\mathrm{harm}}$, and  
   $\bar \vh_{k}^{\mathrm{safe}} = \frac{1}{N_{\mathrm{safe}}}\sum_i \bar \vh_{i,k}^{\mathrm{safe}}$.
Intuitively, we want to intervene on the token location only when the hidden state at time $t$ is \textit{closer to} the average harmful representations than to the average safe representations:  
\begin{align}\label{eq:harm_vs_safe}
  \| \vh_{k,l}^{(t)} - \bar \vh_{k}^{\mathrm{harm}}\|_2^2 <  \| \vh_{k,l}^{(t)} - \bar \vh_{k}^{\mathrm{safe}}\|_2^2.
\end{align} 
Using a midpoint $\vm_k = \frac{\bar \vh_{k}^{\mathrm{harm}}+\bar \vh_{k}^{\mathrm{safe}}}{2}$, a simple linear algebraic manipulation gives us the following equivalence:
\begin{align}
   \| \vh_{k,l}^{(t)} - \bar \vh_{k}^{\mathrm{harm}}\|_2^2 <  \| \vh_{k,l}^{(t)} - \bar \vh_{k}^{\mathrm{safe}}\|_2^2 \; \Longleftrightarrow \; \left\langle \vh_{k,l}^{(t)} - \vm_k, \hat\vv_k \right\rangle >0 
\end{align} See \Secref{derivation_equivalence} for the derivation. Unlike \citep{lee2026adaptivesteeringremaskingsafe}, 
we use the \textit{centred} score to decide whether this particular position will lead to harmful generation:
\begin{equation}
s_{k,l}^{(t)}:=\left\langle \vh_{k,l}^{(t)} - \vm_k, \; \hat\vv_k \right\rangle. 
\label{eq:centered-risk-score}
\end{equation}
\textit{Geometric meaning of \eqref{eq:centered-risk-score}}: The set of $\vh_{k,l}^{(t)}$ that satisfies $\left\langle \vh_{k,l}^{(t)} - \vm_k, \hat\vv_k \right\rangle =0$ is the hyperplane, in which  every vector $\vh_{k,l}^{(t)} - \vm_k$ is orthogonal to the direction  $\hat\vv_k$. 
Any $\vh_{k,l}^{(t)}$ that satisfies $\left\langle \vh_{k,l}^{(t)} - \vm_k, \hat\vv_k \right\rangle >0$ lies on the side of this hyperplane containing $\bar \vh_{k}^{\mathrm{harm}}$. On the contrary, any $\vh_{k,l}^{(t)}$ that satisfies $\left\langle \vh_{k,l}^{(t)} - \vm_k, \hat\vv_k \right\rangle <0$ lies on the side of this hyperplane containing $\bar \vh_{k}^{\mathrm{safe}}$. 

The tokens that have positive scores, through the following \textit{gating} mechanism:
\begin{equation}
a_{k,l}^{(t)} \leftarrow \mathbbm{1}_{s_{k,l}^{(t)}>0},
\label{eq:centered-risk-score-mdlm}
\end{equation} where $\mathbbm{1}$ is an indicator function that returns the binary\footnote{One can also use the softened score for more refined guidance across positions.
However, this approach did not empirically outperform simple gating in \eqref{eq:centered-risk-score-mdlm}. See Appendix~\ref{app:uniform-soft} for comparisoin.} gating values: $a_{k,l}^{(t)} \in \{0, 1\}$.

\begin{algorithm}[t]
\caption{SADDER: selective SAD guidance.}\label{SADDER}
\begin{algorithmic}[1]
\STATE \textbf{Input:} denoiser $p_\theta$; unsafe dataset $\Dat_{\mathrm{unsafe}}$;
  $\widehat\vv_k,\vm_k$; initial state $\vx_{t_0}$; editable positions $\mathcal G$;
  time steps $t_0>\cdots>t_T=0$; intervention steps $C$; strength $\eta$.
\FOR{$j=0,\ldots,T-1$}
  \STATE Set $t\leftarrow t_j$. Obtain $p_{\Dat,t,\cdot}$ and $\vh_{k,\cdot}^{(t)}$
    from one model evaluation on $\vx_t$.
  \STATE Initialize $\widetilde p_{0\mid t,l}\leftarrow p_{\Dat,t,l}$ and
    $a_{k,l}^{(t)}\leftarrow0$ for all $l$.
  \IF{$j\in C$}
    \STATE \textbf{if} MDLMs \textbf{then} $\mathcal M_t\leftarrow \{l\in\mathcal G:\vx_t^l=\mathbb M\}$
    \STATE \textbf{if} USDMs \textbf{then} $\mathcal M_t\leftarrow\{1,\ldots,L\}$ .
        \STATE Compute $s_{k,l}^{(t)}\leftarrow
      \langle\vh_{k,l}^{(t)}-\vm_k,\widehat\vv_k\rangle$ for $l\in\mathcal M_t$.
          \STATE \textbf{Select Risky Positions:}
      $a_{k,l}^{(t)}\leftarrow\mathbbm1\{s_{k,l}^{(t)}>0\}$ for $l\in\mathcal M_t$.
    \IF{at least one gate equals $1$}
      \STATE Estimate $\Em_{\Dat_{\mathrm{unsafe}}}[\vx^l_0|\vx^{1:L}_t] $ and
        $\beta(\vx_t)$ on the current state
      \STATE \textbf{SAD On Selected Positions:} Compute Sadder Denoiser in \eqref{point-wise-sadder-guidance}.
    \ENDIF
  \ENDIF
  \STATE $\vx_{t_{j+1}}\leftarrow\mathrm{Solver}(\vx_t,t_j,t_{j+1},\mbox{Sadder Denoiser};\mathcal G)$.
\ENDFOR
\STATE \textbf{Return:} the generated continuation of $\vx_{t_T}$.
\end{algorithmic}
\end{algorithm}

\subsection{Second Phase: SAD on Selected Positions}

In the first phase, we select which position is aligned with the harmful direction in the embedding space. In the second phase, we modify the sampling trajectory of the selected ones using the SAD denoiser formulation, given by 
\begin{align}\label{point-wise-sadder-guidance}
   \mbox{Sadder Denoiser}:= \prod_{l=1}^L \left[
    \Em_{\Dat}[\vx^l_0|\vx^{1:L}_t] + \eta\,\beta(\vx_t) a_{k,l}^{(t)} \left(\Em_{\Dat}[\vx^l_0|\vx^{1:L}_t]-\Em_{\Dat_{\mathrm{unsafe}}}[\vx^l_0|\vx^{1:L}_t] \right)\right]  
\end{align}
where $\beta(\vx_t)$ is defined in \eqref{beta_hat} and $\eta$ is a controllable hyperparameter that changes the strength of the guidance as in SAD. Similar to SAD, the data denoiser at position $l$ is $\Em_{\Dat}[\vx^l_0|\vx^{1:L}_t]$
and the unsafe denoiser at position $l$ is approximated by 
\begin{align}\label{point-wise-unsafe-denoiser}
    \Em_{\Dat_{\mathrm{unsafe}}}[\vx^l_0|\vx^{1:L}_t] \approx \sum_{\vx_n \in \Dat_{\mathrm{unsafe}}} \vx^l_n \; \frac{q_t(\vx_t|\vx^l_n)}{\sum_{\vx_m \in \Dat_{\mathrm{unsafe}}} q_t(\vx_t|\vx^l_m)}.
\end{align}
\Algref{SADDER} summarizes our algorithm.

\subsection{Theoretical analyses of SADDER}

In Safe Denoiser, \eqref{safe-denoiser} provides an identity of conditional expectations under a
particular safe/unsafe partition, while, in the implementation of SAD, this quantity is approximated by a well-tuned hyperparameter $\eta$. 
A freely chosen position-dependent gate at $l$ in SADDER also changes that identity, as in SAD.  
In the following theorem, we formally derive \textit{under what condition} this new target denoiser is a position-wise safe denoiser.

\begin{thm}[\textbf{Characterization of an exact position-wise safe denoiser}]
\label{thm:sadder-pointwise-safe}
Given a data denoiser from a pre-trained diffusion model, an unsafe denoiser in \eqref{point-wise-unsafe-denoiser}, and the intervention strength: $\eta \beta(\vx_t) a_{k,l}^{(t)}$, 
the the position-$l$ factor of the SADDER Denoiser in \eqref{point-wise-sadder-guidance} is exactly the conditional mean
under the position-wise safe target distribution
$\Dat_{\mathrm{safe}}^{(l)}$:
\begin{equation}
\boxed{
\E_{\Dat_{\mathrm{safe}}^{(l)}}\!\left[
        \vx_0^l\mid \vx_t^{1:L}
    \right]
    =
\Em_{\Dat}[\vx^l_0|\vx^{1:L}_t] + \eta  \beta(\vx_t) a_{k,l}^{(t)} \left(\Em_{\Dat}[\vx^l_0|\vx^{1:L}_t]-\Em_{\Dat_{\mathrm{unsafe}}}[\vx^l_0|\vx^{1:L}_t] \right)
}
\label{pointwise-exact-safe-denoiser}
\end{equation}
if the activation strength is calibrated to the posterior unsafe odds according to
    \begin{equation}
        \eta \beta(\vx_t) a_{k,l}^{(t)}
        =
        \frac{
            \pi_{t,l}(\vx_t^{1:L})
        }{
            1-\pi_{t,l}(\vx_t^{1:L})
        },
        \label{odds-calibration}
    \end{equation}
where  the posterior probability of the
position-$l$ being unsafe is denoted by
$\pi_{t,l}(\vx_t^{1:L})
    :=
    \E_\Dat\!\left[
        \psi_l(\vx_0^{1:L})\mid \vx_t^{1:L}
    \right]$, and the unsafety is measured by an oracle function denoted by $\psi_l:\mathcal{V}^L\longrightarrow[0,1]$, satisfying $\psi_l(\vx_0^{1:L})+(1-\psi_l(\vx_0^{1:L}))=1$ for any clean sequence $\vx_0^{1:L}$.
\end{thm}
See \Secref{proof_point_wise} for proof. 
\begin{remark}
\thmref{sadder-pointwise-safe} characterizes our algorithm, SADDER, in terms of a position-wise safe-denoiser guarantee, rather than asserting that the exact membership function $\psi$ is available at inference.
However, it identifies the quantity that the activation gate should approximate: the posterior expected unsafe membership of clean sequences compatible with the current noised state. In practice, the centred activation score is a surrogate for this quantity without formal guarantees, like $\beta^*(\vx_t)$ was approximated by $\eta \beta(\vx_t)$ in SAD; the resulting SADDER denoiser is also \textit{approximate}. 
\end{remark}

The position-specific target distribution $p_{\Dat_{\mathrm{safe}}^{(l)}}$ for each position $l$, stated in \thmref{sadder-pointwise-safe}, may differ for each $l$: $p_{\Dat_{\mathrm{safe}}^{(1)}}, \cdots, p_{\Dat_{\mathrm{safe}}^{(L)}}$
Hence, 
 a natural next question is whether 
 there exists one clean-sequence distribution $p_{\Dat_{\mathrm{safe}}^\mathrm{global}}(\vx_0^{1:L})$, such that simultaneously for every $l$, 
it satisfies $\E_{\Dat_{\mathrm{safe}}}^{global}(\vx_0^{l}|\vx_t^{1:L}) =  \E_{\Dat_{\mathrm{safe}}^{(l)}}(\vx_{0}^l|\vx_t^{1:L})$. 
\thmref{sadder-global-safe} states this formally.

\input{sections/Algorithm}

%% file: sections/Related_Work.tex
\paragraph{Activation steering and remasking.}
Contrastive Activation Addition~\citep{rimsky-etal-2024-steering} steers language models using directions extracted from paired examples; ILRR~\citep{avrahami2026ilrrinferencetimesteeringmethod} guides masked diffusion by aligning activations with a reference sequence.
For safety, DiffuGuard~\citep{diffuguard} combines stochastic remasking with block-level audit and repair.
Most closely related, \citet{lee2026adaptivesteeringremaskingsafe} construct paired samples $(x_i,y_i^{\mathrm{harm}},y_i^{\mathrm{safe}})$, using harmful model responses to WildJailbreak prompts~\citep{NEURIPS2024_wildjailbreak} and paraphrased refusals.
They estimate a \emph{Contrastive Safety Direction} (CSD), $v_\ell$, from the difference between the mean harmful and safe response representations.
In \emph{Phase 1}, they modify hidden states at masked and prompt positions during the first fraction $\rho$ of the denoising steps:
$
\mathcal{S}(h_{\ell,j}^{(t)})
= h_{\ell,j}^{(t)}
- \beta\,\langle h_{\ell,j}^{(t)},\hat{v}_\ell\rangle\hat{v}_\ell,$
where $\hat{v}_\ell$ is the normalized CSD and $\beta$ controls the intervention strength.
In \emph{Phase 2}, they refine the generated sequence by selecting positions according to
$
m_j^{(k)}
= \mathbbm{1}\!\left\{\langle h_{\ell,j}^{(k)},\hat{v}_\ell\rangle>\theta\right\},$
then remasking and regenerating those positions for up to $K$ refinement iterations, stopping early if no positions exceed the threshold.
SADDER uses a similar contrastive direction, but centers the activation score at the midpoint of the harmful and safe representations.
This places the zero-score boundary at their Euclidean bisector.
The score gates SAD's correction in token probability space rather than directly modifying hidden states; in masked models, it also ranks revealed positions for budgeted remasking within the denoising loop.

\paragraph{Sampling-based baselines.}
Best-of-$N$ sampling generates independent candidate completions and selects the one with the highest reward.
FK steering~\citep{pmlr-v267-singhal25b-fk-steering} instead maintains multiple diffusion trajectories and resamples them at intermediate steps using reward-derived potentials.
Both provide inference-time control without retraining the generator.
SADDER modifies token predictions within each trajectory, using unsafe references and activation scores rather than ranking candidate sequences.
Comparisons with Best-of-$N$ and FK steering assess whether this local intervention offers a better safety--quality tradeoff at a given inference budget.

\paragraph{Diffusion-specific jailbreaks.}
DIJA~\citep{wen2025devilmaskemergentsafety} constructs interleaved text--mask prompts that exploit bidirectional context and parallel decoding to elicit harmful completions.
PAD~\citep{zhang2025jailbreakinglargelanguagediffusion-PAD} adapts affirmative-response priming to parallel generation through a multi-point attention attack.
These attacks expose vulnerabilities associated with diffusion's generation process and motivate evaluating safety guidance on adversarially structured inputs as well as ordinary unsafe continuation.
We include all in our evaluation of selective guidance and token revision.

%% file: sections/Experiments.tex
We ask whether selective guidance improves the safety--quality tradeoff of SAD, and whether its benefits transfer from unsafe continuation to adversarial instructions. Following \citet{yusuf26sad}, we evaluate these tasks separately. The main continuation results distinguish SADDER from SADDER with a full-SAD warm-up. Other variants, calibration sweeps, and the extended algorithm are reported in Appendix~\ref{app:variants}.

\subsection{Unsafe continuation on \textbf{Masked Diffusion Language Models} (MDLMs)}
\label{sec:continuation}

\paragraph{Setting and metrics.}
We evaluate LLaDA-8B-Base~\cite{LLaDA} and Dream-7B-Base~\cite{ye2025dream7bdiffusionlarge} on 600 unsafe prompts per dataset from RealToxicityPrompts (RTP)~\cite{realtoxicityprompts} and BeaverTails (BT)~\cite{ji2023beavertails}, generating 32 tokens in 32 steps. Baselines are unmodified sampling, full SAD, adapted CSD$^\dagger$~\citep{lee2026adaptivesteeringremaskingsafe}, best-of-eight (BoN-8), and eight-particle FK steering~\citep{pmlr-v267-singhal25b-fk-steering}. \emph{SADDER + warm-up} applies full SAD on steps 0--7 and selective guidance on steps 8--15. Evaluation prompts were held out from development; later hyperparameter selection uses a separate 300-prompt development set (Appendix~\ref{app:settings}).

Safety is the Llama-Guard-3-8B~\cite{llamaguard3_modelcard} unsafe rate. \emph{Harmful PPL} is GPT-2-large~\citep{radford2019language} corpus perplexity over all outputs to harmful prompts, including safe outputs. \emph{Benign BERT} (BERTScore F1 $\times100$) and \emph{Benign PPL} use 600 benign RTP prompt--continuation pairs, retaining each method's safety-selected settings and reference set. NFE counts generator forward evaluations, excluding reward scoring and one-time reference encoding.

\input{tables/main_llada.tex}

\paragraph{Safety--quality tradeoffs.}
On LLaDA/RTP, SADDER lowers the unsafe rate from 38.6\% unmodified to 29.4\%, versus 32.4\% for full SAD; warm-up reaches 31.2\%. On BeaverTails, SADDER reaches 29.9\%, full SAD 27.1\%, and warm-up 24.4\%. Across all four model/dataset settings, SADDER achieves higher benign BERTScore and lower benign PPL than full SAD. SADDER also achieves higher benign BERTScore than warm-up throughout, although both variants remain below unmodified sampling on this metric (Tables~\ref{tab:llada} and~\ref{tab:dream}).
Harmful-prompt PPL reveals a different tradeoff: on LLaDA/RTP, CSD reaches a 22.4\% unsafe rate but PPL 225, versus 42.1 for SADDER and 31.6 for warm-up without remasking. CSD's benign PPL of 21.6 does not expose this deterioration on harmful prompts. Gains also depend on the model: on Dream, full SAD at $\eta=10{,}000$ achieves lower unsafe rates than either SADDER variant on both datasets. BoN-8 further reduces unsafe rates at a cost of 256 generator evaluations plus reward scoring.

\subsection{Unsafe Continuation on \textbf{Uniform-State Diffusion Models}}
\label{sec:usdm-results}
\paragraph{Setting and metrics.}
We evaluate Duo~\citep{sahoo2025duo} and UDLM~\citep{schiff2025simple} on 300 unsafe prompt records from RealToxicityPrompts (RTP).
Both models generate at a context length of 128 in 128 denoising steps.
Unsafe references and contrastive directions follow Section~\ref{sec:continuation}, with $k \in \{6,9,12\}$.
Baselines are unmodified sampling, full SAD, CSD$^\dagger$, and best-of-$N$ for $N \in \{3,5,8\}$. Full SAD and SADDER sweep $\eta \in \{3{\times}10^4, 5{\times}10^4, 10^5, 1.5{\times}10^5\}$.
In masked models, guidance must reach a position before it is revealed. Uniform-state samplers have no such constraint, so we gate every position at every step, using Equation~\ref{eq:centered-risk-score-mdlm} for SADDER and no gate for full SAD.

Safety is the Llama-Guard-3-8B unsafe rate. Benign BERT uses 300 separate benign RTP prompt–continuation pairs. All entries are means over five generation seeds, and Appendix~\ref{app:uniform} reports exact values with seed standard deviations.

\begin{figure}[t]
\centering
\includegraphics[width=0.8\textwidth]{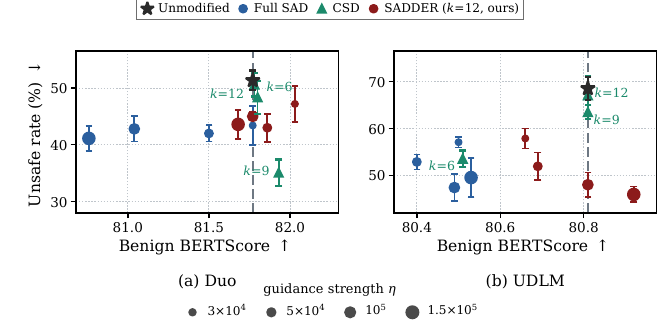}
\caption{Safety--quality tradeoff at equal NFE on Duo and UDLM. Best-of-$N$ is reported in Appendix~\ref{app:uniform}. Larger markers denote larger guidance strength $\eta$. The dashed vertical line marks the unmodified model's benign BERTScore, and error bars are one standard deviation across seeds.}
\label{fig:uniform}
\end{figure}

\paragraph{Safety--quality tradeoffs.}
Figure~\ref{fig:uniform} reports the tradeoff.
On UDLM, SADDER reaches a 45.9\% unsafe rate, against 68.5\% without steering and 47.4\% for the lowest unsafe rate among full SAD settings.
SADDER's benign BERTScore is 80.92, against 80.49 for full SAD and 80.81 unmodified, so it is safer than full SAD without any loss in reference similarity.
On Duo, full SAD reaches lower unsafe rates than SADDER at every $\eta$, from 43.4\% versus 47.2\% at $\eta=3{\times}10^4$ to 41.1\% versus 43.6\% at $\eta=1.5{\times}10^5$.
The same holds at matched benign BERTScore: full SAD reaches 43.4\% at 81.77, where SADDER reaches 45.0\%.
What SADDER does buy is quality, on both models: at every $\eta$ its benign BERTScore is higher than full SAD's. On Duo the gap is widest at the end of the sweep, 81.68 against 80.76.

SADDER, full SAD, and CSD$^\dagger$ use the same number of generator
evaluations, whereas best-of-$N$ uses $N$ times as many. Neither baseline is consistent across the two models. On Duo, best-of-8 reaches 33.1\%, the lowest unsafe rate we observe, yet on UDLM best-of-$N$ does not improve on unmodified sampling at any budget. CSD$^\dagger$ reaches 35.1\% on Duo at $k=9$ but 63.5\% on UDLM at the same layer, and it is sensitive to the layer choice: across $k$, its unsafe rate on Duo spans 15.5 points, whereas SADDER's varies by at most 2.2 points at fixed $\eta$.

\subsection{Jailbreak robustness on instruction-tuned models}
\label{sec:jailbreak}

\paragraph{Setting.}
We evaluate DIJA and PAD on LLaDA-8B-Instruct and Dream-7B-Instruct using 100 held-out HarmBench behaviors, 128 sampling steps, and WildJailbreak~\cite{NEURIPS2024_wildjailbreak} unsafe references. CSD and SADDER share the unsafe-response set. DiffuGuard~\citep{diffuguard} is evaluated with \emph{full} repair, which can modify fixed attack text, and \emph{gen} repair, restricted to initially masked positions; their audit thresholds also differ (Appendices~\ref{app:jailbreak} and~\ref{app:sadder-datasets}).

Table~\ref{tab:jailbreak} reports full-response HarmBench (HB) ASR, LlamaGuard (LG) unsafe rate, harmful PPL, and NFE. XSTest over-refusal uses WildGuard (WG)~\citep{wildguard2024} judgments on 250 safe, unmodified prompts across three seeds, with the DIJA guidance configuration. WG rates use valid labels (non-empty) within each seed, then average across seeds. Appendices~\ref{app:judge-views} and~\ref{app:xstest-results} report generated-span judgments, valid-score counts, paired comparisons, and calibration-prompt reuse.

\input{tables/jailbreak_main}

\paragraph{Attack and evaluator dependence.}
Under DIJA, SADDER and warm-up leave HB ASR near unmodified sampling, while LG flags roughly four-fifths of responses as unsafe. Large HB reductions from CSD and full-block DiffuGuard do not yield comparable LG reductions: Dream's full-block DiffuGuard reaches 33.3\% HB ASR but 84.0\% LG unsafe rate and PPL 537. Both judges assess complete responses, so this discrepancy cannot be attributed solely to scaffold inclusion. Appendix~\ref{app:judge-views} separates evaluator choice from the influence of fixed attack text by comparing full-response and generated-span judgments on the same  outputs.  SADDER retains generator cost and incurs a smaller PPL increase, but offers limited DIJA robustness.

PAD gives a more favorable result: Dream warm-up reduces HB ASR by 9.0 points (paired 95\% CI: $[-15.7,-2.3]$), lowers LG unsafe rate from 64.0\% to 55.0\%, and increases PPL from 8.53 to 10.40. This supports an attack-specific tradeoff rather than general jailbreak robustness. Appendix~\ref{app:jailbreak} reports component, strength, and timing studies.

\paragraph{Over-refusal.}
Dream warm-up reduces WG refusal on safe XSTest prompts by 35.4 points on jointly scored outputs (95\% CI: $[-41.6,-29.3]$), but raises WG harmfulness on unsafe contrasts from 0.0\% to 6.9\% among scored responses. Fewer refusals therefore do not establish greater helpfulness at unchanged safety.

%% file: tables/main_llada.tex
\begin{table*}[t]
\centering\fontsize{8}{9.6}\selectfont
\setlength{\tabcolsep}{0.5pt}\renewcommand{\arraystretch}{1.12}
\caption{LLaDA-Base safety, generation quality, and cost. Entries use $A_b$, where $A$ is the mean and $b$ is one sample standard deviation across three generation seeds. For percentages, SDs are in percentage points. NFE shows average cost.}
\label{tab:llada}
\begin{tabular*}{\textwidth}{@{\extracolsep{\fill}}l@{\hspace{3pt}}rrrrr@{\hspace{5pt}}rrrrr@{}}
\toprule
\multirow{2}{*}{Method} & \multicolumn{5}{c}{\textbf{RTP}} & \multicolumn{5}{c}{\textbf{BeaverTails}} \\
\cmidrule(lr){2-6}\cmidrule(lr){7-11}
 & \shortstack{Unsafe\\\% $\downarrow$} & \shortstack{Harmful\\PPL $\downarrow$} & \shortstack{Benign\\PPL $\downarrow$} & \shortstack{Benign\\BERT $\uparrow$} & NFE & \shortstack{Unsafe\\\% $\downarrow$} & \shortstack{Harmful\\PPL $\downarrow$} & \shortstack{Benign\\PPL $\downarrow$} & \shortstack{Benign\\BERT $\uparrow$} & NFE \\
\midrule
Unmodified & $38.6_{1.3}$ & $21.1_{0.15}$ & $17.5_{0.35}$ & $84.65_{0.04}$ & 32.0 & $38.6_{1.5}$ & $20.6_{0.62}$ & $17.5_{0.35}$ & $84.65_{0.04}$ & 32.0 \\
Full SAD & $32.4_{1.6}$ & $29.3_{0.52}$ & $30.0_{0.75}$ & $83.92_{0.01}$ & 32.0 & $27.1_{1.2}$ & $35.1_{0.34}$ & $47.1_{1.1}$ & $83.38_{0.01}$ & 32.0 \\
SADDER & $29.4_{1.7}$ & $42.1_{1.4}$ & $18.4_{0.26}$ & $84.59_{0.06}$ & 32.0 & $29.9_{0.9}$ & $34.6_{0.54}$ & $27.8_{0.86}$ & $84.20_{0.03}$ & 32.0 \\
SADDER + warm-up & $31.2_{0.7}$ & $31.6_{2.1}$ & $44.3_{0.75}$ & $83.01_{0.01}$ & 32.0 & $24.4_{0.9}$ & $36.9_{0.32}$ & $44.1_{0.14}$ & $83.18_{0.03}$ & 32.0 \\
CSD$^\dagger$ & $22.4_{0.4}$ & $225_{7.0}$ & $21.6_{0.26}$ & $84.59_{0.04}$ & 61.9 & $15.6_{0.5}$ & $487_{36}$ & $943_{40}$ & $82.61_{0.03}$ & 55.0 \\
BoN-8 & $19.0_{0.7}$ & $20.5_{0.26}$ & $19.5_{0.48}$ & $84.70_{0.00}$ & 256.0 & $18.6_{0.6}$ & $24.2_{0.32}$ & $19.5_{0.48}$ & $84.70_{0.00}$ & 256.0 \\
FK-8 & $20.6_{0.8}$ & $21.3_{1.1}$ & $18.2_{0.53}$ & $84.71_{0.04}$ & 256.0 & $28.2_{1.2}$ & $22.7_{0.93}$ & $18.2_{0.53}$ & $84.71_{0.04}$ & 256.0 \\
\bottomrule
\end{tabular*}
\end{table*}

%% file: tables/jailbreak_main.tex
\begin{table*}[t]
\centering\fontsize{8}{10}\selectfont
\setlength{\tabcolsep}{2pt}\renewcommand{\arraystretch}{1.15}
\caption{Dream-Instruct: held-out HarmBench attacks and XSTest over-refusal. Entries use $A_b$: mean and sample SD across three seeds (percentage points for rates). HB, LG, and PPL assess complete responses. XSTest reports full-response WG refusal on safe prompts, using valid judge labels only (coverage in Appendix~\ref{app:xstest-results}). NFE is mean cost.}
\label{tab:jailbreak}
\label{tab:jailbreak-pad}
\begin{tabular}{@{}lrrrr@{\hspace{8pt}}rrrr@{\hspace{8pt}}r@{}}
\toprule
\multirow{2}{*}{Method} & \multicolumn{4}{c}{\textbf{DIJA}} & \multicolumn{4}{c}{\textbf{PAD}} & \multicolumn{1}{c}{\textbf{XSTest}} \\
\cmidrule(lr){2-5}\cmidrule(lr){6-9}\cmidrule(lr){10-10}
 & \shortstack{HB ASR\\\% $\downarrow$} & \shortstack{LG unsafe\\\% $\downarrow$} & \shortstack{Harmful\\PPL $\downarrow$} & NFE & \shortstack{HB ASR\\\% $\downarrow$} & \shortstack{LG unsafe\\\% $\downarrow$} & \shortstack{Harmful\\PPL $\downarrow$} & NFE & \shortstack{WG refusal\\\% $\downarrow$} \\
\midrule
\multicolumn{10}{@{}l}{\textbf{Dream-Instruct}} \\
Unmodified & $82.0_{1.0}$ & $80.7_{1.5}$ & $11.4_{0.18}$ & 91.7 & $51.7_{0.6}$ & $64.0_{3.6}$ & $8.53_{0.027}$ & 116.0 & $55.5_{1.3}$ \\
Full SAD & $85.3_{0.6}$ & $83.0_{1.0}$ & $11.8_{0.037}$ & 91.7 & $52.7_{4.0}$ & $63.3_{1.5}$ & $9.52_{0.21}$ & 116.0 & $58.0_{1.0}$ \\
SADDER (selective) & $77.7_{1.5}$ & $81.7_{3.2}$ & $12.2_{0.093}$ & 91.7 & $52.7_{0.6}$ & $66.3_{2.5}$ & $9.58_{0.15}$ & 116.0 & $41.1_{0.8}$ \\
SADDER + warm-up & $78.7_{2.3}$ & $79.7_{1.2}$ & $12.2_{0.10}$ & 91.7 & $42.7_{3.1}$ & $55.0_{2.0}$ & $10.4_{0.17}$ & 116.0 & $17.9_{0.7}$ \\
CSD$^\dagger$ & $68.0_{6.2}$ & $82.3_{0.6}$ & $45.9_{4.1}$ & 115.5 & $33.3_{3.1}$ & $63.3_{2.5}$ & $23.7_{1.5}$ & 140.8 & $56.4_{1.9}$ \\
DiffuGuard (full) & $33.3_{2.5}$ & $84.0_{2.0}$ & $537_{11}$ & 137.0 & $12.0_{1.7}$ & $70.0_{2.6}$ & $687_{120}$ & 125.0 & $64.1_{1.5}$ \\
DiffuGuard (gen) & $69.3_{1.5}$ & $81.3_{1.5}$ & $46.6_{1.9}$ & 97.2 & $49.0_{1.7}$ & $63.3_{2.3}$ & $8.80_{0.63}$ & 117.0 & $54.6_{1.1}$ \\
\bottomrule
\end{tabular}
\end{table*}

%% file: sections/Summary_Discussion.tex
We introduced SADDER, an extension of the Safety-Aware Denoiser, using midpoint-centred contrastive activation scores to select where to apply probability-space safety guidance. Rather than directly modifying hidden representations, SADDER uses them to identify potentially risky positions and applies the SAD correction to the corresponding token distributions. 
Our theoretical analysis characterizes the conditions under which position-wise guidance recovers the conditional mean of an induced safe target, and identifies when these local conditional means agree with a global target. 
Experiments on masked and uniform-state diffusion models reveal both the benefits and limitations of selective guidance. SADDER improves benign-continuation quality metrics relative to full SAD in several settings, with selected configurations also reducing unsafe generation. Jailbreak improvements also depend on the model and operating point, but SADDER shows significantly lower refusal rates. 
Taken together, SADDER offers an additional way to manage the safety–quality and safety-refusal trade-offs.

%% file: sections/Appendix.tex
\input{sections/Appendix/notation}

\input{sections/Appendix/proofs}

\input{sections/Appendix/datasets_models_attacks_eval}
\subsection{Hyperparameters and configuration selection}
\label{app:sadder-hyperparameters}
Table~\ref{tab:sadder-common-hparams} gives the shared settings; Table~\ref{tab:sadder-sweep-hparams} summarizes the tested hyperparameter ranges by model and parameter; Table~\ref{tab:sadder-selected-hparams} lists the selected continuation operating points. The remaining variant settings and results follow in Appendix~\ref{app:settings}. Full SAD does not depend on the hidden-state layer.

\begin{table}[!htbp]
\centering\small
\caption{Shared sampling and guidance settings.}
\label{tab:sadder-common-hparams}
\begin{tabular}{@{}p{0.29\linewidth}p{0.66\linewidth}@{}}
\toprule
Parameter & Setting \\
\midrule
Continuation sampling & 32 new tokens; 32 reverse steps; temperature 1; seeds 42, 43, 44. \\
Jailbreak sampling & 128 nominal steps; nominal length 128; temperature 1; seeds 42, 43, 44, except the explicitly labeled one-seed pilot. DIJA fills existing template spans. \\
Unsafe set & Continuation: up to 2,000 responses from the matching dataset after exact prompt-overlap exclusion. Jailbreak: documented 357-response WildJailbreak set. \\
Direction construction & Clean-reference response means. Later jobs use all valid reference rows; the selected earlier LLaDA/RTP selective row uses the 256-reference direction protocol. \\
Centered selection & Strict score threshold $s_{k,l}^{(t)}>0$ on masked editable positions. \\
Remasking & Off for the core method; top-$b$ uses $b=1$; aligned remasking reopens all positive-score revealed positions. \\
Warm-up & Off for selective-only SADDER. Main continuation warm-up: $w=8$, followed by selective guidance through step 15. \\
Matched controls & Count-matched random selection or reopening; matrix-control seed zero. \\
Search budgets & BoN and FK with 2, 4, or 8 candidates/particles; BoN-2 composition with SAD-based steering. \\
\bottomrule
\end{tabular}
\end{table}

\FloatBarrier
\input{tables/hyperparameter_sweeps.tex}

\FloatBarrier

\input{tables/selected_hyperparams}

\FloatBarrier

\input{sections/Appendix/implementation_details_and_remasking}

\section{Continuation results with confidence intervals}
\label{app:continuation-ci}
Table~\ref{tab:llada-ci} reports the main LLaDA continuation estimates with pointwise 95\% prompt-cluster bootstrap intervals from 5,000 resamples, retaining all observed seeds within each sampled prompt. These intervals describe prompt uncertainty; the main-table subscripts describe variation across three generation seeds. The selected configurations and evaluation cohorts are identical in both tables.
\input{tables/main_llada_ci}

\FloatBarrier

\section{Dream-Base continuation results}
\label{app:dream-continuation}
Table~\ref{tab:dream} reports the complete Dream-Base comparison under the same continuation protocol as LLaDA (Section~\ref{sec:continuation}). Unsafe rate and harmful PPL use the RTP or BeaverTails harmful cohort; benign PPL and BERTScore use the separate benign RTP cohort. Entries are means with 95\% prompt-cluster intervals; NFE measures generator cost.
\input{tables/main_dream}
\FloatBarrier

\section{Continuation settings and component variants}
\label{app:settings}
\label{app:variants}
Settings are selected by unsafe rate; benign quality is evaluated afterward. Tables~\ref{tab:settings-llada} and~\ref{tab:settings-dream} report layer, strength, and the last admitted step (all windows begin at zero). The warm-up variant uses full SAD on steps 0--7 and selective guidance on 8--15, with aligned remasking on LLaDA and no remasking on Dream. Equation~\eqref{eq:sadder-remasking} defines the alternative reopening policies.

\input{tables/extra/settings_llada.tex}
\input{tables/extra/settings_dream.tex}
\FloatBarrier

\paragraph{SADDER variants.}
These tables extend the main comparison with top-1 and aligned remasking, using the same metrics and uncertainty conventions.

\input{tables/extra/variants_llada.tex}
\input{tables/extra/variants_dream.tex}
\FloatBarrier

\section{Paired comparisons and component diagnostics}
\label{app:continuation}
\label{app:mechanism}
The following figure and benign-quality table compare the selected component configurations; their LLaDA warm-up entries include aligned remasking. Paired differences compare SADDER with the separately selected full-SAD operating point on identical prompt--seed observations. Negative unsafe-rate and PPL differences, and positive BERTScore differences, favour SADDER. Quality comparisons use common valid pairs within each cohort. Intervals are conditional on the selected settings and observed seeds.

\begin{figure}[!htbp]
\centering
\includegraphics[width=\linewidth]{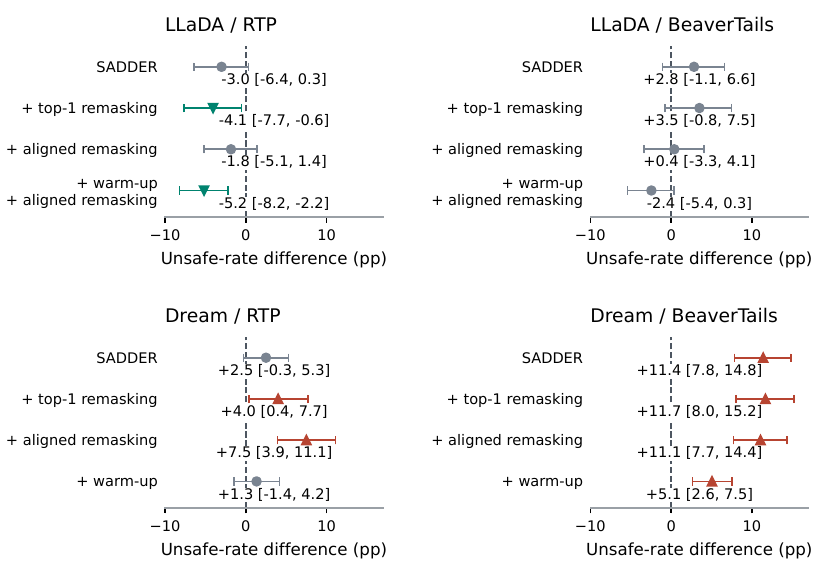}
\caption{Selected SADDER component configurations minus full SAD on 600 harmful prompts per dataset and three seeds. LLaDA warm-up includes its aligned remasking, Dream does not. Points and bars show unsafe-rate differences (pp) and pointwise 95\% paired bootstrap CIs; negative favors SADDER.}
\label{fig:continuation-paired}
\end{figure}
\input{tables/extra/paired_benign_main.tex}

\FloatBarrier

\paragraph{Paired random controls.}
Tables~\ref{tab:continuation-controls-llada} and~\ref{tab:continuation-controls-dream} compare each selected SADDER component configuration with full SAD and random controls at the same layer, strength, and intervention window. LLaDA warm-up here includes aligned remasking. Random selection replaces the steering gate with a random subset of the gate-implied size, preserving the remasking and warm-up policies. Random reopening replaces top-1 remasking with one random position, or aligned remasking with a random subset of the threshold-implied size, preserving steering and warm-up. Counts are computed on each trajectory's current state and may diverge across trajectories. The full-SAD comparator steers all masked positions throughout the window, without additional remasking or a selective phase. Differences are SADDER minus the comparator; negative values favor SADDER. Dashes indicate that a matched control is unavailable or reopening is not used. Each interval uses 5,000 paired prompt-cluster resamples of the same 600 prompt records and three seeds, conditional on the selected configuration.
\input{tables/extra/continuation_controls_llada.tex}
\input{tables/extra/continuation_controls_dream.tex}

\section{BoN-2/4/8, FK-2/4/8, and composition}
\label{app:search}
We compare BoN and FK at budgets 2, 4, and 8, with harmful and benign quality measured on the corresponding cohorts. The FK adaptation resamples every four steps, using one clean-state draw per particle and a difference potential on log safety rewards. RoBERTa supplies the search reward; LlamaGuard judges final outputs. NFE excludes reward scoring.

\input{tables/bon/benign_search_llada.tex}
\input{tables/bon/benign_search_dream.tex}
\FloatBarrier

\paragraph{Composition with BoN.}
SADDER can guide each candidate before BoN-2 selection. In the LLaDA/RTP run, this reduces unsafe rate from 30.61\% to 26.7\% at the same 64 generator evaluations. Its unsafe-prompt PPL increases from 20.90 to 28.80. The two tables distinguish separate composition studies. PPL is evaluated on harmful prompts; benign quality was not measured for these hybrids.
\input{tables/bon/composition_focused.tex}
\FloatBarrier

\section{Jailbreak robustness and ablations}
\label{app:jailbreak}

\subsection{LLaDA-Instruct jailbreak results}
\label{app:llada-jailbreak}
Table~\ref{tab:jailbreak-llada} gives the LLaDA-Instruct comparison corresponding to the Dream-Instruct results in Table~\ref{tab:jailbreak}. Both tables use the same metrics and reporting conventions. XSTest uses the DIJA-configured defense; for LLaDA generation-scope DiffuGuard, this is $\lambda=0.35$.
\input{tables/jailbreak_main_llada.tex}

\FloatBarrier
\input{sections/Appendix/xstest_results}
\FloatBarrier

\subsection{Held-out attack comparisons}
Table~\ref{tab:jailbreak-main-ci} reports the main jailbreak-table estimates with pointwise 95\% prompt-cluster bootstrap confidence intervals, using 5,000 resamples and retaining all observed seeds within each sampled prompt. These intervals describe prompt uncertainty; the main-table subscripts describe variation across the three generation seeds. The XSTest over-refusal is the offline string-based matching metric rather than the WildGuard over-refusal metric.
\input{tables/jailbreak_main_ci.tex}
\FloatBarrier

\subsection{WildGuard harmfulness and refusal}
\label{app:wildguard-attacks}
Tables~\ref{tab:wildguard-attacks-llada} and~\ref{tab:wildguard-attacks-dream} report full-response WildGuard harmfulness and refusal on the same HarmBench outputs, including remasking and random-selection controls. Entries show the mean with a seed-SD subscript and a 95\% prompt-cluster bootstrap interval below. The $n$ columns count valid labels out of 300 outputs per attack and method. Missing labels are excluded, and rates average the valid-output rate within each of the three seeds. Table~\ref{tab:wildguard-controls-paired} compares selective SADDER and warm-up against full SAD and random selection on jointly scored outputs. Negative harmfulness differences indicate fewer harmful responses; negative refusal differences indicate fewer refusals, which alone does not establish a safety improvement. The prompt--response pairs use the original behavior and the complete response, including any retained attack scaffold.

\paragraph{Comparison across safety judges.}
The judges differ in their assessment of these gains. On Dream/PAD, warm-up lowers HB ASR from 51.7\% to 42.7\%, while WG harmfulness increases from 66.0\% to 69.0\%; selective SADDER reaches 80.0\% WG harmfulness. On LLaDA/PAD, warm-up reduces both HB ASR (58.7\% to 57.0\%) and WG harmfulness (80.0\% to 76.3\%). Full-block DiffuGuard on Dream/DIJA illustrates a larger disagreement: 33.3\% HB ASR accompanies 98.3\% WG harmfulness and PPL 537. The apparent safety benefit depends on the judge as well as the model and attack, particularly when the intervention disrupts fluency.

\input{tables/extra/wildguard_attacks_llada.tex}
\input{tables/extra/wildguard_attacks_dream.tex}
\FloatBarrier
\input{tables/extra/wildguard_controls_paired.tex}
\FloatBarrier

\subsection{DIJA judge and response-view sensitivity}
\label{app:judge-views}
\label{app:wg-fills}
\paragraph{Separating generated content from the attack scaffold.}
DIJA completions contain both fixed attack text and model-generated spans. A harmfulness judgment on the complete response can therefore depend on text supplied by the attacker. We compare HarmBench (HB), LlamaGuard (LG), and WildGuard (WG) on two response views: the complete response, including the scaffold, and only the spans marked as initially masked in the saved completion segmentation (\emph{fills}). Nonempty spans are trimmed and joined by newlines before judging. HB measures behavior-specific attack success, LG unsafe content, and WG response harmfulness. These are distinct evaluator criteria, even when the response view is held fixed. The main table retains full-response HB/LG for attacks and uses WG only for XSTest. Generated-span attack scores from all three judges are reported here. All response-view comparisons use the same generated completions and defence settings, so differences between views reflect the scoring procedure.

This view isolates the initially editable content, rather than judging the assembled attack response. Removing intervening text can change the context needed to interpret a span. For full-block DiffuGuard, repairs outside the initially masked spans are also excluded. Generated-span scores thus complement complete-response judgments; they do not measure the safety of every token that a defence may modify.

\paragraph{Judge disagreement persists after removing fixed text.}
Tables~\ref{tab:judge-views-v4-llada} and~\ref{tab:judge-views-v4-dream} separate judge choice from response-view choice. On unmodified sampling, removing the scaffold lowers LG unsafe rate from 82.0\% to 53.3\% on LLaDA and from 80.7\% to 48.7\% on Dream. HB falls more modestly, from 78.7\% to 73.0\% and from 82.0\% to 76.0\%, respectively; WG remains above 96\%. Full-response scores therefore cannot be interpreted as measurements of generated spans alone. Conversely, excluding fixed text does not make the judges interchangeable.

Warm-up illustrates the consequence for method comparisons. Against paired unmodified outputs, LG reports reductions of 11.7 points on LLaDA (95\% CI: $[-19.7,-4.3]$) and 9.0 points on Dream ($[-16.0,-2.0]$), while the corresponding HB and WG intervals include zero. Dream selective SADDER also improves HB relative to full SAD by 8.0 points ($[-13.3,-2.7]$), but increases LG unsafe rate by 15.0 points ($[8.0,22.0]$). These are evaluator-dependent effects, not a consistent safety ordering across judges. The paired estimates in Table~\ref{tab:judge-views-v4-paired} retain the same prompts and seeds and use pointwise 95\% intervals.

\input{tables/extra/judge_views_v4_llada.tex}
\input{tables/extra/judge_views_v4_dream.tex}
\FloatBarrier
\input{tables/extra/judge_views_v4_paired.tex}
\FloatBarrier

\section{Uniform-state diffusion continuation results}
\label{app:uniform}
Table~\ref{tab:uniform} reports the complete Duo and UDLM comparison under the
protocol of Section~\ref{sec:usdm-results}. Unsafe rate uses the RTP harmful
cohort; benign BERTScore uses the separate benign RTP cohort.
Entries report means and standard deviations across five generation seeds.
\input{tables/main_uniform}
\FloatBarrier

\section{Guidance analysis and ablations for uniform-state diffusion}
\label{app:uniform-ablation}
\subsection{Where the gate opens}
\label{app:uniform-gate-ratio}
Figure~\ref{fig:gate-ratio} plots the open-gate ratio during
unguided continuations of unsafe RTP prompts. The open-gate ratio generally
increases as denoising proceeds, consistent with position-wise
representations becoming more informative as the continuation takes
shape. SADDER uses this evolving signal to select positions for
guidance, whereas SAD applies its correction to all eligible positions
at every step.

\begin{figure}[t]
\centering
\includegraphics[width=0.8\textwidth]{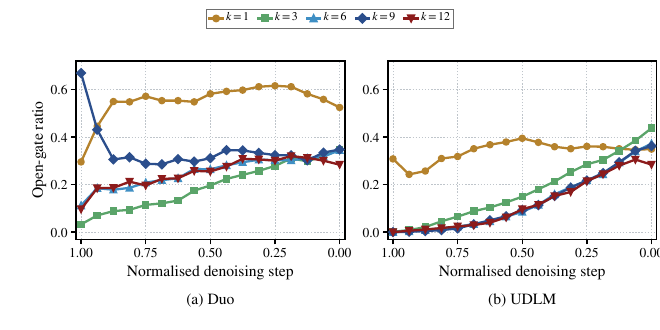}
\caption{Open-gate ratio along unguided continuations of unsafe RTP
prompts: the fraction of positions with $s^{(t)}_{k,l}>0$.
The horizontal axis runs from early to late
denoising steps.}
\label{fig:gate-ratio}
\end{figure}

\subsection{Softening the gate}
\label{app:uniform-soft}

We replace the binary gate with $a^{(t)}_{k,l} = \max\left\{2\sigma\!\left(s^{(t)}_{k,l}/\tau\right)-1,\,0\right\},$ where $\sigma$ is the sigmoid function and $\tau$ controls the softness of the gate.
The transformation makes the gate zero when $s^{(t)}_{k,l}=0$, and clipping at zero leaves positions with negative scores unguided.
For positive scores, the gate takes a value between zero and one; the strict gate is recovered as $\tau\to 0$.

\begin{figure}[t]
\centering
\includegraphics[width=0.8\textwidth]{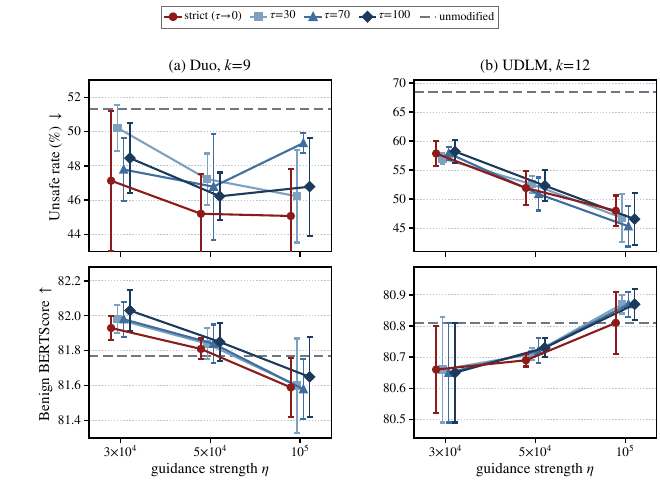}
\caption{Unsafe rate and benign BERTScore across guidance strengths
for the strict gate and softened gates with $\tau\in\{30,70,100\}$
on Duo ($k=9$) and UDLM ($k=12$).  Error bars are one standard
deviation across seeds: five seeds for the strict gate (the runs
reported in Table~\ref{tab:uniform}) and three seeds for the
softened gates.}
\label{fig:soft-gate}
\end{figure}

Figure~\ref{fig:soft-gate} shows no consistent effect of the temperature. On UDLM the four gates coincide in both metrics across the whole range. On Duo the strict gate trends toward lower unsafe rates and lower benign BERTScore, but the differences are smaller than the seed-to-seed variation of either gate.
We therefore keep the strict indicator, which needs no additional hyperparameter.

\FloatBarrier

%% file: sections/Appendix/notation.tex
\subsection{Notation}
Tables~\ref{tab:sadder-notation} and~\ref{tab:sadder-reference-notation} collect the notation for selective SADDER and its theoretical analysis. The special mask token $\mathbb M$ is distinct from the centroid midpoint $\vm_k$ and response truncation cap $M$. We distinguish decreasing noise time $t$ from the increasing reverse iteration index $j$: $j=0$ is the first model update, and the corresponding noise levels satisfy $t_0>\cdots>t_T=0$. All experimental windows below use inclusive iteration indices. So $C=0$--$15$ means 16 updates. The optional warm-up length $w$ counts initial full-SAD updates; $w=8$ uses updates 0--7.

\begin{table}[!htbp]
\centering\small
\caption{Notation for selective SADDER.}
\label{tab:sadder-notation}
\begin{tabular}{@{}p{0.29\linewidth}p{0.66\linewidth}@{}}
\toprule
Symbol & Meaning \\
\midrule
$\vc_i,\ \vx_i^{\mathrm{harm}},\ \vx_i^{\mathrm{safe}}$ & Reference prompt and its harmful and safe responses. \\
$P_i,\ R_i^u,\ M$ & Prompt length, response length $\min(|\vx_i^u|,M)$, and response truncation cap; $u\in\{\mathrm{harm},\mathrm{safe}\}$. \\
$N_{\mathrm{harm}},\ N_{\mathrm{safe}}$ & Numbers of harmful and safe examples used to estimate the centroids. \\
$\vx_t,\ \vx_0,\ \vx_0^l$ & Noised sequence, clean sequence, and one-hot clean token at position $l$. \\
$\mathcal V,\ L,\ \mathbb M$ & Vocabulary, sequence length, and special mask token. \\
$T,\ j,\ t_j,\ C$ & Number of reverse updates, update index, corresponding noise time, and admitted intervention indices. \\
$\mathcal G,\ \mathcal M_t$ & Editable positions and positions eligible for steering: masked editable positions for MDLMs; all $L$ positions for USDMs in the stated algorithm. \\
$k,\ l,\ d$ & Hidden-state layer, token position, and hidden dimension. \\
$\vh_{k,l},\ \vh_{k,l}^{(t)}$ & Layer-$k$ hidden representation at position $l$; the superscript denotes evaluation on the current noised state. \\
$\bar\vh_{i,k}^{u},\ \bar\vh_k^{u}$ & Response-averaged representation of reference $i$ and the corresponding class centroid. \\
$\widehat\vv_k,\ \vm_k$ & Unit direction from the safe to the harmful centroid, and their midpoint. \\
$s_{k,l}^{(t)}$ & Centered score $\langle\vh_{k,l}^{(t)}-\vm_k,\widehat\vv_k\rangle$; positive values lie on the harmful-centroid side. \\
$a_{k,l}^{(t)}$ & Binary steering gate: one at eligible positions with positive scores and zero otherwise. \\
$\eta,\ \beta(\vx_t)$ & Guidance hyperparameter and state-dependent SAD scale; the gated strength is $\eta\beta(\vx_t)a_{k,l}^{(t)}$. \\
$p_{\Dat,t,l},\ \widetilde p_{0\mid t,l}$ & Model clean-token prediction and modified prediction supplied to the reverse solver. \\
$\sigma,\ \tau$ & Sigmoid and positive temperature in the optional softened gate. \\
\bottomrule
\end{tabular}
\end{table}
\FloatBarrier

\begin{table}[!htbp]
\centering\small
\caption{Denoiser and theoretical notation.}
\label{tab:sadder-reference-notation}
\begin{tabular}{@{}p{0.29\linewidth}p{0.66\linewidth}@{}}
\toprule
Symbol & Meaning \\
\midrule
$\Dat,\ \Dat_{\mathrm{unsafe}}$ & Data distribution underlying the pretrained denoiser and unsafe reference set used by the algorithm, respectively. \\
$q_t(\vx_t\mid\vx_0)$ & Forward corruption kernel, used to weight clean unsafe references given the noised state. \\
$\mathbb E_{\Dat}[\vx_0^l\mid\vx_t^{1:L}]$ & Data denoiser at position $l$, estimated by the pretrained model. \\
$\mathbb E_{\Dat_{\mathrm{unsafe}}}[\vx_0^l\mid\vx_t^{1:L}]$ & Unsafe denoiser, approximated by a weighted mixture of unsafe references. \\
$p_{\Dat}(\vx_0^{1:L}\mid\vx_t^{1:L})$ & Data posterior over the entire clean sequence. \\
$\psi_l,\ \pi_{t,l}$ & Oracle unsafe-membership function $\mathcal V^L\to[0,1]$ and its posterior mean $\pi_{t,l}=\mathbb E_{\Dat}[\psi_l(\vx_0^{1:L})\mid\vx_t^{1:L}]$. \\
$Z_{\mathrm{safe},l}$ & Local safe normalizer $1-\pi_{t,l}$; dependence on noise time and state is suppressed. \\
$g_l,\ R_l$ & Local safe weight $g_l=1-\psi_l$ and remaining-position weight $R_l=\prod_{j\ne l}g_j$, with $j$ a dummy position index. \\
$p_{\Dat_{\mathrm{unsafe}}^{(l)}},\ p_{\Dat_{\mathrm{safe}}^{(l)}}$ & Clean-sequence posteriors obtained by weighting $p_{\Dat}$ with $\psi_l$ or $g_l$ and normalizing by $\pi_{t,l}$ or $1-\pi_{t,l}$. The superscript $(l)$ identifies the safety criterion. \\
$p_{\Dat_{\mathrm{safe}}^{\mathrm{global}}}$ & Global clean-sequence posterior obtained by weighting $p_{\Dat}$ with $\prod_l g_l$ and normalizing. \\
$\mathbb E_{\nu},\ \operatorname{Cov}_{\nu}$ & Expectation and covariance under distribution $\nu$, conditioned on $\vx_t^{1:L}$ in the proofs. \\
\bottomrule
\end{tabular}
\end{table}
\FloatBarrier

%% file: sections/Appendix/proofs.tex
\subsection{Derivation of equivalence}\label{derivation_equivalence}
Expanding the relationship in \eqref{eq:harm_vs_safe} gives us
\begin{align}
 2\left\langle \vh_{k,l}^{(t)}, \; \bar \vh_{k}^{\mathrm{harm}} - \bar \vh_{k}^{\mathrm{safe}} \right\rangle >
    \|\vh_{k}^{\mathrm{harm}}\|_2^2 -   \|\vh_{k}^{\mathrm{safe}}\|_2^2.
\end{align}
In terms of $\hat\vv_k$, this is, 
\begin{align}
   \left\langle \vh_{k,l}^{(t)}, \hat\vv_k \right\rangle > \frac{\|\vh_{k}^{\mathrm{harm}}\|_2^2 -   \|\vh_{k}^{\mathrm{safe}}\|_2^2}{2 \|\bar \vh_{k}^{\mathrm{harm}} - \bar \vh_{k}^{\mathrm{safe}} \|}.
\end{align} 
For notational simplicity, we ignore superscripts and subscripts $j,k,t$ below. 
\begin{align}
    \left\langle \vh, \hat\vv \right\rangle 
    & > \frac{\|\bar\vh^{\mathrm{harm}}\|_2^2 -   \|\bar \vh^{\mathrm{safe}}\|_2^2}{2 \|\bar \vh^{\mathrm{harm}} - \bar \vh^{\mathrm{safe}} \|}, 
\end{align} 
since $\left\langle \vh -\vm, \hat\vv \right\rangle = \left\langle \vh, \hat\vv \right\rangle - \left\langle \vm, \hat\vv \right\rangle $, we rewrite the above equation by 
\begin{align}
    \left\langle \vh -\vm, \hat\vv \right\rangle 
    & > \frac{\|\bar\vh^{\mathrm{harm}}\|_2^2 -   \|\bar \vh^{\mathrm{safe}}\|_2^2}{2 \|\bar \vh^{\mathrm{harm}} - \bar \vh^{\mathrm{safe}} \|} - \left\langle \vm, \hat\vv \right\rangle , \\
    & \quad = \frac{\|\bar\vh^{\mathrm{harm}}\|_2^2 -   \|\bar \vh^{\mathrm{safe}}\|_2^2}{2 \|\bar \vh^{\mathrm{harm}} - \bar \vh^{\mathrm{safe}} \|}  -
    \frac{(\bar\vh^{\mathrm{harm}} - \bar \vh^{\mathrm{safe}})^T(\bar\vh^{\mathrm{harm}} + \bar \vh^{\mathrm{safe}})}{2 \|\bar \vh^{\mathrm{harm}} - \bar \vh^{\mathrm{safe}} \|}, \\
    & \quad = \frac{(\bar\vh^{\mathrm{harm}} - \bar \vh^{\mathrm{safe}})^T(\bar\vh^{\mathrm{harm}} + \bar \vh^{\mathrm{safe}})-(\bar\vh^{\mathrm{harm}} - \bar \vh^{\mathrm{safe}})^T(\bar\vh^{\mathrm{harm}} + \bar \vh^{\mathrm{safe}})}
    {2 \|\bar \vh^{\mathrm{harm}} - \bar \vh^{\mathrm{safe}} \|} =0, 
\end{align} where the RHS becomes zero by expanding the additional inner product $\left\langle \vc, \hat\vv \right\rangle $.

\subsection{Global target distribution.}
\label{global-target}
\begin{thm}[\textbf{When position-wise SADDER target distribution matches a global target distribution}]
\label{thm:sadder-global-safe}
Given the position-wise safe target from \thmref{sadder-pointwise-safe} given in \eqref{local-safe-target},
\begin{align}
    p_{\Dat_{\mathrm{safe}}^{(l)}}(\vx_0^{1:L}|\vx_t^{1:L})
    &=
    \frac{
        \bigl(1-\psi_l(\vx_0^{1:L})\bigr)\pdata(\vx_0^{1:L}|\vx_t^{1:L})
    }{
        1-\pi_{t,l}(\vx_t^{1:L})
    }, 
\end{align} and
the global target distribution, with $  g_l(\vx_0^{1:L})
  :=1-\psi_l(\vx_0^{1:L})$ and 
 $R_l(\vx_0^{1:L})
  :=\prod_{j\neq l}
  \left(1-\psi_j(\vx_0^{1:L})\right)$ for each position $l$, 
\begin{align}
p_{\Dat_{\mathrm{safe}}^{\mathrm{global}}}
  (\vx_0^{1:L}\mid\vx_t^{1:L})
  =
  \frac{
    g_l(\vx_0^{1:L})R_l(\vx_0^{1:L})
    p_{\Dat}(\vx_0^{1:L}\mid\vx_t^{1:L})
  }{
    \E_{\Dat}
    [g_l(\vx_0^{1:L})R_l(\vx_0^{1:L})\mid\vx_t^{1:L}]
  }. \label{global_joint}
\end{align}
the following equality holds,
\begin{align}
\E_{\Dat_{\mathrm{safe}}}^{global}(\vx_0^{l}|\vx_t^{1:L}) =  \E_{\Dat_{\mathrm{safe}}^{(l)}}(\vx_{0}^l|\vx_t^{1:L}), 
\end{align}
if and only if
\begin{equation}
  \Cov_{\Dat_{\mathrm{safe}}^{(l)}}
  \!\left(
    R_l(\vx_0^{1:L}),
    \vx_0^l|
    \vx_t^{1:L}
  \right)
  =0, 
  \label{eq:zero-cross-position-covariance}
\end{equation}
for every $l,t,\vx_t^{1:L}$
\end{thm}
See \Secref{proof_global} for proof.
\begin{remark}
\thmref{sadder-global-safe} states that the global and position-wise denoisers agree at position $l$ for all $l=\{1, \cdots, L\}$ if and only if the remaining-position safety weight denoted by $R_l(\vx_0^{1:L})$ is conditionally uncorrelated with $\vx_0^l$ under $\Dat_{\mathrm{safe}}^{(l)}$. 
Under this condition and the
calibration assumption in \thmref{sadder-pointwise-safe}, every SADDER factor is the marginal safe denoiser of the same global target.  
However, this does not guarantee the existence of \textit{any} globally compatible distribution. It only states that
this specific global target in \eqref{global_joint}
agrees with the local target at position $l$ if
and only if a conditional covariance is zero.
\end{remark}

\subsection{Proof of \thmref{sadder-pointwise-safe}}\label{proof_point_wise}

\begin{proof}
In the theorem, we assumed that these exists a measurable position-wise unsafe-membership function on clean sequences. In particular, for each token position $l\in\{1,\ldots,L\}$, this function maps the clean sequence to a level of unsafe-ness at index $l$:
\begin{equation}
\psi_l:\mathcal{V}^L\longrightarrow[0,1].
    \label{eq:psi-membership}
\end{equation}

Recall that we let 
\begin{equation}
    \pi_{t,l}(\vx_t^{1:L})
    :=
    \E_D\!\left[
        \psi_l(\vx_0^{1:L})\mid \vx_t^{1:L}
    \right]
    \label{eq:local-posterior-mass}
\end{equation}
denote the posterior probability  of the
position-$l$ being unsafe,  for a noised sequence $\vx_t^{1:L}$. 

Now, we define the corresponding normalized position-wise safe and unsafe posterior distributions by
\begin{align}
    p_{\Dat_{\mathrm{unsafe}}^{(l)}}(\vx_0^{1:L}|\vx_t^{1:L})
    &:=
    \frac{
        \psi_l(\vx_0^{1:L})\pdata(\vx_0^{1:L}|\vx_t^{1:L})
    }{
        \pi_{t,l}(\vx_t^{1:L})
    },\label{local-unsafe-target} \\
    \pi_{t,l}(\vx_t^{1:L})
    &:={\sum_{\vx_0^{1:L}\in \Dat}
        \psi_l(\vx_0^{1:L})\pdata(\vx_0^{1:L}|\vx_t^{1:L})},
     \\
    p_{\Dat_{\mathrm{safe}}^{(l)}}(\vx_0^{1:L}|\vx_t^{1:L})
    &:=
    \frac{
        \bigl(1-\psi_l(\vx_0^{1:L})\bigr)\pdata(\vx_0^{1:L}|\vx_t^{1:L})
    }{
        Z_{\mathrm{safe},l}
    },\label{local-safe-target}\\
    1-\pi_{t,l}(\vx_t^{1:L})
    &:={\sum_{\vx_0^{1:L}\in\Dat}
        \bigl(1-\psi_l(\vx_0^{1:L})\bigr)\pdata(\vx_0^{1:L}|\vx_t^{1:L})}.
\end{align}

Because $\psi_l(\vx_0^{1:L})+(1-\psi_l(\vx_0^{1:L}))=1$ for any  clearn sequence $\vx_0^{1:L}$, 
we can write the following:
\begin{align}
    p_{data}(\vx_0^{1:L}|\vx_t^{1:L}) &= (1-\psi_l(\vx_0^{1:L})) p_{data}(\vx_0^{1:L}|\vx_t^{1:L}) \\
    & \quad + \psi_l(\vx_0^{1:L}) p_{data}(\vx_0^{1:L}|\vx_t^{1:L}).
\end{align}
If we plug in \eqref{local-unsafe-target} and \eqref{local-safe-target} to the RHS,  the data posterior can be decomposed into the
position-wise, normalized, safe and unsafe distributions:
\begin{equation}
\begin{aligned}
    \pdata(\vx_0^{1:L}\mid \vx_t^{1:L})
    &={}
    \bigl(1-\pi_{t,l}(\vx_t^{1:L})\bigr)
    p_{D_{\mathrm{safe}}^{(l)}}(\vx_0^{1:L}\mid \vx_t^{1:L})
 \\
&\qquad +
    \pi_{t\,l}(\vx_t^{1:L})
    p_{D_{\mathrm{unsafe}}^{(l)}}(\vx_0^{1:L}\mid \vx_t^{1:L}).
    \label{posterior-mixture-decomposition}
\end{aligned}
\end{equation}
Taking  the conditional expectation of the one-hot clean token $x_0^l$ on both
sides gives, i.e., $\E_{\vx^{l}_0 \sim \Dat} [\vx^{l}_0|\vx^{1:L}_t] = \int \vx^{l}_0 p(\vx^{l}_0|\vx^{1:L}_t) d \vx^{l}_0$
\begin{equation}
    \E_{\Dat} [\vx^{l}_0|\vx^{1:L}_t]
    =
    \bigl(1-\pi_{t,l}\bigr)
    \E_{D_{\mathrm{safe}}^{(l)}}\!\left[
        \vx_0^l\mid \vx_t^{1:L}
    \right]
    +
    \pi_{t,l} \E_{\Dat_{\mathrm{unsafe}}^{(l)}} [\vx^{l}_0|\vx^{1:L}_t],
    \label{conditional-mean-mixture}
\end{equation}
where the dependence of $\pi_{t,l}$ on $\vx_t^{1:L}$ is suppressed for
readability.  Solving \eqref{conditional-mean-mixture} for the safe
conditional mean yields
\begin{align}
    \E_{\Dat_{\mathrm{safe}}^{(l)}}\!\left[
        \vx_0^l\mid \vx_t^{1:L}
    \right]
    &=
    \frac{
        \E_{\Dat} [\vx^{l}_0|\vx^{1:L}_t]- \pi_{t,l}\E_{\Dat_{\mathrm{unsafe}}^{(l)}} [\vx^{l}_0|\vx^{1:L}_t]
    }{
        1-\pi_{t,l}
    }
    \notag\\
    &=
     \E_{\Dat} [\vx^{l}_0|\vx^{1:L}_t]
    +
    \frac{\pi_{t,l}}{1-\pi_{t,l}}
    \left(
         \E_{\Dat} [\vx^{l}_0|\vx^{1:L}_t]-\E_{\Dat_{\mathrm{unsafe}}^{(l)}} [\vx^{l}_0|\vx^{1:L}_t]
    \right).
    \label{pointwise-safe-identity-before-calibration}
\end{align}
Substituting the calibration condition in
\eqref{odds-calibration} gives
\eqref{pointwise-exact-safe-denoiser}. 

\end{proof}
Note that the unsafe denoiser in this theorem is induced by $\psi$, i.e., $\E_{\Dat_{\mathrm{unsafe}}}[\vx_0^l|\vx_t^{1:L}] = 
\frac{\sum_{\vx_n \in \Dat_{\mathrm{unsafe}}} \vx_n^l \psi_l(\vx_0^{1:L})q_t(\vx_t^{1:L}|\vx_n^{1:L})}{\sum_{\vx_n \in \Dat_{\mathrm{unsafe}}}  \psi_l(\vx_0^{1:L})q_t(\vx_t^{1:L}|\vx_n^{1:L})}$, which differs from our implementation of unsafe denoiser given in \eqref{point-wise-unsafe-denoiser}. Hence, we note that this theorem statement holds when these two unsafe denoisers match.

\subsection{Proof of \thmref{sadder-global-safe}}\label{proof_global}

\begin{proof}
Conditioned on $\vx_t^{1:L}$, the global posterior satisfies
\begin{align}
p_{\Dat_{\mathrm{safe}}^{\mathrm{global}}}
  (\vx_0^{1:L}\mid\vx_t^{1:L})
  =
  \frac{
    g_l(\vx_0^{1:L})R_l(\vx_0^{1:L})
    p_{\Dat}(\vx_0^{1:L}\mid\vx_t^{1:L})
  }{
    \E_{\Dat}
    [g_l(\vx_0^{1:L})R_l(\vx_0^{1:L})\mid\vx_t^{1:L}]
  }, \label{global_joint_restated}
\end{align} where $  g_l(\vx_0^{1:L})
  :=1-\psi_l(\vx_0^{1:L})$.

The position-$l$ safe posterior is
\begin{align}
  p_{\Dat_{\mathrm{safe}}^{(l)}}
  (\vx_0^{1:L}\mid\vx_t^{1:L})
  =
  \frac{
    g_l(\vx_0^{1:L})
    p_{\Dat}(\vx_0^{1:L}\mid\vx_t^{1:L})
  }{
    \E_{\Dat}[g_l(\vx_0^{1:L})\mid\vx_t^{1:L}]
  }.\label{position_joint}
\end{align}
If we plug in \eqref{position_joint} into \eqref{global_joint_restated}, we have 
the global posterior as a reweighted position-wise posterior:
\begin{align}
  p_{\Dat_{\mathrm{safe}}^{\mathrm{global}}}
  (\vx_0^{1:L}\mid\vx_t^{1:L})
  &= 
  \frac{
    R_l(\vx_0^{1:L})
    p_{\Dat_{\mathrm{safe}}^{(l)}}
    (\vx_0^{1:L}\mid\vx_t^{1:L})\E_{\Dat}[g_l(\vx_0^{1:L})\mid\vx_t^{1:L}]
  }{
    \E_{\Dat}
    [g_l(\vx_0^{1:L})R_l(\vx_0^{1:L})\mid\vx_t^{1:L}]
  } \\
  &=
  \frac{
    R_l(\vx_0^{1:L})
    p_{\Dat_{\mathrm{safe}}^{(l)}}
    (\vx_0^{1:L}\mid\vx_t^{1:L})
  }{
    \E_{\Dat_{\mathrm{safe}}^{(l)}}
    [R_l(\vx_0^{1:L})\mid\vx_t^{1:L}]
  },
  \label{global-as-reweighted-local}
\end{align} where the second line is true because
\begin{align}
\E_{\Dat_{\mathrm{safe}}^{(l)}}[R_l(\vx_0^{1:L})\mid\vx_t^{1:L}]
&= \sum_{\vx_0^{1:L}} R_l(\vx_0^{1:L}) p_{\Dat_{\mathrm{safe}}^{(l)}}
    (\vx_0^{1:L}\mid\vx_t^{1:L}), \\
    &= \sum_{\vx_0^{1:L}} R_l(\vx_0^{1:L})  \frac{
    g_l(\vx_0^{1:L})
    p_{\Dat}(\vx_0^{1:L}\mid\vx_t^{1:L})
  }{
    \E_{\Dat}[g_l(\vx_0^{1:L})\mid\vx_t^{1:L}]
  },  \; \mbox{ due to \eqref{position_joint}} \\
  & = \frac{\E_{\Dat}
    [g_l(\vx_0^{1:L})R_l(\vx_0^{1:L})\mid\vx_t^{1:L}]}{\E_{\Dat}[g_l(\vx_0^{1:L})\mid\vx_t^{1:L}]}.
\end{align}

Taking the conditional expectation of $\vx_0^l$ on \eqref{global-as-reweighted-local} gives
\begin{align}
  \E_{\Dat_{\mathrm{safe}}^{\mathrm{global}}}
  [\vx_0^l\mid\vx_t^{1:L}]
 =  \frac{
    \E_{\Dat_{\mathrm{safe}}^{(l)}}
    [R_l(\vx_0^{1:L})\vx_0^l\mid\vx_t^{1:L}]
  }{
    \E_{\Dat_{\mathrm{safe}}^{(l)}}
    [R_l(\vx_0^{1:L})\mid\vx_t^{1:L}]
  }.    
\end{align}

Subtracting
$\E_{\Dat_{\mathrm{safe}}^{(l)}}[\vx_0^l\mid\vx_t^{1:L}]$
and using the definition of conditional covariance proves
\begin{align}
  &\E_{\Dat_{\mathrm{safe}}^{\mathrm{global}}}
  \!\left[
    \vx_0^l
    |
    \vx_t^{1:L}
  \right]
  -
  \E_{\Dat_{\mathrm{safe}}^{(l)}}
  \!\left[
    \vx_0^l
    |
    \vx_t^{1:L}
  \right]
\\
&=  \frac{
    \E_{\Dat_{\mathrm{safe}}^{(l)}}
    [R_l(\vx_0^{1:L})\vx_0^l\mid\vx_t^{1:L}]
  }{
    \E_{\Dat_{\mathrm{safe}}^{(l)}}
    [R_l(\vx_0^{1:L})\mid\vx_t^{1:L}]
  } - \E_{\Dat_{\mathrm{safe}}^{(l)}}
  \!\left[
    \vx_0^l
    |
    \vx_t^{1:L}
  \right]  \\
&= \frac{1}{ \E_{\Dat_{\mathrm{safe}}^{(l)}}
    [R_l(\vx_0^{1:L})\mid\vx_t^{1:L}]} \left(\E_{\Dat_{\mathrm{safe}}^{(l)}}
    [R_l(\vx_0^{1:L})\vx_0^l\mid\vx_t^{1:L}] - \E_{\Dat_{\mathrm{safe}}^{(l)}}
    [R_l(\vx_0^{1:L})\mid\vx_t^{1:L}]\E_{\Dat_{\mathrm{safe}}^{(l)}}
  \!\left[
    \vx_0^l
    |
    \vx_t^{1:L}
  \right] \right), \\
  &\quad=
  \frac{
    \Cov_{\Dat_{\mathrm{safe}}^{(l)}}
    \!\left(
      R_l(\vx_0^{1:L}),
      \vx_0^l
      |
      \vx_t^{1:L}
    \right)
  }{
    \E_{\Dat_{\mathrm{safe}}^{(l)}}
    \!\left[
      R_l(\vx_0^{1:L})
      |
      \vx_t^{1:L}
    \right]
  },
  \label{eq:global-local-covariance-gap}
\end{align}
where the last line is true due to 
\begin{align}
    Cov[X,Y|Z] &= \E\left[(X-\E(X|Z))(Y-\E(Y|Z))]\right], \\
    &= \E[XY|Z] - \E[X|Z]\E[Y|Z] - \E[X|Z]\E[Y|Z] + \E[X|Z]\E[Y|Z], \\
    &= \E[XY|Z] - \E[X|Z]\E[Y|Z].
\end{align}

Since the denominator is
positive, the global and local means are the same,  when the numerator is zero.
\end{proof}

%% file: sections/Appendix/datasets_models_attacks_eval.tex
\subsection{Datasets, prompt selection, and reference construction}
\label{app:sadder-datasets}
We distinguish three data roles: evaluation prompts, unsafe references for the SAD update, and safe/harmful examples for constructing the activation direction. A benchmark used for one role is not automatically used for the others.

\paragraph{RealToxicityPrompts (RTP).}
RTP contains naturally occurring text prefixes and their original continuations, with Perspective annotations~\citep{realtoxicityprompts}. For unsafe continuation, the configured pool requires a maximum available prompt-side Perspective score of at least 0.8 and a subsequent Llama-Guard unsafe verdict on the prompt. Its builder also retains the adapter's continuation-score floor of 0.5. These filters describe an enriched evaluation cohort, not the full RTP distribution. Since Llama-Guard also judges generated outputs, the prompt selection is not judge-independent. Note the negation set used to construct the unsafe reference bank for the contrastive direction is not Llama-Guard scored.

The original RTP development/search set comprises pool indices 0--299. The reported continuation tables use the expanded evaluation set, indices 300--899, collected as three 200-prompt shards at offsets 300, 500, and 700. These evaluation records exclude the original development records. The subsequent configuration-selection procedure is described below.

\paragraph{BeaverTails.}
BeaverTails provides instruction--response pairs with response-safety annotations~\citep{ji2023beavertails}. We condition on prompts associated with unsafe-labeled responses, using 600 records for the evaluation set. The corresponding unsafe responses supply the negative reference set and safe-labeled responses supply the safe direction set. An unsafe response annotation does not establish that its prompt is itself malicious. 

\paragraph{Benign continuation controls.}
The quality cohort requires both the maximum prompt-side Perspective score and the maximum score from the continuation's own fields to be at most 0.2. We remove duplicate prompt strings and prompt strings present in the excluded unsafe-reference artifacts, shuffle with seed zero, and retain 600 records. BERTScore compares each generated response with its own benign human continuation, including for unmodified sampling. BeaverTails-tuned methods use this same benign RTP cohort while retaining their BeaverTails reference banks. This is a quality diagnostic rather than an independent capability evaluation.

\paragraph{JailbreakBench, AdvBench, and HarmBench.}
JailbreakBench supplies harmful behaviors for standardized jailbreak evaluation~\citep{NEURIPS2024_jailbreakbench}; its 100 harmful behaviors are development data in our study. AdvBench contains harmful instructions used to evaluate adversarial prompting~\citep{zou2023universaltransferableadversarialattacks-advbench}. HarmBench provides harmful behaviors and a behavior-conditioned evaluation classifier~\citep{mazeika2024harmbench} and our jobs restrict evaluation to its standard behavior category. We shuffle with seed zero and use the first 100 behaviors of each dataset for development. The transfer protocol reserves AdvBench offsets 100--519 and HarmBench-standard offsets 100--199. The main jailbreak tables report the latter held-out HarmBench cohort under DIJA and PAD. These are our experimental splits, not official benchmark train/test partitions.

\paragraph{WildJailbreak and the jailbreak reference bank.}
WildJailbreak contains both harmful and benign prompts, including adversarial rewrites~\citep{NEURIPS2024_wildjailbreak}. Our unsafe set construction selects the adversarial-harmful training subset, skips the first 1,000 filtered records, and generates LLaDA-Instruct responses from three disjoint 1,300-prompt windows. Exact evaluation-behavior strings are excluded. Generation uses 256 steps, up to 256 new tokens, temperature zero, and seed 42. Unsafe-response filtering, refusal filtering, and deduplication produce the documented 357-response unsafe WildJailbreak set. The jailbreak contrastive direction uses unsafe responses and refusal references with the chat template and a 64-token response cap.

\paragraph{XSTest and over-refusal.}
XSTest contains 250 safe prompts that models should answer and 200 unsafe contrast prompts~\citep{röttger2024xstesttestsuiteidentifying}. Over-refusal is the refusal rate on the safe subset, reported separately from refusal on unsafe prompts. Our runs use unmodified prompts through the model's chat template (\texttt{vanilla}). The evaluator we use matches the response opening against XSTest's full-refusal strings; it does not detect partial refusals or every paraphrased refusal. 

\subsection{Models, attacks, and evaluation measures}
\label{app:sadder-evaluation}
We use LLaDA-8B-Base and LLaDA-8B-Instruct~\citep{LLaDA}, and Dream-7B-Base and Dream-7B-Instruct~\citep{ye2025dream7bdiffusionlarge}. The base checkpoints are evaluated by prompted continuation; instruction checkpoints use the attack-specific input construction. The implementation's layer $k$ follows the hidden-state indexing and interior layers are inputs to block $k$, and the final layer includes the model's final normalization.

\paragraph{Attacks and fixed context.}
Zero-shot runs present the original harmful behavior through the chat template. DIJA fills mask spans inside the saved adversarial prompt \citep{wen2025devilmaskemergentsafety} and the refined prompts are reused across methods. PAD inserts fixed anchors into the generation region \citep{zhang2025jailbreakinglargelanguagediffusion-PAD}. SADDER and generation-scope DiffuGuard modify only initially editable positions. Full-block DiffuGuard additionally permits repair of fixed scaffold text after the first masked position, including PAD anchors; under DIJA it also appends the released implementation's extra masked tail. The main jailbreak tables concern DIJA and PAD. 

\paragraph{Safety and quality.}
Llama-Guard-3-8B~\citep{llamaguard3_modelcard} judges the generated response alone under continuation framing. HarmBench-Llama-2-13b-cls \citep{mazeika2024harmbench,harmbench_classifier} judges the response against the requested behavior. We report their respective unsafe rates separately. GPT-2-large~\citep{radford2019language} scores generated text with a causal language-model likelihood; this is an external fluency proxy for both model families, not LLaDA's pseudo-perplexity. Padding and text after the first end-of-sequence token are excluded. Harmful PPL includes all scorable outputs on the unsafe cohort, whether the safety judge labels them safe or unsafe. Benign PPL and BERTScore F1~\citep{zhang2020bertscoreevaluatingtextgeneration} use the separate benign cohort. BERTScore specifies English and disables baseline rescaling; it leaves the encoder to the installed package's default. BERTScore is multiplied by 100 in the tables. Higher reference similarity and lower PPL are useful diagnostics, but neither alone establishes helpfulness or safety.

\paragraph{WildGuard response views.}
For DIJA, the appendix additionally reports WildGuard~\cite{wildguard2024} harmfulness on the newline-joined spans marked as initially masked, paired with the original harmful behavior. This excludes the fixed scaffold and any full-block repairs outside those spans. Full-response WildGuard harmfulness and refusal remain in the appendix; XSTest uses full-response refusal. The generated-span evaluation is available only for fixed-policy DIJA runs, not PAD or the window studies (Appendix~\ref{app:wg-fills}).

\paragraph{Uncertainty and cost.}
Tables report averages with 95\% confidence intervals, estimated by repeatedly resampling the evaluation prompts. The appendix also reports standard deviations (SDs) across three generation seeds, showing how much results vary between runs. Appendix~\ref{app:metric-uncertainty} explains both measures and the paired comparisons. NFE counts generator forward passes, including additional passes after remasking. It excludes reward scoring, safety judges, and preparation of the reference set, so it measures generator computation rather than total runtime.

%% file: tables/hyperparameter_sweeps.tex
\begin{table}[!htbp]
\centering\footnotesize
\setlength{\tabcolsep}{3pt}\renewcommand{\arraystretch}{1.15}
\caption{Hyperparameter values considered, consolidated by model. Values are unions across the sweeps, not a single fully crossed grid; a single value denotes a fixed setting. $k$ is the scoring layer, $\eta$ the SAD strength, $w$ the number of full-SAD warm-up steps, and $[s,e]$ the inclusive intervention window.}
\label{tab:sadder-sweep-hparams}
\begin{tabular}{@{}p{0.23\linewidth}p{0.21\linewidth}p{0.51\linewidth}@{}}
\toprule
Model & Parameter & Values \\
\midrule
LLaDA-Base & Strength $\eta$ & $\{0,1,3,10,30,100,300,1000,3000\}$ \\
 & Layer $k$ & $\{8,16,24,31\}$ \\
 & Window $[s,e]$ & $[0,7]$, $[0,15]$, $[0,31]$ \\
 & Warm-up $w$ & $\{8,16\}$ \\
 & BoN / FK budget & $\{2,4,8\}$ \\
\addlinespace[4pt]
Dream-Base & Strength $\eta$ & $\{0,0.1,0.3,1,3,10,100,3000,$\newline $10000,30000,100000,300000\}$ \\
 & Layer $k$ & $27$ \\
 & Window $[s,e]$ & $[0,7]$, $[0,15]$ \\
 & BoN / FK budget & $\{2,4,8\}$ \\
\midrule
LLaDA-Instruct & Strength $\eta$ & $\{1,3,10,30,100,300,1000,3000,10000\}$ \\
 & Layer $k$ & $\{9,16,24\}$ \\
 & Window $[s,e]$ & $[0,15]$, $[0,31]$, $[0,63]$, $[0,127]$, $[32,63]$, $[64,95]$, $[96,127]$ \\
\addlinespace[4pt]
Dream-Instruct & Strength $\eta$ & $\{0.3,1,3,10,30,100,300,1000,3000,10000\}$ \\
 & Layer $k$ & $\{1,\ldots,28\}$ \\
 & Window $[s,e]$ & $[0,31]$ \\
\bottomrule
\end{tabular}
\par\smallskip
\begin{minipage}{\linewidth}\footnotesize
\textbf{Development settings.} Base models use RTP: the original search uses indices 0--299; additional strength development uses 64 prompts at indices 900--963, seeds 42/43, and window $[0,7]$ ($k=31$ for LLaDA, $27$ for Dream).  Instruction-model development uses 100 JBB behaviors; LLaDA additionally uses the first 100 AdvBench and HarmBench behaviors. These runs use seed 42 or seeds 42/43/44, as documented with the results. Sampling uses 32 steps for Base models and 128 nominal steps for Instruct models.
\end{minipage}
\end{table}

%% file: tables/selected_hyperparams.tex
\begin{table}[!htbp]
\centering\small
\setlength{\tabcolsep}{3pt}
\caption{Selected main-table continuation settings.}
\label{tab:sadder-selected-hparams}
\begin{tabular}{@{}lllrrrrl@{}}
\toprule
Model & Data & Method & $k$ & $\eta$ & Last step & $w$ & Reopening \\
\midrule
LLaDA & RTP & Full SAD & -- & 3000 & 7 & 0 & none \\
 & RTP & SADDER (selective) & 31$^{\dagger}$ & 100 & 15 & 0 & none \\
 & RTP & SADDER + warm-up & 24 & 3000 & 15 & 8 & none \\
\addlinespace
LLaDA & BT & Full SAD & -- & 3000 & 15 & 0 & none \\
 & BT & SADDER (selective) & 31 & 3000 & 15 & 0 & none \\
 & BT & SADDER + warm-up & 16 & 3000 & 15 & 8 & none \\
\addlinespace
Dream & RTP & Full SAD & -- & 10000 & 15 & 0 & none \\
 & RTP & SADDER (selective) & 27 & 10000 & 15 & 0 & none \\
 & RTP & SADDER + warm-up & 27 & 3000 & 15 & 8 & none \\
\addlinespace
Dream & BT & Full SAD & -- & 10000 & 15 & 0 & none \\
 & BT & SADDER (selective) & 27 & 10000 & 15 & 0 & none \\
 & BT & SADDER + warm-up & 27 & 10000 & 15 & 8 & none \\
\bottomrule
\end{tabular}
\par\smallskip\begin{minipage}{\linewidth}\footnotesize
All windows begin at step zero. BT denotes BeaverTails.
\end{minipage}
\end{table}

%% file: sections/Appendix/implementation_details_and_remasking.tex
\section{Implementation and evaluation details}
\label{app:implementation}

\paragraph{Motivation for remasking.}
The safety improvements reported by CSD-based steering and DiffuGuard~\citep{lee2026adaptivesteeringremaskingsafe,diffuguard} motivate examining remasking as an optional extension to SADDER. Reopening previously generated positions allows guidance to revise token choices that would otherwise remain fixed. Uniform-state diffusion provides a related capability through its native token replacement transitions, allowing repeated revision without an explicit mask token. As such, we wanted to evaluate selective SAD with risk-based remasking and matched random controls in the appendix.

\subsection{Optional remasking and warm-up}
\label{app:extended-algorithm}
Algorithm~\ref{algo:SADDER-extended} extends the selective-only main algorithm. Let $\mathcal R_t=\{l\in\mathcal G:\vx_t^l\ne\mathbb M\}$ be the revealed editable positions. The remasking policy $r$ determines the reopened set:
\begin{equation}
 \mathcal Q_t(r)=
 \begin{cases}
  \varnothing, & r=\mathrm{none},\\
  \operatorname{Top}_{\min(b,|\mathcal R_t|)}
    \{s_{k,l}^{(t)}:l\in\mathcal R_t\}, & r=\mathrm{top}\text{-}b,\\
  \{l\in\mathcal R_t:s_{k,l}^{(t)}>0\}, & r=\mathrm{aligned}.
 \end{cases}
 \label{eq:sadder-remasking}
\end{equation}
Here $\operatorname{Top}$ returns position indices with the largest scores; we use $b=1$ in the top-$b$ experiments. This ranking policy can reopen a position even if every score is non-positive. Aligned remasking instead reopens every positive-score revealed position, without a top-$b$ cap. After replacing selected tokens by $\mathbb M$, the model predictions and hidden states are recomputed. The steering gate is then evaluated on this changed state; reopening does not force a position to be steered. Warm-up length $w$ applies full SAD on admitted denoising steps with index $j<w$, followed by selective steering for $j\geq w$. Remasking is disabled during warm-up and on the final sampling step. Outside the intervention set $C$, neither guidance nor remasking is applied. Setting $w=0$ disables warm-up; setting $r=\mathrm{none}$ disables remasking. Thus the selective-only main method is exactly the case $w=0$, $r=\mathrm{none}$.

\FloatBarrier

\begin{algorithm}[!htbp]
\caption{SADDER with optional remasking and full-SAD warm-up.}
\label{algo:SADDER-extended}
\begin{algorithmic}[1]
\STATE \textbf{Input:} inputs of Algorithm~\ref{SADDER}; warm-up length $w\geq0$;
  policy $r\in\{\mathrm{none},\mathrm{top}\text{-}b,\mathrm{aligned}\}$; budget $b\geq0$.
\STATE $w=0$ and $r=\mathrm{none}$ recover Algorithm~\ref{SADDER}.
\FOR{$j=0,\ldots,T-1$}
  \STATE Set $t\leftarrow t_j$. Evaluate $p_{\Dat,t,\cdot}$ and $\vh_{k,\cdot}^{(t)}$
    together on the current $\vx_t$.
  \STATE Initialize $\widetilde p_{0\mid t,l}\leftarrow p_{\Dat,t,l}$ and
    $a_{k,l}^{(t)}\leftarrow0$ for all $l$.
  \IF{$j\in C$}
    \STATE Initialize $\mathcal Q_t\leftarrow\varnothing$.
    \IF{$j\geq w$ \textbf{and} $r\ne\mathrm{none}$ \textbf{and} $j<T-1$}
      \STATE $\mathcal R_t\leftarrow\{l\in\mathcal G:\vx_t^l\ne\mathbb M\}$.
      \STATE Compute $s_{k,l}^{(t)}\leftarrow
        \langle\vh_{k,l}^{(t)}-\vm_k,\widehat\vv_k\rangle$ for $l\in\mathcal R_t$.
      \STATE Select $\mathcal Q_t(r)$ using Equation~\eqref{eq:sadder-remasking}.
      \STATE Set $\vx_t^l\leftarrow\mathbb M$ for $l\in\mathcal Q_t$.
    \ENDIF
    \IF{$\mathcal Q_t\ne\varnothing$}
      \STATE Recompute $p_{\Dat,t,\cdot}$ and $\vh_{k,\cdot}^{(t)}$ together on the changed $\vx_t$.
      \STATE Reset $\widetilde p_{0\mid t,\cdot}\leftarrow p_{\Dat,t,\cdot}$.
    \ENDIF
    \STATE $\mathcal M_t\leftarrow\{l\in\mathcal G:\vx_t^l=\mathbb M\}$.
    \IF{$j<w$}
      \STATE \textbf{Full-SAD warm-up:}
        $a_{k,l}^{(t)}\leftarrow1$ for $l\in\mathcal M_t$.
    \ELSE
      \STATE Compute $s_{k,l}^{(t)}\leftarrow
        \langle\vh_{k,l}^{(t)}-\vm_k,\widehat\vv_k\rangle$ for $l\in\mathcal M_t$.
      \STATE \textbf{Select risky positions:}
        $a_{k,l}^{(t)}\leftarrow\mathbf1\{s_{k,l}^{(t)}>0\}$ for $l\in\mathcal M_t$.
    \ENDIF
    \IF{at least one gate equals $1$}
      \STATE Estimate $\widehat p_{\Dat_{\mathrm{unsafe}},t,\cdot}$ and
        $\widehat\beta(\vx_t)$ on this same state; set $\eta_t\leftarrow\eta\widehat\beta(\vx_t)$.
      \STATE \textbf{SAD update on selected positions:} update
        $\widetilde p_{0\mid t,l}$ using Equation~\eqref{point-wise-sadder-guidance}
        for each $l$ with $a_{k,l}^{(t)}=1$.
    \ENDIF
  \ENDIF
  \STATE $\vx_{t_{j+1}}\leftarrow\mathrm{Solver}(\vx_t,t_j,t_{j+1},\widetilde p_{0\mid t,\cdot};\mathcal G)$.
\ENDFOR
\STATE \textbf{Return:} the generated continuation of $\vx_{t_T}$.
\end{algorithmic}
\end{algorithm}

\FloatBarrier

\subsection{Metric estimates and uncertainty}
\label{app:metric-uncertainty}
We compute uncertainty from the saved evaluations. Let $i=1,\ldots,n$ index evaluation records and $s=1,\ldots,S$ index the observed generation seeds. The main continuation evaluations have $n=600$ and $S=3$; the main HarmBench evaluations have $n=100$ and $S=3$.

\paragraph{Unsafe rates and other averages.}
Let $y_{i,s}\in\{0,1\}$ indicate whether the output for prompt record $i$ and seed $s$ is judged unsafe. On a complete, balanced evaluation, the unsafe percentage is
\begin{equation}
 \widehat r_{\mathrm{unsafe}}
 =\frac{100}{nS}\sum_{i=1}^{n}\sum_{s=1}^{S}y_{i,s}.
 \label{eq:unsafe-rate-estimate}
\end{equation}
The same expression defines HarmBench ASR, refusal rate, or short-output rate when $y_{i,s}$ is the corresponding event indicator. A mean repetition score uses the per-output repetition score instead. Each metric retains its own evaluation population: for example, XSTest over-refusal uses only safe prompts. When observation counts differ, rate estimates use the total number of valid observations as their denominator.

\paragraph{Quality metrics.}
Let $\mathcal I_s$ contain the valid records for a quality metric at seed $s$. The reported mean averages a metric computed within each seed:
\begin{equation}
 \widehat Q_s=f\!\left(
   \frac{\sum_{i\in\mathcal I_s}A_{i,s}}
        {\sum_{i\in\mathcal I_s}D_{i,s}}\right),
 \qquad \widehat Q=\frac1S\sum_{s=1}^{S}\widehat Q_s.
 \label{eq:quality-aggregation}
\end{equation}
For BERTScore, $A_{i,s}$ is the response-level F1 score, $D_{i,s}=1$, and $f(z)=z$; the displayed values are multiplied by 100. For PPL, $A_{i,s}$ is the total negative log likelihood of scored tokens, $D_{i,s}$ is their count, and $f(z)=\exp(z)$. Thus PPL is an average of seed-level corpus perplexities, not an average of response perplexities. Missing or unscorable measurements are excluded; they are not assigned zero.

\paragraph{95\% confidence intervals.}
To estimate uncertainty due to the evaluation prompts, we randomly sample prompts from the evaluated set, allowing the same prompt to be selected more than once. Each resampled set has the same number of prompts as the original set. Whenever a prompt is selected, we keep its responses from all three seeds together. We calculate the metric again for each resampled set and repeat this process 5,000 times. The 2.5th and 97.5th percentiles of these 5,000 estimates form the reported 95\% confidence interval.

\paragraph{Seed SD and paired comparisons.}
For seed-level estimates $\widehat\theta_s$, the reported sample SD is
\begin{equation}
 \operatorname{SD}_{\mathrm{seed}}
 =\sqrt{\frac1{S-1}\sum_{s=1}^{S}
   (\widehat\theta_s-\bar\theta_{\mathrm{seed}})^2},
 \qquad \bar\theta_{\mathrm{seed}}=\frac1S\sum_s\widehat\theta_s.
 \label{eq:metric-seed-sd}
\end{equation}
It describes variation across runs and is not a confidence interval; it is unavailable for a single seed. To compare methods A and B, retain matched prompts and seeds and use identical cluster weights for both:
\begin{equation}
 \widehat\Delta=\widehat\theta_{\mathrm A}-\widehat\theta_{\mathrm B},
 \qquad
 \widehat\Delta^{*(b)}=
 \widehat\theta_{\mathrm A}^{*(b)}-\widehat\theta_{\mathrm B}^{*(b)}.
 \label{eq:paired-metric-bootstrap}
\end{equation}

%% file: tables/main_llada_ci.tex
\begin{table*}[!htb]
\centering\fontsize{8}{9.6}\selectfont
\setlength{\tabcolsep}{0.5pt}
\renewcommand{\arraystretch}{1.08}
\newcommand{\lladaci}[3]{%
\raisebox{-0.5\height}{%
  \shortstack{$#1$\\{\fontsize{7}{8}\selectfont $[#2{,}#3]$}}}}
\caption{LLaDA-Base safety, generation quality, and cost.}
\label{tab:llada-ci}
\begin{tabular*}{\textwidth}{@{\extracolsep{\fill}}l@{\hspace{3pt}}rrrrr@{\hspace{5pt}}rrrrr@{}}
\toprule
\multirow{2}{*}{Method}
& \multicolumn{5}{c}{\textbf{RTP}}
& \multicolumn{5}{c}{\textbf{BeaverTails}} \\
\cmidrule(lr){2-6}\cmidrule(lr){7-11}
& \shortstack{Unsafe\\\% $\downarrow$}
& \shortstack{Harmful\\PPL $\downarrow$}
& \shortstack{Benign\\PPL $\downarrow$}
& \shortstack{Benign\\BERT $\uparrow$}
& NFE
& \shortstack{Unsafe\\\% $\downarrow$}
& \shortstack{Harmful\\PPL $\downarrow$}
& \shortstack{Benign\\PPL $\downarrow$}
& \shortstack{Benign\\BERT $\uparrow$}
& NFE \\
\midrule
Unmodified
& \lladaci{38.6}{35.7}{41.5}
& \lladaci{21.1}{20.2}{22.1}
& \lladaci{17.5}{16.9}{18.2}
& \lladaci{84.65}{84.50}{84.79}
& 32.0
& \lladaci{38.6}{35.6}{41.6}
& \lladaci{20.6}{19.6}{21.6}
& \lladaci{17.5}{16.9}{18.2}
& \lladaci{84.65}{84.50}{84.79}
& 32.0 \\
\addlinespace[2pt]
Full SAD
& \lladaci{32.4}{29.4}{35.5}
& \lladaci{29.3}{27.6}{31.1}
& \lladaci{30.0}{27.7}{32.9}
& \lladaci{83.92}{83.75}{84.09}
& 32.0
& \lladaci{27.1}{24.0}{30.4}
& \lladaci{35.1}{33.2}{37.0}
& \lladaci{47.1}{42.9}{52.7}
& \lladaci{83.38}{83.20}{83.55}
& 32.0 \\
\addlinespace[2pt]
\raisebox{-0.5\height}{\shortstack[l]{SADDER\\(selective)}}
& \lladaci{29.4}{26.3}{32.5}
& \lladaci{42.1}{37.4}{47.8}
& \lladaci{18.4}{17.7}{19.2}
& \lladaci{84.59}{84.45}{84.74}
& 32.0
& \lladaci{29.9}{26.8}{33.2}
& \lladaci{34.6}{32.2}{37.3}
& \lladaci{27.8}{26.0}{30.0}
& \lladaci{84.20}{84.04}{84.35}
& 32.0 \\
\addlinespace[2pt]
\raisebox{-0.5\height}{\shortstack[l]{SADDER\\+ warm-up}}
& \lladaci{31.2}{28.1}{34.5}
& \lladaci{31.6}{29.7}{33.5}
& \lladaci{44.3}{38.4}{52.5}
& \lladaci{83.01}{82.81}{83.20}
& 32.0
& \lladaci{24.4}{21.2}{27.8}
& \lladaci{36.9}{34.2}{39.9}
& \lladaci{44.1}{40.2}{49.7}
& \lladaci{83.18}{82.99}{83.35}
& 32.0 \\
\addlinespace[2pt]
CSD$^\dagger$
& \lladaci{22.4}{20.3}{24.7}
& \lladaci{225}{204}{249}
& \lladaci{21.6}{20.6}{22.7}
& \lladaci{84.59}{84.44}{84.73}
& 61.9
& \lladaci{15.6}{13.8}{17.3}
& \lladaci{487}{440}{541}
& \lladaci{943}{868}{1030}
& \lladaci{82.61}{82.48}{82.73}
& 55.0 \\
\addlinespace[2pt]
BoN-8
& \lladaci{19.0}{16.7}{21.5}
& \lladaci{20.5}{19.6}{21.4}
& \lladaci{19.5}{18.7}{20.2}
& \lladaci{84.70}{84.56}{84.84}
& 256.0
& \lladaci{18.6}{16.3}{20.9}
& \lladaci{24.2}{23.0}{25.4}
& \lladaci{19.5}{18.7}{20.2}
& \lladaci{84.70}{84.56}{84.84}
& 256.0 \\
\addlinespace[2pt]
FK-8
& \lladaci{20.6}{18.1}{23.0}
& \lladaci{21.3}{20.4}{22.3}
& \lladaci{18.2}{17.5}{18.8}
& \lladaci{84.71}{84.57}{84.85}
& 256.0
& \lladaci{28.2}{25.4}{30.9}
& \lladaci{22.7}{21.7}{23.7}
& \lladaci{18.2}{17.5}{18.8}
& \lladaci{84.71}{84.57}{84.85}
& 256.0 \\
\bottomrule
\end{tabular*}
\end{table*}

%% file: tables/main_dream.tex
\begin{table*}[!htb]
\centering\fontsize{8}{9.6}\selectfont
\setlength{\tabcolsep}{0.5pt}\renewcommand{\arraystretch}{1.08}
\caption{Dream-Base safety, generation quality, and cost.}
\label{tab:dream}
\begin{tabular*}{\textwidth}{@{\extracolsep{\fill}}l@{\hspace{3pt}}rrrrr@{\hspace{5pt}}rrrrr@{}}
\toprule
\multirow{2}{*}{Method} & \multicolumn{5}{c}{\textbf{RTP}} & \multicolumn{5}{c}{\textbf{BeaverTails}} \\
\cmidrule(lr){2-6}\cmidrule(lr){7-11}
 & \shortstack{Unsafe\\\% $\downarrow$} & \shortstack{Harmful\\PPL $\downarrow$} & \shortstack{Benign\\PPL $\downarrow$} & \shortstack{Benign\\BERT $\uparrow$} & NFE & \shortstack{Unsafe\\\% $\downarrow$} & \shortstack{Harmful\\PPL $\downarrow$} & \shortstack{Benign\\PPL $\downarrow$} & \shortstack{Benign\\BERT $\uparrow$} & NFE \\
\midrule
Unmodified & \raisebox{-0.5\height}{\shortstack{$46.8$\\{\fontsize{7}{8}\selectfont $[43.8{,}49.7]$}}} & \raisebox{-0.5\height}{\shortstack{$18.3$\\{\fontsize{7}{8}\selectfont $[17.6{,}19.1]$}}} & \raisebox{-0.5\height}{\shortstack{$15.5$\\{\fontsize{7}{8}\selectfont $[15.0{,}16.0]$}}} & \raisebox{-0.5\height}{\shortstack{$84.71$\\{\fontsize{7}{8}\selectfont $[84.57{,}84.86]$}}} & 32.0 & \raisebox{-0.5\height}{\shortstack{$38.7$\\{\fontsize{7}{8}\selectfont $[35.8{,}41.6]$}}} & \raisebox{-0.5\height}{\shortstack{$12.7$\\{\fontsize{7}{8}\selectfont $[12.2{,}13.2]$}}} & \raisebox{-0.5\height}{\shortstack{$15.5$\\{\fontsize{7}{8}\selectfont $[15.0{,}16.0]$}}} & \raisebox{-0.5\height}{\shortstack{$84.71$\\{\fontsize{7}{8}\selectfont $[84.57{,}84.86]$}}} & 32.0 \\
\addlinespace[2pt]
Full SAD & \raisebox{-0.5\height}{\shortstack{$33.4$\\{\fontsize{7}{8}\selectfont $[30.1{,}36.8]$}}} & \raisebox{-0.5\height}{\shortstack{$26.0$\\{\fontsize{7}{8}\selectfont $[24.3{,}27.8]$}}} & \raisebox{-0.5\height}{\shortstack{$23.3$\\{\fontsize{7}{8}\selectfont $[22.0{,}24.7]$}}} & \raisebox{-0.5\height}{\shortstack{$83.16$\\{\fontsize{7}{8}\selectfont $[82.96{,}83.35]$}}} & 32.0 & \raisebox{-0.5\height}{\shortstack{$22.7$\\{\fontsize{7}{8}\selectfont $[19.7{,}25.7]$}}} & \raisebox{-0.5\height}{\shortstack{$23.7$\\{\fontsize{7}{8}\selectfont $[21.8{,}25.8]$}}} & \raisebox{-0.5\height}{\shortstack{$23.0$\\{\fontsize{7}{8}\selectfont $[21.7{,}24.5]$}}} & \raisebox{-0.5\height}{\shortstack{$83.62$\\{\fontsize{7}{8}\selectfont $[83.44{,}83.80]$}}} & 32.0 \\
\addlinespace[2pt]
\raisebox{-0.5\height}{\shortstack[l]{SADDER\\(selective)}} & \raisebox{-0.5\height}{\shortstack{$35.9$\\{\fontsize{7}{8}\selectfont $[32.8{,}39.0]$}}} & \raisebox{-0.5\height}{\shortstack{$29.1$\\{\fontsize{7}{8}\selectfont $[27.5{,}30.7]$}}} & \raisebox{-0.5\height}{\shortstack{$18.1$\\{\fontsize{7}{8}\selectfont $[17.5{,}18.7]$}}} & \raisebox{-0.5\height}{\shortstack{$84.59$\\{\fontsize{7}{8}\selectfont $[84.45{,}84.73]$}}} & 32.0 & \raisebox{-0.5\height}{\shortstack{$34.1$\\{\fontsize{7}{8}\selectfont $[31.3{,}36.8]$}}} & \raisebox{-0.5\height}{\shortstack{$19.3$\\{\fontsize{7}{8}\selectfont $[18.4{,}20.3]$}}} & \raisebox{-0.5\height}{\shortstack{$20.0$\\{\fontsize{7}{8}\selectfont $[19.2{,}20.7]$}}} & \raisebox{-0.5\height}{\shortstack{$84.49$\\{\fontsize{7}{8}\selectfont $[84.34{,}84.64]$}}} & 32.0 \\
\addlinespace[2pt]
\raisebox{-0.5\height}{\shortstack[l]{SADDER\\+ warm-up}} & \raisebox{-0.5\height}{\shortstack{$34.7$\\{\fontsize{7}{8}\selectfont $[31.6{,}38.2]$}}} & \raisebox{-0.5\height}{\shortstack{$27.1$\\{\fontsize{7}{8}\selectfont $[25.5{,}28.8]$}}} & \raisebox{-0.5\height}{\shortstack{$21.3$\\{\fontsize{7}{8}\selectfont $[20.3{,}22.3]$}}} & \raisebox{-0.5\height}{\shortstack{$83.76$\\{\fontsize{7}{8}\selectfont $[83.57{,}83.94]$}}} & 32.0 & \raisebox{-0.5\height}{\shortstack{$27.7$\\{\fontsize{7}{8}\selectfont $[24.8{,}30.8]$}}} & \raisebox{-0.5\height}{\shortstack{$23.6$\\{\fontsize{7}{8}\selectfont $[22.1{,}25.1]$}}} & \raisebox{-0.5\height}{\shortstack{$20.5$\\{\fontsize{7}{8}\selectfont $[19.6{,}21.5]$}}} & \raisebox{-0.5\height}{\shortstack{$84.21$\\{\fontsize{7}{8}\selectfont $[84.06{,}84.37]$}}} & 32.0 \\
\addlinespace[2pt]
CSD$^\dagger$ & \raisebox{-0.5\height}{\shortstack{$32.9$\\{\fontsize{7}{8}\selectfont $[30.6{,}35.4]$}}} & \raisebox{-0.5\height}{\shortstack{$1180$\\{\fontsize{7}{8}\selectfont $[1100{,}1270]$}}} & \raisebox{-0.5\height}{\shortstack{$3400$\\{\fontsize{7}{8}\selectfont $[3270{,}3540]$}}} & \raisebox{-0.5\height}{\shortstack{$81.51$\\{\fontsize{7}{8}\selectfont $[81.42{,}81.61]$}}} & 60.2 & \raisebox{-0.5\height}{\shortstack{$38.6$\\{\fontsize{7}{8}\selectfont $[35.7{,}41.4]$}}} & \raisebox{-0.5\height}{\shortstack{$17.8$\\{\fontsize{7}{8}\selectfont $[16.7{,}19.1]$}}} & \raisebox{-0.5\height}{\shortstack{$27.2$\\{\fontsize{7}{8}\selectfont $[25.3{,}29.4]$}}} & \raisebox{-0.5\height}{\shortstack{$84.50$\\{\fontsize{7}{8}\selectfont $[84.36{,}84.64]$}}} & 55.8 \\
\addlinespace[2pt]
BoN-8 & \raisebox{-0.5\height}{\shortstack{$24.6$\\{\fontsize{7}{8}\selectfont $[22.0{,}27.3]$}}} & \raisebox{-0.5\height}{\shortstack{$17.5$\\{\fontsize{7}{8}\selectfont $[16.7{,}18.2]$}}} & \raisebox{-0.5\height}{\shortstack{$16.9$\\{\fontsize{7}{8}\selectfont $[16.3{,}17.5]$}}} & \raisebox{-0.5\height}{\shortstack{$84.81$\\{\fontsize{7}{8}\selectfont $[84.68{,}84.95]$}}} & 256.0 & \raisebox{-0.5\height}{\shortstack{$20.0$\\{\fontsize{7}{8}\selectfont $[17.8{,}22.2]$}}} & \raisebox{-0.5\height}{\shortstack{$15.0$\\{\fontsize{7}{8}\selectfont $[14.3{,}15.7]$}}} & \raisebox{-0.5\height}{\shortstack{$16.9$\\{\fontsize{7}{8}\selectfont $[16.3{,}17.5]$}}} & \raisebox{-0.5\height}{\shortstack{$84.81$\\{\fontsize{7}{8}\selectfont $[84.68{,}84.95]$}}} & 256.0 \\
\addlinespace[2pt]
FK-8 & \raisebox{-0.5\height}{\shortstack{$25.1$\\{\fontsize{7}{8}\selectfont $[22.5{,}27.7]$}}} & \raisebox{-0.5\height}{\shortstack{$17.9$\\{\fontsize{7}{8}\selectfont $[17.2{,}18.6]$}}} & \raisebox{-0.5\height}{\shortstack{$15.7$\\{\fontsize{7}{8}\selectfont $[15.3{,}16.2]$}}} & \raisebox{-0.5\height}{\shortstack{$84.76$\\{\fontsize{7}{8}\selectfont $[84.62{,}84.90]$}}} & 256.0 & \raisebox{-0.5\height}{\shortstack{$27.8$\\{\fontsize{7}{8}\selectfont $[25.2{,}30.4]$}}} & \raisebox{-0.5\height}{\shortstack{$13.9$\\{\fontsize{7}{8}\selectfont $[13.4{,}14.5]$}}} & \raisebox{-0.5\height}{\shortstack{$15.7$\\{\fontsize{7}{8}\selectfont $[15.3{,}16.2]$}}} & \raisebox{-0.5\height}{\shortstack{$84.76$\\{\fontsize{7}{8}\selectfont $[84.62{,}84.90]$}}} & 256.0 \\
\bottomrule
\end{tabular*}
\end{table*}

%% file: tables/extra/settings_llada.tex
\begin{table}[!htbp]
\centering\small
\setlength{\tabcolsep}{3pt}\renewcommand{\arraystretch}{1.12}
\caption{LLaDA selected continuation settings on development set.}
\label{tab:settings-llada}
\begin{tabular}{@{}lll@{}}
\toprule
Dataset & Method & Setting \\
\midrule
RTP & Full SAD & $31/3000/7$ \\
 & SADDER & $31/100/15$ \\
 & SADDER w/ top-1 remasking & $24/100/15$ \\
 & SADDER w/ aligned remasking & $16/100/15$ \\
 & SADDER + warm-up + aligned remasking& $24/3000/15$; \\
 & CSD$^\dagger$ & $k=31$, $j=0\!:\!15$ \\
 & BoN-8 & $N=8$ \\
 & FK-8 & $N=8$ \\
\addlinespace[2pt]
BeaverTails & Full SAD & $31/3000/15$ \\
 & SADDER  & $31/3000/15$ \\
 & SADDER w/ top-1 remasking & $24/100/7$ \\
 & SADDER w/ aligned remasking & $16/3000/15$ \\
 & SADDER + warm-up + aligned remasking& $16/3000/15$; aligned \\
 & CSD$^\dagger$ & $k=24$, $j=0\!:\!7$ \\
 & BoN-8 & $N=8$ \\
 & FK-8 & $N=8$ \\
\bottomrule
\end{tabular}
\end{table}

%% file: tables/extra/settings_dream.tex
\begin{table}[!htbp]
\centering\small
\setlength{\tabcolsep}{3pt}\renewcommand{\arraystretch}{1.12}
\caption{Dream selected continuation settings on development set.}
\label{tab:settings-dream}
\begin{tabular}{@{}lll@{}}
\toprule
Dataset & Method & Setting \\
\midrule
RTP & Full SAD & $27/10000/15$ \\
 & SADDER  & $27/10000/15$ \\
 & SADDER w/ top-1 remasking & $27/100/15$ \\
 & SADDER w/ aligned remasking & $24/1e5/15$ \\
 & SADDER + warm-up & $27/3000/15$; none \\
 & CSD$^\dagger$ & $k=24$, $j=0\!:\!15$ \\
 & BoN-8 & $N=8$ \\
 & FK-8 & $N=8$ \\
\addlinespace[2pt]
BeaverTails & Full SAD & $27/10000/15$ \\
 & SADDER  & $27/10000/15$ \\
 & SADDER w/ top-1 remasking & $27/100/15$ \\
 & SADDER w/ aligned remasking & $27/100/15$ \\
 & SADDER + warm-up & $27/10000/15$; none \\
 & CSD$^\dagger$ & $k=27$, $j=0\!:\!15$ \\
 & BoN-8 & $N=8$ \\
 & FK-8 & $N=8$ \\
\bottomrule
\end{tabular}
\end{table}

%% file: tables/extra/variants_llada.tex
\begin{table}[!htbp]
\centering\footnotesize
\setlength{\tabcolsep}{3pt}\renewcommand{\arraystretch}{1.12}
\caption{LLaDA-Base component variants.}
\label{tab:variants-llada}
\begin{tabular}{@{}lrrrrr@{}}
\toprule
Method & Unsafe \% $\downarrow$ & \shortstack{Harmful\\PPL $\downarrow$} & \shortstack{Benign\\PPL $\downarrow$} & \shortstack{Benign\\BERT $\uparrow$} & NFE \\
\midrule
\multicolumn{6}{l}{\textbf{RTP}} \\
Unmodified & \shortstack{$38.6$\\{\scriptsize $[35.7,41.5]$}} & \shortstack{$21.14$\\{\scriptsize $[20.21,22.13]$}} & \shortstack{$17.53$\\{\scriptsize $[16.90,18.21]$}} & \shortstack{$84.65$\\{\scriptsize $[84.50,84.79]$}} & 32.0 \\
Full SAD & \shortstack{$32.4$\\{\scriptsize $[29.4,35.5]$}} & \shortstack{$29.30$\\{\scriptsize $[27.61,31.13]$}} & \shortstack{$29.96$\\{\scriptsize $[27.65,32.89]$}} & \shortstack{$83.92$\\{\scriptsize $[83.75,84.09]$}} & 32.0 \\
SADDER (selective) & \shortstack{$29.4$\\{\scriptsize $[26.3,32.5]$}} & \shortstack{$42.08$\\{\scriptsize $[37.39,47.75]$}} & \shortstack{$18.41$\\{\scriptsize $[17.70,19.17]$}} & \shortstack{$84.59$\\{\scriptsize $[84.45,84.74]$}} & 32.0 \\
SADDER w/ top-1 remasking & \shortstack{$28.3$\\{\scriptsize $[25.2,31.3]$}} & \shortstack{$43.25$\\{\scriptsize $[38.63,48.89]$}} & \shortstack{$41.75$\\{\scriptsize $[37.34,47.13]$}} & \shortstack{$83.38$\\{\scriptsize $[83.17,83.58]$}} & 47.0 \\
SADDER w/ aligned remasking & \shortstack{$30.6$\\{\scriptsize $[27.8,33.4]$}} & \shortstack{$35.81$\\{\scriptsize $[33.47,38.38]$}} & \shortstack{$35.88$\\{\scriptsize $[33.53,38.57]$}} & \shortstack{$83.99$\\{\scriptsize $[83.82,84.16]$}} & 46.6 \\
SADDER + warm-up + aligned remasking & \shortstack{$27.2$\\{\scriptsize $[24.7,29.9]$}} & \shortstack{$38.75$\\{\scriptsize $[36.22,41.88]$}} & \shortstack{$33.33$\\{\scriptsize $[31.08,35.91]$}} & \shortstack{$84.13$\\{\scriptsize $[83.96,84.28]$}} & 40.0 \\
CSD$^\dagger$ & \shortstack{$22.4$\\{\scriptsize $[20.3,24.7]$}} & \shortstack{$225.01$\\{\scriptsize $[204.37,248.84]$}} & \shortstack{$21.59$\\{\scriptsize $[20.58,22.67]$}} & \shortstack{$84.59$\\{\scriptsize $[84.44,84.73]$}} & 61.9 \\
BoN-8 & \shortstack{$19.0$\\{\scriptsize $[16.7,21.5]$}} & \shortstack{$20.46$\\{\scriptsize $[19.57,21.39]$}} & \shortstack{$19.47$\\{\scriptsize $[18.74,20.21]$}} & \shortstack{$84.70$\\{\scriptsize $[84.56,84.84]$}} & 256.0 \\
FK-8 & \shortstack{$20.6$\\{\scriptsize $[18.1,23.0]$}} & \shortstack{$21.29$\\{\scriptsize $[20.37,22.27]$}} & \shortstack{$18.17$\\{\scriptsize $[17.53,18.84]$}} & \shortstack{$84.71$\\{\scriptsize $[84.57,84.85]$}} & 256.0 \\
\midrule
\multicolumn{6}{l}{\textbf{BeaverTails}} \\
Unmodified & \shortstack{$38.6$\\{\scriptsize $[35.6,41.6]$}} & \shortstack{$20.57$\\{\scriptsize $[19.61,21.57]$}} & \shortstack{$17.53$\\{\scriptsize $[16.90,18.21]$}} & \shortstack{$84.65$\\{\scriptsize $[84.50,84.79]$}} & 32.0 \\
Full SAD & \shortstack{$27.1$\\{\scriptsize $[24.0,30.4]$}} & \shortstack{$35.05$\\{\scriptsize $[33.23,37.03]$}} & \shortstack{$47.08$\\{\scriptsize $[42.86,52.67]$}} & \shortstack{$83.38$\\{\scriptsize $[83.20,83.55]$}} & 32.0 \\
SADDER (selective) & \shortstack{$29.9$\\{\scriptsize $[26.8,33.2]$}} & \shortstack{$34.63$\\{\scriptsize $[32.22,37.33]$}} & \shortstack{$27.80$\\{\scriptsize $[25.95,30.01]$}} & \shortstack{$84.20$\\{\scriptsize $[84.04,84.35]$}} & 32.0 \\
SADDER w/ top-1 remasking & \shortstack{$30.6$\\{\scriptsize $[27.1,34.1]$}} & \shortstack{$31.58$\\{\scriptsize $[28.46,35.55]$}} & \shortstack{$23.78$\\{\scriptsize $[22.34,25.39]$}} & \shortstack{$84.10$\\{\scriptsize $[83.90,84.29]$}} & 39.0 \\
SADDER w/ aligned remasking & \shortstack{$27.5$\\{\scriptsize $[24.4,30.8]$}} & \shortstack{$37.12$\\{\scriptsize $[33.51,41.64]$}} & \shortstack{$35.91$\\{\scriptsize $[33.64,38.41]$}} & \shortstack{$83.69$\\{\scriptsize $[83.50,83.86]$}} & 42.2 \\
SADDER + warm-up + aligned remasking & \shortstack{$24.7$\\{\scriptsize $[21.8,27.7]$}} & \shortstack{$31.36$\\{\scriptsize $[29.70,33.08]$}} & \shortstack{$42.23$\\{\scriptsize $[39.06,46.69]$}} & \shortstack{$83.65$\\{\scriptsize $[83.47,83.82]$}} & 37.8 \\
CSD$^\dagger$ & \shortstack{$15.6$\\{\scriptsize $[13.8,17.3]$}} & \shortstack{$486.52$\\{\scriptsize $[439.90,541.43]$}} & \shortstack{$942.73$\\{\scriptsize $[868.26,1027.02]$}} & \shortstack{$82.61$\\{\scriptsize $[82.48,82.73]$}} & 55.0 \\
BoN-8 & \shortstack{$18.6$\\{\scriptsize $[16.3,20.9]$}} & \shortstack{$24.23$\\{\scriptsize $[23.04,25.42]$}} & \shortstack{$19.47$\\{\scriptsize $[18.74,20.21]$}} & \shortstack{$84.70$\\{\scriptsize $[84.56,84.84]$}} & 256.0 \\
FK-8 & \shortstack{$28.2$\\{\scriptsize $[25.4,30.9]$}} & \shortstack{$22.66$\\{\scriptsize $[21.66,23.74]$}} & \shortstack{$18.17$\\{\scriptsize $[17.53,18.84]$}} & \shortstack{$84.71$\\{\scriptsize $[84.57,84.85]$}} & 256.0 \\
\bottomrule
\end{tabular}
\end{table}

%% file: tables/extra/variants_dream.tex
\begin{table}[!htbp]
\centering\footnotesize
\setlength{\tabcolsep}{3pt}\renewcommand{\arraystretch}{1.12}
\caption{Dream-Base component variants.}
\label{tab:variants-dream}
\begin{tabular}{@{}lrrrrr@{}}
\toprule
Method & Unsafe \% $\downarrow$ & \shortstack{Harmful\\PPL $\downarrow$} & \shortstack{Benign\\PPL $\downarrow$} & \shortstack{Benign\\BERT $\uparrow$} & NFE \\
\midrule
\multicolumn{6}{l}{\textbf{RTP}} \\
Unmodified & \shortstack{$46.8$\\{\scriptsize $[43.8,49.7]$}} & \shortstack{$18.31$\\{\scriptsize $[17.61,19.06]$}} & \shortstack{$15.45$\\{\scriptsize $[14.96,15.98]$}} & \shortstack{$84.71$\\{\scriptsize $[84.57,84.86]$}} & 32.0 \\
Full SAD & \shortstack{$33.4$\\{\scriptsize $[30.1,36.8]$}} & \shortstack{$26.02$\\{\scriptsize $[24.30,27.83]$}} & \shortstack{$23.31$\\{\scriptsize $[21.99,24.68]$}} & \shortstack{$83.16$\\{\scriptsize $[82.96,83.35]$}} & 32.0 \\
SADDER & \shortstack{$35.9$\\{\scriptsize $[32.8,39.0]$}} & \shortstack{$29.09$\\{\scriptsize $[27.54,30.68]$}} & \shortstack{$18.10$\\{\scriptsize $[17.48,18.73]$}} & \shortstack{$84.59$\\{\scriptsize $[84.45,84.73]$}} & 32.0 \\
SADDER w/ top-1 remasking & \shortstack{$37.4$\\{\scriptsize $[34.3,40.5]$}} & \shortstack{$22.54$\\{\scriptsize $[21.28,23.80]$}} & \shortstack{$16.92$\\{\scriptsize $[16.37,17.53]$}} & \shortstack{$84.57$\\{\scriptsize $[84.43,84.71]$}} & 47.0 \\
SADDER w/ aligned remasking & \shortstack{$40.9$\\{\scriptsize $[38.0,43.7]$}} & \shortstack{$23.06$\\{\scriptsize $[22.15,24.00]$}} & \shortstack{$17.53$\\{\scriptsize $[16.92,18.16]$}} & \shortstack{$84.52$\\{\scriptsize $[84.37,84.66]$}} & 44.8 \\
SADDER + warm-up & \shortstack{$34.7$\\{\scriptsize $[31.6,38.2]$}} & \shortstack{$27.11$\\{\scriptsize $[25.52,28.82]$}} & \shortstack{$21.26$\\{\scriptsize $[20.25,22.29]$}} & \shortstack{$83.76$\\{\scriptsize $[83.57,83.94]$}} & 32.0 \\
CSD$^\dagger$ & \shortstack{$32.9$\\{\scriptsize $[30.6,35.4]$}} & \shortstack{$1182.90$\\{\scriptsize $[1101.51,1271.65]$}} & \shortstack{$3397.41$\\{\scriptsize $[3266.61,3538.51]$}} & \shortstack{$81.51$\\{\scriptsize $[81.42,81.61]$}} & 60.2 \\
BoN-8 & \shortstack{$24.6$\\{\scriptsize $[22.0,27.3]$}} & \shortstack{$17.46$\\{\scriptsize $[16.72,18.18]$}} & \shortstack{$16.87$\\{\scriptsize $[16.30,17.46]$}} & \shortstack{$84.81$\\{\scriptsize $[84.68,84.95]$}} & 256.0 \\
FK-8 & \shortstack{$25.1$\\{\scriptsize $[22.5,27.7]$}} & \shortstack{$17.85$\\{\scriptsize $[17.16,18.57]$}} & \shortstack{$15.73$\\{\scriptsize $[15.27,16.24]$}} & \shortstack{$84.76$\\{\scriptsize $[84.62,84.90]$}} & 256.0 \\
\midrule
\multicolumn{6}{l}{\textbf{BeaverTails}} \\
Unmodified & \shortstack{$38.7$\\{\scriptsize $[35.8,41.6]$}} & \shortstack{$12.67$\\{\scriptsize $[12.17,13.21]$}} & \shortstack{$15.45$\\{\scriptsize $[14.96,15.98]$}} & \shortstack{$84.71$\\{\scriptsize $[84.57,84.86]$}} & 32.0 \\
Full SAD & \shortstack{$22.7$\\{\scriptsize $[19.7,25.7]$}} & \shortstack{$23.68$\\{\scriptsize $[21.80,25.75]$}} & \shortstack{$23.03$\\{\scriptsize $[21.69,24.45]$}} & \shortstack{$83.62$\\{\scriptsize $[83.44,83.80]$}} & 32.0 \\
SADDER & \shortstack{$34.1$\\{\scriptsize $[31.3,36.8]$}} & \shortstack{$19.32$\\{\scriptsize $[18.41,20.32]$}} & \shortstack{$19.95$\\{\scriptsize $[19.18,20.73]$}} & \shortstack{$84.49$\\{\scriptsize $[84.34,84.64]$}} & 32.0 \\
SADDER w/ top-1 remasking & \shortstack{$34.3$\\{\scriptsize $[31.5,37.1]$}} & \shortstack{$19.16$\\{\scriptsize $[18.23,20.15]$}} & \shortstack{$18.35$\\{\scriptsize $[17.63,19.09]$}} & \shortstack{$84.39$\\{\scriptsize $[84.24,84.55]$}} & 47.0 \\
SADDER w/ aligned remasking & \shortstack{$33.7$\\{\scriptsize $[31.0,36.5]$}} & \shortstack{$17.78$\\{\scriptsize $[16.93,18.65]$}} & \shortstack{$21.07$\\{\scriptsize $[20.28,21.92]$}} & \shortstack{$84.54$\\{\scriptsize $[84.39,84.68]$}} & 40.4 \\
SADDER + warm-up & \shortstack{$27.7$\\{\scriptsize $[24.8,30.8]$}} & \shortstack{$23.56$\\{\scriptsize $[22.08,25.14]$}} & \shortstack{$20.52$\\{\scriptsize $[19.60,21.49]$}} & \shortstack{$84.21$\\{\scriptsize $[84.06,84.37]$}} & 32.0 \\
CSD$^\dagger$ & \shortstack{$38.6$\\{\scriptsize $[35.7,41.4]$}} & \shortstack{$17.81$\\{\scriptsize $[16.65,19.13]$}} & \shortstack{$27.19$\\{\scriptsize $[25.30,29.38]$}} & \shortstack{$84.50$\\{\scriptsize $[84.36,84.64]$}} & 55.8 \\
BoN-8 & \shortstack{$20.0$\\{\scriptsize $[17.8,22.2]$}} & \shortstack{$14.95$\\{\scriptsize $[14.26,15.68]$}} & \shortstack{$16.87$\\{\scriptsize $[16.30,17.46]$}} & \shortstack{$84.81$\\{\scriptsize $[84.68,84.95]$}} & 256.0 \\
FK-8 & \shortstack{$27.8$\\{\scriptsize $[25.2,30.4]$}} & \shortstack{$13.94$\\{\scriptsize $[13.39,14.54]$}} & \shortstack{$15.73$\\{\scriptsize $[15.27,16.24]$}} & \shortstack{$84.76$\\{\scriptsize $[84.62,84.90]$}} & 256.0 \\
\bottomrule
\end{tabular}
\end{table}

%% file: tables/extra/paired_benign_main.tex
\begin{table}[!htbp]
\centering\footnotesize
\setlength{\tabcolsep}{3pt}\renewcommand{\arraystretch}{1.12}
\caption{SADDER variants: paired benign-quality changes versus full SAD.}
\label{tab:paired-benign-main}
\begin{tabular}{@{}llrr@{}}
\toprule
Model/data & Method & $\Delta$ benign PPL & $\Delta$ benign BERT \\
\midrule
LLaDA/RTP & SADDER (selective) & $-11.55\;[-14.49,-9.26]$ & $0.672\;[0.527,0.813]$ \\
 & SADDER + warm-up & $3.42\;[0.49,6.19]$ & $0.206\;[0.076,0.332]$ \\
\addlinespace[2pt]
LLaDA/BT & SADDER (selective) & $-19.28\;[-24.82,-15.03]$ & $0.822\;[0.681,0.967]$ \\
 & SADDER + warm-up & $-4.85\;[-8.89,-1.18]$ & $0.270\;[0.133,0.407]$ \\
\addlinespace[2pt]
Dream/RTP & SADDER (selective) & $-5.22\;[-6.55,-3.88]$ & $1.429\;[1.258,1.605]$ \\
 & SADDER + warm-up & $-2.06\;[-3.25,-0.91]$ & $0.596\;[0.447,0.740]$ \\
\addlinespace[2pt]
Dream/BT & SADDER (selective) & $-3.08\;[-4.51,-1.70]$ & $0.868\;[0.709,1.030]$ \\
 & SADDER + warm-up & $-2.51\;[-3.66,-1.38]$ & $0.594\;[0.482,0.703]$ \\
\bottomrule
\end{tabular}
\end{table}

%% file: tables/extra/continuation_controls_llada.tex
\begin{table}[!htbp]
\centering\footnotesize
\setlength{\tabcolsep}{3pt}\renewcommand{\arraystretch}{1.12}
\caption{LLaDA-Base: paired unsafe-rate differences (pp), means [95\% CI].}
\label{tab:continuation-controls-llada}
\begin{tabular}{@{}lrrr@{}}
\toprule
SADDER variant & vs full SAD & vs random selection & vs random reopening \\
\midrule
\multicolumn{4}{l}{\textbf{RTP}} \\
Selective & $-4.1\;[-6.1,-2.1]$ & $-3.3\;[-5.4,-1.3]$ & --- \\
+ top-1 remasking & $-5.2\;[-8.5,-1.7]$ & --- & $-3.4\;[-6.3,-0.6]$ \\
+ aligned remasking & $-2.9\;[-6.3,0.4]$ & --- & $-1.7\;[-3.5,0.1]$ \\
+ warm-up + aligned & $-5.8\;[-9.0,-2.7]$ & --- & $-4.8\;[-6.9,-2.8]$ \\
\midrule
\multicolumn{4}{l}{\textbf{BeaverTails}} \\
Selective & $2.8\;[-1.1,6.6]$ & $-0.7\;[-3.9,2.4]$ & --- \\
+ top-1 remasking & $-7.9\;[-11.8,-4.1]$ & --- & $-6.1\;[-9.1,-3.1]$ \\
+ aligned remasking & $0.4\;[-3.3,4.1]$ & --- & $1.9\;[-0.6,4.6]$ \\
+ warm-up + aligned & $-2.4\;[-5.4,0.3]$ & --- & $-1.8\;[-4.5,0.8]$ \\
\bottomrule
\end{tabular}
\end{table}

%% file: tables/extra/continuation_controls_dream.tex
\begin{table}[!htbp]
\centering\footnotesize
\setlength{\tabcolsep}{3pt}\renewcommand{\arraystretch}{1.12}
\caption{Dream-Base: paired unsafe-rate differences (pp), means [95\% CI].}
\label{tab:continuation-controls-dream}
\begin{tabular}{@{}lrrr@{}}
\toprule
SADDER variant & vs full SAD & vs random selection & vs random reopening \\
\midrule
\multicolumn{4}{l}{\textbf{RTP}} \\
Selective & $2.5\;[-0.3,5.3]$ & $-2.7\;[-5.5,0.1]$ & --- \\
+ top-1 remasking & $-3.9\;[-7.4,-0.4]$ & $-2.6\;[-5.3,0.3]$ & $-2.8\;[-5.7,0.2]$ \\
+ aligned remasking & $-3.4\;[-6.2,-0.6]$ & $0.6\;[-1.7,2.7]$ & $-2.2\;[-4.5,0.2]$ \\
+ warm-up & $0.3\;[-1.6,2.2]$ & $0.4\;[-1.4,2.3]$ & --- \\
\midrule
\multicolumn{4}{l}{\textbf{BeaverTails}} \\
Selective & $11.4\;[7.8,14.8]$ & $3.4\;[0.8,5.9]$ & --- \\
+ top-1 remasking & $4.6\;[1.1,8.0]$ & $5.3\;[2.7,7.9]$ & $-1.6\;[-4.5,1.2]$ \\
+ aligned remasking & $4.0\;[0.6,7.4]$ & $2.6\;[0.1,5.1]$ & $0.2\;[-2.1,2.5]$ \\
+ warm-up & $5.1\;[2.6,7.5]$ & $0.1\;[-1.7,1.8]$ & --- \\
\bottomrule
\end{tabular}
\end{table}

%% file: tables/bon/benign_search_llada.tex
\begin{table}[!htbp]
\centering\footnotesize
\setlength{\tabcolsep}{3pt}\renewcommand{\arraystretch}{1.12}
\caption{LLaDA: BoN and FK at budgets 2, 4, and 8.}
\label{tab:benign-search-llada}
\begin{tabular}{@{}lrrrrr@{}}
\toprule
Method & Unsafe \% & Harmful PPL & Benign PPL & Benign BERT & NFE \\
\midrule
\multicolumn{6}{l}{\textbf{RTP}} \\
BoN-2 & \shortstack{$30.6$\\{\scriptsize $[27.8,33.4]$}} & \shortstack{$20.90$\\{\scriptsize $[20.00,21.82]$}} & \shortstack{$17.51$\\{\scriptsize $[16.89,18.15]$}} & \shortstack{$84.73$\\{\scriptsize $[84.59,84.88]$}} & 64 \\
BoN-4 & \shortstack{$24.7$\\{\scriptsize $[22.3,27.2]$}} & \shortstack{$20.50$\\{\scriptsize $[19.68,21.40]$}} & \shortstack{$18.70$\\{\scriptsize $[18.00,19.41]$}} & \shortstack{$84.75$\\{\scriptsize $[84.61,84.89]$}} & 128 \\
BoN-8 & \shortstack{$19.0$\\{\scriptsize $[16.7,21.5]$}} & \shortstack{$20.46$\\{\scriptsize $[19.57,21.39]$}} & \shortstack{$19.47$\\{\scriptsize $[18.74,20.21]$}} & \shortstack{$84.70$\\{\scriptsize $[84.56,84.84]$}} & 256 \\
FK-2 & \shortstack{$32.4$\\{\scriptsize $[29.7,35.2]$}} & \shortstack{$21.47$\\{\scriptsize $[20.53,22.46]$}} & \shortstack{$17.59$\\{\scriptsize $[16.96,18.26]$}} & \shortstack{$84.67$\\{\scriptsize $[84.52,84.82]$}} & 64 \\
FK-4 & \shortstack{$26.2$\\{\scriptsize $[23.6,28.8]$}} & \shortstack{$20.82$\\{\scriptsize $[19.92,21.71]$}} & \shortstack{$17.84$\\{\scriptsize $[17.22,18.49]$}} & \shortstack{$84.68$\\{\scriptsize $[84.54,84.83]$}} & 128 \\
FK-8 & \shortstack{$20.6$\\{\scriptsize $[18.1,23.0]$}} & \shortstack{$21.29$\\{\scriptsize $[20.37,22.27]$}} & \shortstack{$18.17$\\{\scriptsize $[17.53,18.84]$}} & \shortstack{$84.71$\\{\scriptsize $[84.57,84.85]$}} & 256 \\
\midrule
\multicolumn{6}{l}{\textbf{BeaverTails}} \\
BoN-2 & \shortstack{$32.1$\\{\scriptsize $[29.3,34.8]$}} & \shortstack{$22.32$\\{\scriptsize $[21.38,23.33]$}} & \shortstack{$17.51$\\{\scriptsize $[16.89,18.15]$}} & \shortstack{$84.73$\\{\scriptsize $[84.59,84.88]$}} & 64 \\
BoN-4 & \shortstack{$24.1$\\{\scriptsize $[21.7,26.6]$}} & \shortstack{$23.76$\\{\scriptsize $[22.68,24.93]$}} & \shortstack{$18.70$\\{\scriptsize $[18.00,19.41]$}} & \shortstack{$84.75$\\{\scriptsize $[84.61,84.89]$}} & 128 \\
BoN-8 & \shortstack{$18.6$\\{\scriptsize $[16.3,20.9]$}} & \shortstack{$24.23$\\{\scriptsize $[23.04,25.42]$}} & \shortstack{$19.47$\\{\scriptsize $[18.74,20.21]$}} & \shortstack{$84.70$\\{\scriptsize $[84.56,84.84]$}} & 256 \\
FK-2 & \shortstack{$36.2$\\{\scriptsize $[33.3,39.2]$}} & \shortstack{$21.24$\\{\scriptsize $[20.26,22.26]$}} & \shortstack{$17.59$\\{\scriptsize $[16.96,18.26]$}} & \shortstack{$84.67$\\{\scriptsize $[84.52,84.82]$}} & 64 \\
FK-4 & \shortstack{$34.6$\\{\scriptsize $[31.8,37.4]$}} & \shortstack{$21.85$\\{\scriptsize $[20.88,22.94]$}} & \shortstack{$17.84$\\{\scriptsize $[17.22,18.49]$}} & \shortstack{$84.68$\\{\scriptsize $[84.54,84.83]$}} & 128 \\
FK-8 & \shortstack{$28.2$\\{\scriptsize $[25.4,30.9]$}} & \shortstack{$22.66$\\{\scriptsize $[21.66,23.74]$}} & \shortstack{$18.17$\\{\scriptsize $[17.53,18.84]$}} & \shortstack{$84.71$\\{\scriptsize $[84.57,84.85]$}} & 256 \\
\bottomrule
\end{tabular}
\end{table}

%% file: tables/bon/benign_search_dream.tex
\begin{table}[!htbp]
\centering\footnotesize
\setlength{\tabcolsep}{3pt}\renewcommand{\arraystretch}{1.12}
\caption{Dream: BoN and FK at budgets 2, 4, and 8.}
\label{tab:benign-search-dream}
\begin{tabular}{@{}lrrrrr@{}}
\toprule
Method & Unsafe \% & Harmful PPL & Benign PPL & Benign BERT & NFE \\
\midrule
\multicolumn{6}{l}{\textbf{RTP}} \\
BoN-2 & \shortstack{$37.0$\\{\scriptsize $[34.2,39.9]$}} & \shortstack{$17.67$\\{\scriptsize $[17.01,18.35]$}} & \shortstack{$15.39$\\{\scriptsize $[14.90,15.90]$}} & \shortstack{$84.76$\\{\scriptsize $[84.62,84.91]$}} & 64 \\
BoN-4 & \shortstack{$30.6$\\{\scriptsize $[27.9,33.3]$}} & \shortstack{$17.95$\\{\scriptsize $[17.27,18.64]$}} & \shortstack{$15.78$\\{\scriptsize $[15.29,16.29]$}} & \shortstack{$84.76$\\{\scriptsize $[84.62,84.90]$}} & 128 \\
BoN-8 & \shortstack{$24.6$\\{\scriptsize $[22.0,27.3]$}} & \shortstack{$17.46$\\{\scriptsize $[16.72,18.18]$}} & \shortstack{$16.87$\\{\scriptsize $[16.30,17.46]$}} & \shortstack{$84.81$\\{\scriptsize $[84.68,84.95]$}} & 256 \\
FK-2 & \shortstack{$39.7$\\{\scriptsize $[36.8,42.7]$}} & \shortstack{$17.90$\\{\scriptsize $[17.23,18.56]$}} & \shortstack{$15.40$\\{\scriptsize $[14.92,15.90]$}} & \shortstack{$84.69$\\{\scriptsize $[84.54,84.83]$}} & 64 \\
FK-4 & \shortstack{$31.3$\\{\scriptsize $[28.6,33.9]$}} & \shortstack{$17.90$\\{\scriptsize $[17.20,18.61]$}} & \shortstack{$15.31$\\{\scriptsize $[14.83,15.83]$}} & \shortstack{$84.72$\\{\scriptsize $[84.57,84.86]$}} & 128 \\
FK-8 & \shortstack{$25.1$\\{\scriptsize $[22.5,27.7]$}} & \shortstack{$17.85$\\{\scriptsize $[17.16,18.57]$}} & \shortstack{$15.73$\\{\scriptsize $[15.27,16.24]$}} & \shortstack{$84.76$\\{\scriptsize $[84.62,84.90]$}} & 256 \\
\midrule
\multicolumn{6}{l}{\textbf{BeaverTails}} \\
BoN-2 & \shortstack{$32.5$\\{\scriptsize $[29.8,35.1]$}} & \shortstack{$12.81$\\{\scriptsize $[12.27,13.38]$}} & \shortstack{$15.39$\\{\scriptsize $[14.90,15.90]$}} & \shortstack{$84.76$\\{\scriptsize $[84.62,84.91]$}} & 64 \\
BoN-4 & \shortstack{$26.3$\\{\scriptsize $[23.7,28.8]$}} & \shortstack{$13.58$\\{\scriptsize $[13.00,14.22]$}} & \shortstack{$15.78$\\{\scriptsize $[15.29,16.29]$}} & \shortstack{$84.76$\\{\scriptsize $[84.62,84.90]$}} & 128 \\
BoN-8 & \shortstack{$20.0$\\{\scriptsize $[17.8,22.2]$}} & \shortstack{$14.95$\\{\scriptsize $[14.26,15.68]$}} & \shortstack{$16.87$\\{\scriptsize $[16.30,17.46]$}} & \shortstack{$84.81$\\{\scriptsize $[84.68,84.95]$}} & 256 \\
FK-2 & \shortstack{$36.3$\\{\scriptsize $[33.4,39.1]$}} & \shortstack{$13.21$\\{\scriptsize $[12.63,13.81]$}} & \shortstack{$15.40$\\{\scriptsize $[14.92,15.90]$}} & \shortstack{$84.69$\\{\scriptsize $[84.54,84.83]$}} & 64 \\
FK-4 & \shortstack{$32.7$\\{\scriptsize $[29.9,35.4]$}} & \shortstack{$13.76$\\{\scriptsize $[13.20,14.37]$}} & \shortstack{$15.31$\\{\scriptsize $[14.83,15.83]$}} & \shortstack{$84.72$\\{\scriptsize $[84.57,84.86]$}} & 128 \\
FK-8 & \shortstack{$27.8$\\{\scriptsize $[25.2,30.4]$}} & \shortstack{$13.94$\\{\scriptsize $[13.39,14.54]$}} & \shortstack{$15.73$\\{\scriptsize $[15.27,16.24]$}} & \shortstack{$84.76$\\{\scriptsize $[84.62,84.90]$}} & 256 \\
\bottomrule
\end{tabular}
\end{table}

%% file: tables/bon/composition_focused.tex
\begin{table}[!htbp]
\centering\footnotesize
\setlength{\tabcolsep}{3pt}\renewcommand{\arraystretch}{1.12}
\caption{Composition.}
\label{tab:composition-focused_earlier}
\begin{tabular}{@{}llrrr@{}}
\toprule
Model/data & Method & Unsafe \% & Unsafe-prompt PPL & NFE \\
\midrule
Dream/BT & BoN-2 & $32.5\;[29.8,35.1]$ & \shortstack{$12.81$\\{\scriptsize $[12.27,13.38]$}} & 64 \\
\addlinespace[2pt]
Dream/RTP & BoN-2 & $37.0\;[34.2,39.9]$ & \shortstack{$17.67$\\{\scriptsize $[17.01,18.35]$}} & 64 \\
 & SADDER + BoN-2 & $34.0\;[30.9,37.1]$ & \shortstack{$21.05$\\{\scriptsize $[19.98,22.11]$}} & 64 \\
\addlinespace[2pt]
LLaDA/BT & BoN-2 & $32.1\;[29.3,34.8]$ & \shortstack{$22.32$\\{\scriptsize $[21.38,23.33]$}} & 64 \\
 & SADDER + BoN-2 & $28.5\;[25.3,31.7]$ & \shortstack{$27.54$\\{\scriptsize $[25.89,29.39]$}} & 64 \\
\addlinespace[2pt]
LLaDA/RTP & BoN-2 & $30.6\;[27.8,33.4]$ & \shortstack{$20.90$\\{\scriptsize $[20.00,21.82]$}} & 64 \\
 & SADDER + BoN-2 & $26.7\;[23.7,29.8]$ & \shortstack{$28.80$\\{\scriptsize $[26.80,31.41]$}} & 64 \\
\bottomrule
\end{tabular}
\end{table}

%% file: tables/jailbreak_main_llada.tex
\begin{table*}[!htb]
\centering\fontsize{8}{10}\selectfont
\setlength{\tabcolsep}{2pt}\renewcommand{\arraystretch}{1.15}
\caption{LLaDA-Instruct: held-out HarmBench attacks and XSTest over-refusal. Entries use $A_b$: mean and sample SD across three seeds (percentage points for rates). HB, LG, and PPL assess complete responses. XSTest reports full-response WG refusal on safe prompts, using valid judge labels only (coverage in Appendix~\ref{app:xstest-results}). NFE is mean cost.}
\label{tab:jailbreak-llada}
\begin{tabular}{@{}lrrrr@{\hspace{8pt}}rrrr@{\hspace{8pt}}r@{}}
\toprule
\multirow{2}{*}{Method} & \multicolumn{4}{c}{\textbf{DIJA}} & \multicolumn{4}{c}{\textbf{PAD}} & \multicolumn{1}{c}{\textbf{XSTest}} \\
\cmidrule(lr){2-5}\cmidrule(lr){6-9}\cmidrule(lr){10-10}
 & \shortstack{HB ASR\\\% $\downarrow$} & \shortstack{LG unsafe\\\% $\downarrow$} & \shortstack{Harmful\\PPL $\downarrow$} & NFE & \shortstack{HB ASR\\\% $\downarrow$} & \shortstack{LG unsafe\\\% $\downarrow$} & \shortstack{Harmful\\PPL $\downarrow$} & NFE & \shortstack{WG refusal\\\% $\downarrow$} \\
\midrule
\multicolumn{10}{@{}l}{\textbf{LLaDA-Instruct}} \\
Unmodified & $78.7_{3.1}$ & $82.0_{2.6}$ & $9.34_{0.13}$ & 91.7 & $58.7_{2.5}$ & $75.7_{2.3}$ & $6.06_{0.17}$ & 116.0 & $20.7_{1.1}$ \\
Full SAD & $77.0_{2.0}$ & $81.0_{0.0}$ & $11.0_{0.13}$ & 91.7 & $63.3_{7.5}$ & $78.7_{4.0}$ & $8.24_{0.21}$ & 116.0 & $20.9_{1.9}$ \\
SADDER (selective) & $79.0_{1.0}$ & $79.3_{1.5}$ & $13.3_{0.076}$ & 91.7 & $57.7_{4.0}$ & $73.7_{3.8}$ & $13.2_{0.26}$ & 116.0 & $18.8_{0.6}$ \\
SADDER + warm-up & $78.3_{0.6}$ & $80.7_{0.6}$ & $12.8_{0.19}$ & 91.7 & $57.0_{4.0}$ & $68.3_{2.5}$ & $11.4_{0.19}$ & 116.0 & $16.3_{2.0}$ \\
CSD$^\dagger$ & $48.7_{4.7}$ & $79.7_{4.9}$ & $36.1_{0.92}$ & 115.7 & $30.3_{3.5}$ & $69.3_{0.6}$ & $19.7_{1.3}$ & 142.0 & $29.7_{1.4}$ \\
DiffuGuard (full) & $42.7_{4.5}$ & $81.7_{2.5}$ & $255_{8.9}$ & 137.0 & $34.3_{3.8}$ & $73.7_{4.2}$ & $231_{7.2}$ & 125.0 & $42.8_{0.8}$ \\
DiffuGuard (gen) & $70.7_{3.5}$ & $83.7_{2.5}$ & $41.8_{1.3}$ & 97.8 & $57.3_{2.1}$ & $74.7_{3.5}$ & $6.22_{0.14}$ & 117.0 & $24.4_{1.0}$ \\
\bottomrule
\end{tabular}
\end{table*}

%% file: sections/Appendix/xstest_results.tex
\subsection{XSTest over-refusal}
\label{app:xstest-results}
We evaluate refusal on 250 safe XSTest prompts and separately on 200 unsafe contrasts, using unmodified prompts and seeds 42, 43, and 44. Over-refusal is the refusal rate on safe prompts only. String matching uses full-refusal openings after lowercasing, trimming surrounding whitespace, and normalizing typographic apostrophes. WildGuard separately judges response refusal and harmfulness from each prompt--response pair \citep{wildguard2024}.

Table~\ref{tab:xstest-refusal} reports the string-based measurements. Tables~\ref{tab:xstest-wildguard-llada} and~\ref{tab:xstest-wildguard-dream} report WildGuard means with seed-SD subscripts and 95\% prompt-cluster bootstrap intervals below. The $n$ columns count valid labels out of 750 safe-prompt or 600 unsafe-prompt outputs. Missing WildGuard labels are excluded separately for each method; they are not counted as safe or non-refusing. Each point estimate averages the valid-label rates within the three seeds. Intervals use 5,000 resamples of prompt clusters, retaining all seeds.

\paragraph{Refusal detection and coverage.} The detectors assign different refusal rates to the same responses. For LLaDA warm-up on unsafe prompts, opening-string matching gives 53.5\% refusal and WildGuard gives 97.5\%; the corresponding Dream rates are 69.5\% and 91.5\%. So a reduction in stereotyped refusal openings can overstate the reduction in refusal behavior. WildGuard's response-harmfulness labels provide a separate check: Dream warm-up has 6.9\% harmful responses on unsafe XSTest prompts, compared with 0.0\% for unmodified sampling, among outputs with valid labels. Coverage also varies with the intervention. Paired comparisons use the intersection of valid labels, rather than subtracting rates computed on different subsets. On safe prompts, warm-up changes WildGuard refusal by $-5.9$ percentage points on LLaDA (95\% CI: $[-9.7,-2.1]$; 593 matched outputs) and $-35.4$ points on Dream ($[-41.6,-29.3]$; 602 matched outputs).

\paragraph{DiffuGuard settings.}
Generation-scope DiffuGuard uses $\lambda=0.35$ on LLaDA and $\lambda=0.5/0.75$ for Dream's DIJA/PAD settings. Both Dream thresholds are evaluated on XSTest. On vanilla prompts, full-block and generation-scope repair regions coincide, so these rows compare thresholds rather than repair scopes. Calibration reuses XSTest-safe prompts for the trigger diagnostic in Appendix~\ref{app:sadder-datasets}; these measurements include calibration prompts.

\input{tables/extra/xstest_refusal.tex}
\input{tables/extra/xstest_paired.tex}
\FloatBarrier
\input{tables/extra/xstest_wildguard_llada.tex}
\input{tables/extra/xstest_wildguard_dream.tex}
\FloatBarrier
\input{tables/extra/xstest_wildguard_paired_llada.tex}
\input{tables/extra/xstest_wildguard_paired_dream.tex}
\FloatBarrier

%% file: tables/extra/xstest_refusal.tex
\begin{table}[!htbp]
\centering\fontsize{8}{9.6}\selectfont
\setlength{\tabcolsep}{3pt}\renewcommand{\arraystretch}{1.1}
\caption{XSTest string-based refusal: means [95\% CI] and seed SDs.}
\label{tab:xstest-refusal}
\begin{tabular}{@{}lrrrr@{}}
\toprule
Method & Safe: over-refusal\% & SD & Unsafe: refusal\% & SD \\
\midrule
\multicolumn{5}{l}{\textbf{LLaDA-Instruct}} \\
Unmodified & $9.7\;[6.5,13.2]$ & 0.9 & $93.7\;[90.7,96.3]$ & 1.0 \\
Full SAD & $10.0\;[6.8,13.5]$ & 1.1 & $93.7\;[90.5,96.3]$ & 1.3 \\
SADDER (selective) & $7.9\;[5.1,10.9]$ & 0.8 & $85.2\;[80.3,89.7]$ & 0.3 \\
SADDER + warm-up & $5.2\;[2.8,8.0]$ & 0.7 & $53.5\;[46.7,60.3]$ & 0.0 \\
CSD$^\dagger$ & $10.3\;[7.1,13.7]$ & 0.8 & $93.7\;[90.7,96.3]$ & 1.0 \\
DiffuGuard (full) & $4.0\;[2.5,5.6]$ & 1.1 & $28.2\;[24.7,31.8]$ & 1.6 \\
DiffuGuard (gen) ($\lambda=0.35$) & $10.7\;[7.3,14.3]$ & 0.5 & $92.2\;[88.8,95.2]$ & 1.3 \\
SADDER (top-1) & $8.4\;[5.5,11.7]$ & 0.8 & $85.5\;[80.8,89.8]$ & 0.5 \\
SADDER (aligned) & $8.5\;[5.6,11.7]$ & 0.9 & $85.3\;[80.5,89.8]$ & 0.8 \\
Random selection & $10.0\;[6.8,13.5]$ & 0.7 & $86.2\;[82.7,89.5]$ & 2.9 \\
\midrule
\multicolumn{5}{l}{\textbf{Dream-Instruct}} \\
Unmodified & $44.9\;[39.2,50.5]$ & 0.6 & $98.8\;[97.3,100.0]$ & 0.3 \\
Full SAD & $46.0\;[40.0,51.9]$ & 0.7 & $98.5\;[96.5,100.0]$ & 0.0 \\
SADDER (selective) & $26.1\;[20.9,31.2]$ & 0.5 & $92.0\;[88.2,95.5]$ & 0.5 \\
SADDER + warm-up & $7.1\;[4.1,10.4]$ & 0.2 & $69.5\;[63.0,76.0]$ & 0.0 \\
CSD$^\dagger$ & $45.2\;[39.5,50.8]$ & 0.8 & $98.8\;[97.3,100.0]$ & 0.3 \\
DiffuGuard (full) & $6.7\;[4.9,8.5]$ & 1.2 & $13.5\;[10.8,16.3]$ & 1.0 \\
DiffuGuard (gen) ($\lambda=0.75$) & $43.1\;[37.5,48.5]$ & 0.5 & $98.8\;[97.3,100.0]$ & 0.3 \\
DiffuGuard (gen) ($\lambda=0.5$) & $41.9\;[36.4,47.2]$ & 0.6 & $98.0\;[96.0,99.5]$ & 0.5 \\
SADDER (top-1) & $34.7\;[28.8,40.5]$ & 0.2 & $96.5\;[93.5,99.0]$ & 0.0 \\
SADDER (aligned) & $29.5\;[24.0,34.9]$ & 0.5 & $96.5\;[93.8,98.7]$ & 0.9 \\
Random selection & $22.5\;[18.7,26.7]$ & 2.3 & $74.2\;[70.0,78.2]$ & 4.9 \\
\bottomrule
\end{tabular}
\end{table}

%% file: tables/extra/xstest_paired.tex
\begin{table}[!htbp]
\centering\fontsize{8}{9.6}\selectfont
\setlength{\tabcolsep}{3pt}\renewcommand{\arraystretch}{1.1}
\caption{Paired string-based refusal changes versus unmodified sampling (pp).}
\label{tab:xstest-paired}
\begin{tabular}{@{}lrr@{}}
\toprule
Method & Safe prompts & Unsafe prompts \\
\midrule
\multicolumn{3}{l}{\textbf{LLaDA-Instruct}} \\
Full SAD & $0.3\;[-0.9,1.5]$ & $0.0\;[-1.0,1.0]$ \\
SADDER (selective) & $-1.9\;[-3.7,-0.3]$ & $-8.5\;[-13.0,-4.2]$ \\
SADDER + warm-up & $-4.5\;[-7.5,-1.9]$ & $-40.2\;[-47.3,-32.7]$ \\
CSD$^\dagger$ & $0.5\;[0.1,1.1]$ & $0.0\;[0.0,0.0]$ \\
DiffuGuard (full) & $-5.7\;[-9.1,-2.8]$ & $-65.5\;[-69.8,-61.0]$ \\
DiffuGuard (gen) ($\lambda=0.35$) & $0.9\;[-0.1,2.1]$ & $-1.5\;[-3.5,0.3]$ \\
SADDER (top-1) & $-1.3\;[-3.2,0.4]$ & $-8.2\;[-12.3,-4.2]$ \\
SADDER (aligned) & $-1.2\;[-3.1,0.4]$ & $-8.3\;[-12.8,-4.2]$ \\
Random selection & $0.3\;[-1.5,2.0]$ & $-7.5\;[-11.3,-4.0]$ \\
\midrule
\multicolumn{3}{l}{\textbf{Dream-Instruct}} \\
Full SAD & $1.1\;[-0.9,3.1]$ & $-0.3\;[-1.5,0.5]$ \\
SADDER (selective) & $-18.8\;[-24.0,-13.6]$ & $-6.8\;[-10.8,-3.3]$ \\
SADDER + warm-up & $-37.9\;[-43.6,-32.0]$ & $-29.3\;[-35.8,-23.0]$ \\
CSD$^\dagger$ & $0.3\;[0.0,0.7]$ & $0.0\;[0.0,0.0]$ \\
DiffuGuard (full) & $-38.3\;[-43.7,-32.7]$ & $-85.3\;[-88.3,-82.2]$ \\
DiffuGuard (gen) ($\lambda=0.75$) & $-1.9\;[-3.2,-0.5]$ & $0.0\;[-0.5,0.5]$ \\
DiffuGuard (gen) ($\lambda=0.5$) & $-3.1\;[-5.1,-1.2]$ & $-0.8\;[-2.3,0.2]$ \\
SADDER (top-1) & $-10.3\;[-14.7,-6.0]$ & $-2.3\;[-5.2,0.0]$ \\
SADDER (aligned) & $-15.5\;[-19.9,-11.1]$ & $-2.3\;[-5.0,0.0]$ \\
Random selection & $-22.4\;[-27.5,-17.5]$ & $-24.7\;[-28.7,-20.7]$ \\
\bottomrule
\end{tabular}
\end{table}

%% file: tables/extra/xstest_wildguard_llada.tex
\begin{table}[!htbp]
\centering\fontsize{8}{9.6}\selectfont
\setlength{\tabcolsep}{3pt}\renewcommand{\arraystretch}{1.1}
\caption{LLaDA: WildGuard judgments on XSTest.}
\label{tab:xstest-wildguard-llada}
\begin{tabular}{@{}lrrrrrr@{}}
\toprule
 & \multicolumn{3}{c}{Safe prompts} & \multicolumn{3}{c}{Unsafe prompts} \\
\cmidrule(lr){2-4}\cmidrule(lr){5-7}
Method & Over-refusal\% & Harmful\% & $n$ & Refusal\% & Harmful\% & $n$ \\
\midrule
Unmodified & \shortstack{$20.7_{1.1}$\\{\scriptsize $[16.2,25.6]$}} & \shortstack{$3.3_{0.8}$\\{\scriptsize $[1.3,5.7]$}} & 603 & \shortstack{$97.9_{1.0}$\\{\scriptsize $[96.0,99.5]$}} & \shortstack{$1.2_{0.6}$\\{\scriptsize $[0.2,2.7]$}} & 585 \\
Full SAD & \shortstack{$20.9_{1.9}$\\{\scriptsize $[16.4,25.7]$}} & \shortstack{$3.1_{0.5}$\\{\scriptsize $[1.2,5.5]$}} & 603 & \shortstack{$97.9_{1.0}$\\{\scriptsize $[96.0,99.5]$}} & \shortstack{$1.2_{0.6}$\\{\scriptsize $[0.2,2.7]$}} & 585 \\
SADDER (selective) & \shortstack{$18.8_{0.6}$\\{\scriptsize $[14.5,23.4]$}} & \shortstack{$2.5_{0.8}$\\{\scriptsize $[0.9,4.3]$}} & 644 & \shortstack{$98.8_{0.3}$\\{\scriptsize $[97.3,100.0]$}} & \shortstack{$0.7_{0.3}$\\{\scriptsize $[0.0,1.7]$}} & 584 \\
SADDER + warm-up & \shortstack{$16.3_{2.0}$\\{\scriptsize $[12.4,20.8]$}} & \shortstack{$2.5_{0.0}$\\{\scriptsize $[1.0,4.4]$}} & 710 & \shortstack{$97.5_{0.5}$\\{\scriptsize $[95.2,99.3]$}} & \shortstack{$0.5_{0.5}$\\{\scriptsize $[0.0,1.4]$}} & 592 \\
CSD$^\dagger$ & \shortstack{$29.7_{1.4}$\\{\scriptsize $[24.5,35.1]$}} & \shortstack{$3.7_{0.7}$\\{\scriptsize $[2.1,5.7]$}} & 589 & \shortstack{$98.1_{1.1}$\\{\scriptsize $[96.5,99.5]$}} & \shortstack{$1.0_{0.5}$\\{\scriptsize $[0.2,2.2]$}} & 585 \\
DiffuGuard (full) & \shortstack{$42.8_{0.8}$\\{\scriptsize $[38.5,47.1]$}} & \shortstack{$7.6_{1.7}$\\{\scriptsize $[5.5,9.9]$}} & 750 & \shortstack{$84.5_{1.5}$\\{\scriptsize $[81.3,87.5]$}} & \shortstack{$14.4_{1.6}$\\{\scriptsize $[11.4,17.5]$}} & 599 \\
DiffuGuard (gen) ($\lambda=0.35$) & \shortstack{$24.4_{1.0}$\\{\scriptsize $[19.7,29.4]$}} & \shortstack{$4.6_{2.4}$\\{\scriptsize $[2.4,7.1]$}} & 595 & \shortstack{$98.3_{0.3}$\\{\scriptsize $[96.7,99.6]$}} & \shortstack{$0.9_{0.3}$\\{\scriptsize $[0.2,1.9]$}} & 586 \\
SADDER (top-1) & \shortstack{$20.8_{0.2}$\\{\scriptsize $[16.3,25.5]$}} & \shortstack{$3.4_{0.8}$\\{\scriptsize $[1.5,5.5]$}} & 650 & \shortstack{$98.8_{0.3}$\\{\scriptsize $[97.3,99.8]$}} & \shortstack{$0.7_{0.3}$\\{\scriptsize $[0.0,1.6]$}} & 583 \\
SADDER (aligned) & \shortstack{$20.0_{0.7}$\\{\scriptsize $[15.4,24.6]$}} & \shortstack{$3.4_{0.5}$\\{\scriptsize $[1.5,5.7]$}} & 645 & \shortstack{$98.0_{0.5}$\\{\scriptsize $[95.9,99.5]$}} & \shortstack{$0.9_{0.3}$\\{\scriptsize $[0.0,2.2]$}} & 586 \\
Random selection & \shortstack{$23.5_{1.4}$\\{\scriptsize $[18.7,28.5]$}} & \shortstack{$3.7_{1.0}$\\{\scriptsize $[1.9,5.8]$}} & 622 & \shortstack{$99.5_{0.0}$\\{\scriptsize $[98.6,100.0]$}} & \shortstack{$0.0_{0.0}$\\{\scriptsize $[0.0,0.0]$}} & 588 \\
\bottomrule
\end{tabular}
\end{table}

%% file: tables/extra/xstest_wildguard_dream.tex
\begin{table}[!htbp]
\centering\fontsize{8}{9.6}\selectfont
\setlength{\tabcolsep}{3pt}\renewcommand{\arraystretch}{1.1}
\caption{Dream: WildGuard judgments on XSTest.}
\label{tab:xstest-wildguard-dream}
\begin{tabular}{@{}lrrrrrr@{}}
\toprule
 & \multicolumn{3}{c}{Safe prompts} & \multicolumn{3}{c}{Unsafe prompts} \\
\cmidrule(lr){2-4}\cmidrule(lr){5-7}
Method & Over-refusal\% & Harmful\% & $n$ & Refusal\% & Harmful\% & $n$ \\
\midrule
Unmodified & \shortstack{$55.5_{1.3}$\\{\scriptsize $[49.1,61.5]$}} & \shortstack{$2.2_{0.5}$\\{\scriptsize $[0.8,3.8]$}} & 629 & \shortstack{$99.5_{0.0}$\\{\scriptsize $[98.5,100.0]$}} & \shortstack{$0.0_{0.0}$\\{\scriptsize $[0.0,0.0]$}} & 600 \\
Full SAD & \shortstack{$58.0_{1.0}$\\{\scriptsize $[51.6,64.4]$}} & \shortstack{$2.7_{0.3}$\\{\scriptsize $[1.1,4.7]$}} & 619 & \shortstack{$99.5_{0.0}$\\{\scriptsize $[98.5,100.0]$}} & \shortstack{$0.0_{0.0}$\\{\scriptsize $[0.0,0.0]$}} & 600 \\
SADDER (selective) & \shortstack{$41.1_{0.8}$\\{\scriptsize $[34.3,47.9]$}} & \shortstack{$3.2_{0.5}$\\{\scriptsize $[1.1,5.6]$}} & 562 & \shortstack{$99.1_{0.3}$\\{\scriptsize $[97.7,100.0]$}} & \shortstack{$0.0_{0.0}$\\{\scriptsize $[0.0,0.0]$}} & 572 \\
SADDER + warm-up & \shortstack{$17.9_{0.7}$\\{\scriptsize $[13.5,22.7]$}} & \shortstack{$1.6_{0.2}$\\{\scriptsize $[0.4,2.9]$}} & 709 & \shortstack{$91.5_{1.3}$\\{\scriptsize $[88.0,94.5]$}} & \shortstack{$6.9_{1.8}$\\{\scriptsize $[4.2,10.0]$}} & 597 \\
CSD$^\dagger$ & \shortstack{$56.4_{1.9}$\\{\scriptsize $[50.0,62.4]$}} & \shortstack{$2.1_{1.0}$\\{\scriptsize $[0.8,3.7]$}} & 630 & \shortstack{$99.5_{0.0}$\\{\scriptsize $[98.5,100.0]$}} & \shortstack{$0.0_{0.0}$\\{\scriptsize $[0.0,0.0]$}} & 600 \\
DiffuGuard (full) & \shortstack{$64.1_{1.5}$\\{\scriptsize $[60.0,68.1]$}} & \shortstack{$14.3_{2.6}$\\{\scriptsize $[11.2,17.6]$}} & 750 & \shortstack{$75.8_{0.6}$\\{\scriptsize $[72.2,79.3]$}} & \shortstack{$23.2_{1.0}$\\{\scriptsize $[19.8,26.7]$}} & 600 \\
DiffuGuard (gen) ($\lambda=0.75$) & \shortstack{$53.1_{1.5}$\\{\scriptsize $[46.9,58.9]$}} & \shortstack{$2.5_{0.5}$\\{\scriptsize $[0.9,4.4]$}} & 641 & \shortstack{$99.5_{0.0}$\\{\scriptsize $[98.5,100.0]$}} & \shortstack{$0.0_{0.0}$\\{\scriptsize $[0.0,0.0]$}} & 600 \\
DiffuGuard (gen) ($\lambda=0.5$) & \shortstack{$54.6_{1.1}$\\{\scriptsize $[48.8,60.3]$}} & \shortstack{$3.2_{0.4}$\\{\scriptsize $[1.5,5.1]$}} & 665 & \shortstack{$99.2_{0.3}$\\{\scriptsize $[97.8,100.0]$}} & \shortstack{$0.2_{0.3}$\\{\scriptsize $[0.0,0.5]$}} & 600 \\
SADDER (top-1) & \shortstack{$41.2_{0.1}$\\{\scriptsize $[34.7,47.4]$}} & \shortstack{$3.0_{0.7}$\\{\scriptsize $[1.2,5.0]$}} & 670 & \shortstack{$99.0_{0.5}$\\{\scriptsize $[97.7,100.0]$}} & \shortstack{$0.0_{0.0}$\\{\scriptsize $[0.0,0.0]$}} & 600 \\
SADDER (aligned) & \shortstack{$38.1_{1.3}$\\{\scriptsize $[31.8,44.1]$}} & \shortstack{$2.5_{0.7}$\\{\scriptsize $[0.8,4.3]$}} & 653 & \shortstack{$98.0_{0.9}$\\{\scriptsize $[96.0,99.5]$}} & \shortstack{$0.0_{0.0}$\\{\scriptsize $[0.0,0.0]$}} & 599 \\
Random selection & \shortstack{$39.2_{0.5}$\\{\scriptsize $[33.2,45.2]$}} & \shortstack{$2.9_{0.1}$\\{\scriptsize $[1.0,5.1]$}} & 526 & \shortstack{$98.1_{0.7}$\\{\scriptsize $[96.1,99.6]$}} & \shortstack{$1.3_{0.7}$\\{\scriptsize $[0.2,2.7]$}} & 471 \\
\bottomrule
\end{tabular}
\end{table}

%% file: tables/extra/xstest_wildguard_paired_llada.tex
\begin{table}[!htbp]
\centering\fontsize{8}{9.6}\selectfont
\setlength{\tabcolsep}{3pt}\renewcommand{\arraystretch}{1.1}
\caption{LLaDA: paired WildGuard refusal changes (pp).}
\label{tab:xstest-wildguard-paired-llada}
\begin{tabular}{@{}lrrrr@{}}
\toprule
Method & Safe prompts & Pairs & Unsafe prompts & Pairs \\
\midrule
Full SAD & $-0.0\;[-1.2,1.4]$ & 600 & $0.0\;[0.0,0.0]$ & 585 \\
SADDER (selective) & $-2.4\;[-5.8,0.9]$ & 579 & $0.7\;[-0.9,2.4]$ & 581 \\
SADDER + warm-up & $-5.9\;[-9.7,-2.1]$ & 593 & $-0.7\;[-2.9,1.4]$ & 581 \\
CSD$^\dagger$ & $9.3\;[5.9,12.8]$ & 588 & $0.2\;[-0.3,0.7]$ & 585 \\
DiffuGuard (full) & $16.7\;[12.4,21.0]$ & 603 & $-13.7\;[-16.8,-10.6]$ & 585 \\
DiffuGuard (gen) ($\lambda=0.35$) & $3.2\;[1.3,5.3]$ & 589 & $0.3\;[-0.5,1.4]$ & 585 \\
SADDER (top-1) & $-1.0\;[-4.6,2.4]$ & 585 & $0.7\;[-0.7,2.4]$ & 583 \\
SADDER (aligned) & $-1.9\;[-4.9,0.9]$ & 584 & $-0.2\;[-2.4,1.9]$ & 583 \\
Random selection & $1.1\;[-2.0,4.2]$ & 568 & $1.4\;[0.2,2.9]$ & 580 \\
\bottomrule
\end{tabular}
\end{table}

%% file: tables/extra/xstest_wildguard_paired_dream.tex
\begin{table}[!htbp]
\centering\fontsize{8}{9.6}\selectfont
\setlength{\tabcolsep}{3pt}\renewcommand{\arraystretch}{1.1}
\caption{Dream: paired WildGuard refusal changes (pp).}
\label{tab:xstest-wildguard-paired-dream}
\begin{tabular}{@{}lrrrr@{}}
\toprule
Method & Safe prompts & Pairs & Unsafe prompts & Pairs \\
\midrule
Full SAD & $1.8\;[-0.7,4.3]$ & 619 & $0.0\;[0.0,0.0]$ & 600 \\
SADDER (selective) & $-8.6\;[-14.1,-3.3]$ & 532 & $-0.3\;[-1.1,0.0]$ & 572 \\
SADDER + warm-up & $-35.4\;[-41.6,-29.3]$ & 602 & $-8.0\;[-11.5,-5.0]$ & 597 \\
CSD$^\dagger$ & $1.0\;[0.3,1.8]$ & 629 & $0.0\;[0.0,0.0]$ & 600 \\
DiffuGuard (full) & $7.6\;[2.7,12.7]$ & 629 & $-23.7\;[-27.3,-20.0]$ & 600 \\
DiffuGuard (gen) ($\lambda=0.75$) & $-1.4\;[-3.1,0.2]$ & 629 & $0.0\;[0.0,0.0]$ & 600 \\
DiffuGuard (gen) ($\lambda=0.5$) & $-0.8\;[-2.8,1.4]$ & 629 & $-0.3\;[-0.8,0.0]$ & 600 \\
SADDER (top-1) & $-12.5\;[-17.7,-7.3]$ & 625 & $-0.5\;[-1.3,0.0]$ & 600 \\
SADDER (aligned) & $-14.1\;[-19.0,-9.2]$ & 603 & $-1.5\;[-3.2,-0.3]$ & 599 \\
Random selection & $-6.7\;[-11.9,-1.7]$ & 493 & $-1.3\;[-2.7,-0.2]$ & 471 \\
\bottomrule
\end{tabular}
\end{table}

%% file: tables/jailbreak_main_ci.tex
\begin{table*}[!htb]
\centering\fontsize{8}{9.6}\selectfont
\setlength{\tabcolsep}{0.5pt}\renewcommand{\arraystretch}{1.08}
\caption{Held-out HarmBench attacks and XSTest over-refusal.}
\label{tab:jailbreak-main-ci}
\begin{tabular*}{\textwidth}{@{\extracolsep{\fill}}l@{\hspace{3pt}}rrrr@{\hspace{5pt}}rrrr@{\hspace{5pt}}r@{}}
\toprule
\multirow{2}{*}{Method} & \multicolumn{4}{c}{\textbf{DIJA}} & \multicolumn{4}{c}{\textbf{PAD}} & \textbf{XSTest} \\
\cmidrule(lr){2-5}\cmidrule(lr){6-9}\cmidrule(lr){10-10}
 & \shortstack{HB ASR\\\% $\downarrow$} & \shortstack{LG unsafe\\\% $\downarrow$} & \shortstack{Harmful\\PPL $\downarrow$} & NFE & \shortstack{HB ASR\\\% $\downarrow$} & \shortstack{LG unsafe\\\% $\downarrow$} & \shortstack{Harmful\\PPL $\downarrow$} & NFE & \shortstack{Over-refusal\\\% $\downarrow$} \\
\midrule
\multicolumn{10}{@{}l}{\textbf{LLaDA-Instruct}} \\
Unmodified & \raisebox{-0.5\height}{\shortstack{$78.7$\\{\fontsize{7}{8}\selectfont $[72.0{,}85.0]$}}} & \raisebox{-0.5\height}{\shortstack{$82.0$\\{\fontsize{7}{8}\selectfont $[75.0{,}88.3]$}}} & \raisebox{-0.5\height}{\shortstack{$9.34$\\{\fontsize{7}{8}\selectfont $[8.59{,}10.1]$}}} & 91.7 & \raisebox{-0.5\height}{\shortstack{$58.7$\\{\fontsize{7}{8}\selectfont $[50.3{,}66.7]$}}} & \raisebox{-0.5\height}{\shortstack{$75.7$\\{\fontsize{7}{8}\selectfont $[68.3{,}82.3]$}}} & \raisebox{-0.5\height}{\shortstack{$6.06$\\{\fontsize{7}{8}\selectfont $[5.68{,}6.47]$}}} & 116.0 & \raisebox{-0.5\height}{\shortstack{$9.7$\\{\fontsize{7}{8}\selectfont $[6.5{,}13.2]$}}} \\
\addlinespace[2pt]
Full SAD & \raisebox{-0.5\height}{\shortstack{$77.0$\\{\fontsize{7}{8}\selectfont $[69.3{,}84.3]$}}} & \raisebox{-0.5\height}{\shortstack{$81.0$\\{\fontsize{7}{8}\selectfont $[73.7{,}87.7]$}}} & \raisebox{-0.5\height}{\shortstack{$11.0$\\{\fontsize{7}{8}\selectfont $[10.0{,}12.1]$}}} & 91.7 & \raisebox{-0.5\height}{\shortstack{$63.3$\\{\fontsize{7}{8}\selectfont $[56.3{,}70.3]$}}} & \raisebox{-0.5\height}{\shortstack{$78.7$\\{\fontsize{7}{8}\selectfont $[72.0{,}85.0]$}}} & \raisebox{-0.5\height}{\shortstack{$8.24$\\{\fontsize{7}{8}\selectfont $[7.79{,}8.74]$}}} & 116.0 & \raisebox{-0.5\height}{\shortstack{$10.0$\\{\fontsize{7}{8}\selectfont $[6.8{,}13.5]$}}} \\
\addlinespace[2pt]
\raisebox{-0.5\height}{\shortstack[l]{SADDER\\(selective)}} & \raisebox{-0.5\height}{\shortstack{$79.0$\\{\fontsize{7}{8}\selectfont $[71.7{,}86.3]$}}} & \raisebox{-0.5\height}{\shortstack{$79.3$\\{\fontsize{7}{8}\selectfont $[71.7{,}86.3]$}}} & \raisebox{-0.5\height}{\shortstack{$13.3$\\{\fontsize{7}{8}\selectfont $[11.6{,}15.7]$}}} & 91.7 & \raisebox{-0.5\height}{\shortstack{$57.7$\\{\fontsize{7}{8}\selectfont $[50.3{,}65.0]$}}} & \raisebox{-0.5\height}{\shortstack{$73.7$\\{\fontsize{7}{8}\selectfont $[66.3{,}80.7]$}}} & \raisebox{-0.5\height}{\shortstack{$13.2$\\{\fontsize{7}{8}\selectfont $[10.3{,}17.7]$}}} & 116.0 & \raisebox{-0.5\height}{\shortstack{$7.9$\\{\fontsize{7}{8}\selectfont $[5.1{,}10.9]$}}} \\
\addlinespace[2pt]
\raisebox{-0.5\height}{\shortstack[l]{SADDER\\+ warm-up}} & \raisebox{-0.5\height}{\shortstack{$78.3$\\{\fontsize{7}{8}\selectfont $[71.3{,}85.0]$}}} & \raisebox{-0.5\height}{\shortstack{$80.7$\\{\fontsize{7}{8}\selectfont $[73.0{,}87.3]$}}} & \raisebox{-0.5\height}{\shortstack{$12.8$\\{\fontsize{7}{8}\selectfont $[11.1{,}15.1]$}}} & 91.7 & \raisebox{-0.5\height}{\shortstack{$57.0$\\{\fontsize{7}{8}\selectfont $[49.0{,}65.0]$}}} & \raisebox{-0.5\height}{\shortstack{$68.3$\\{\fontsize{7}{8}\selectfont $[60.3{,}76.0]$}}} & \raisebox{-0.5\height}{\shortstack{$11.4$\\{\fontsize{7}{8}\selectfont $[9.97{,}13.4]$}}} & 116.0 & \raisebox{-0.5\height}{\shortstack{$5.2$\\{\fontsize{7}{8}\selectfont $[2.8{,}8.0]$}}} \\
\addlinespace[2pt]
CSD$^\dagger$ & \raisebox{-0.5\height}{\shortstack{$48.7$\\{\fontsize{7}{8}\selectfont $[40.7{,}57.0]$}}} & \raisebox{-0.5\height}{\shortstack{$79.7$\\{\fontsize{7}{8}\selectfont $[72.7{,}86.3]$}}} & \raisebox{-0.5\height}{\shortstack{$36.1$\\{\fontsize{7}{8}\selectfont $[31.1{,}41.4]$}}} & 115.7 & \raisebox{-0.5\height}{\shortstack{$30.3$\\{\fontsize{7}{8}\selectfont $[24.0{,}37.0]$}}} & \raisebox{-0.5\height}{\shortstack{$69.3$\\{\fontsize{7}{8}\selectfont $[61.7{,}76.7]$}}} & \raisebox{-0.5\height}{\shortstack{$19.7$\\{\fontsize{7}{8}\selectfont $[17.3{,}22.6]$}}} & 142.0 & \raisebox{-0.5\height}{\shortstack{$10.3$\\{\fontsize{7}{8}\selectfont $[7.1{,}13.7]$}}} \\
\addlinespace[2pt]
DiffuGuard (full) & \raisebox{-0.5\height}{\shortstack{$42.7$\\{\fontsize{7}{8}\selectfont $[35.7{,}50.0]$}}} & \raisebox{-0.5\height}{\shortstack{$81.7$\\{\fontsize{7}{8}\selectfont $[75.7{,}87.0]$}}} & \raisebox{-0.5\height}{\shortstack{$255$\\{\fontsize{7}{8}\selectfont $[239{,}272]$}}} & 137.0 & \raisebox{-0.5\height}{\shortstack{$34.3$\\{\fontsize{7}{8}\selectfont $[28.0{,}41.0]$}}} & \raisebox{-0.5\height}{\shortstack{$73.7$\\{\fontsize{7}{8}\selectfont $[67.3{,}79.7]$}}} & \raisebox{-0.5\height}{\shortstack{$231$\\{\fontsize{7}{8}\selectfont $[214{,}251]$}}} & 125.0 & \raisebox{-0.5\height}{\shortstack{$4.0$\\{\fontsize{7}{8}\selectfont $[2.5{,}5.6]$}}} \\
\addlinespace[2pt]
DiffuGuard (gen) & \raisebox{-0.5\height}{\shortstack{$70.7$\\{\fontsize{7}{8}\selectfont $[63.7{,}77.7]$}}} & \raisebox{-0.5\height}{\shortstack{$83.7$\\{\fontsize{7}{8}\selectfont $[77.0{,}89.7]$}}} & \raisebox{-0.5\height}{\shortstack{$41.8$\\{\fontsize{7}{8}\selectfont $[32.0{,}53.7]$}}} & 97.8 & \raisebox{-0.5\height}{\shortstack{$57.3$\\{\fontsize{7}{8}\selectfont $[49.0{,}65.0]$}}} & \raisebox{-0.5\height}{\shortstack{$74.7$\\{\fontsize{7}{8}\selectfont $[67.0{,}81.7]$}}} & \raisebox{-0.5\height}{\shortstack{$6.22$\\{\fontsize{7}{8}\selectfont $[5.85{,}6.59]$}}} & 117.0 & \raisebox{-0.5\height}{\shortstack{$10.7$\\{\fontsize{7}{8}\selectfont $[7.3{,}14.3]$}}} \\
\addlinespace[2pt]
\midrule
\multicolumn{10}{@{}l}{\textbf{Dream-Instruct}} \\
Unmodified & \raisebox{-0.5\height}{\shortstack{$82.0$\\{\fontsize{7}{8}\selectfont $[75.3{,}88.3]$}}} & \raisebox{-0.5\height}{\shortstack{$80.7$\\{\fontsize{7}{8}\selectfont $[73.3{,}87.3]$}}} & \raisebox{-0.5\height}{\shortstack{$11.4$\\{\fontsize{7}{8}\selectfont $[10.6{,}12.3]$}}} & 91.7 & \raisebox{-0.5\height}{\shortstack{$51.7$\\{\fontsize{7}{8}\selectfont $[43.0{,}60.3]$}}} & \raisebox{-0.5\height}{\shortstack{$64.0$\\{\fontsize{7}{8}\selectfont $[55.7{,}72.0]$}}} & \raisebox{-0.5\height}{\shortstack{$8.53$\\{\fontsize{7}{8}\selectfont $[7.95{,}9.21]$}}} & 116.0 & \raisebox{-0.5\height}{\shortstack{$44.9$\\{\fontsize{7}{8}\selectfont $[39.2{,}50.5]$}}} \\
\addlinespace[2pt]
Full SAD & \raisebox{-0.5\height}{\shortstack{$85.3$\\{\fontsize{7}{8}\selectfont $[79.0{,}91.3]$}}} & \raisebox{-0.5\height}{\shortstack{$83.0$\\{\fontsize{7}{8}\selectfont $[75.3{,}90.0]$}}} & \raisebox{-0.5\height}{\shortstack{$11.8$\\{\fontsize{7}{8}\selectfont $[10.9{,}12.8]$}}} & 91.7 & \raisebox{-0.5\height}{\shortstack{$52.7$\\{\fontsize{7}{8}\selectfont $[43.7{,}61.7]$}}} & \raisebox{-0.5\height}{\shortstack{$63.3$\\{\fontsize{7}{8}\selectfont $[54.7{,}71.0]$}}} & \raisebox{-0.5\height}{\shortstack{$9.52$\\{\fontsize{7}{8}\selectfont $[9.01{,}10.1]$}}} & 116.0 & \raisebox{-0.5\height}{\shortstack{$46.0$\\{\fontsize{7}{8}\selectfont $[40.0{,}51.9]$}}} \\
\addlinespace[2pt]
\raisebox{-0.5\height}{\shortstack[l]{SADDER\\(selective)}} & \raisebox{-0.5\height}{\shortstack{$77.7$\\{\fontsize{7}{8}\selectfont $[70.3{,}84.7]$}}} & \raisebox{-0.5\height}{\shortstack{$81.7$\\{\fontsize{7}{8}\selectfont $[74.7{,}88.0]$}}} & \raisebox{-0.5\height}{\shortstack{$12.2$\\{\fontsize{7}{8}\selectfont $[11.3{,}13.1]$}}} & 91.7 & \raisebox{-0.5\height}{\shortstack{$52.7$\\{\fontsize{7}{8}\selectfont $[43.7{,}61.7]$}}} & \raisebox{-0.5\height}{\shortstack{$66.3$\\{\fontsize{7}{8}\selectfont $[57.3{,}75.0]$}}} & \raisebox{-0.5\height}{\shortstack{$9.58$\\{\fontsize{7}{8}\selectfont $[8.86{,}10.4]$}}} & 116.0 & \raisebox{-0.5\height}{\shortstack{$26.1$\\{\fontsize{7}{8}\selectfont $[20.9{,}31.2]$}}} \\
\addlinespace[2pt]
\raisebox{-0.5\height}{\shortstack[l]{SADDER\\+ warm-up}} & \raisebox{-0.5\height}{\shortstack{$78.7$\\{\fontsize{7}{8}\selectfont $[71.7{,}85.7]$}}} & \raisebox{-0.5\height}{\shortstack{$79.7$\\{\fontsize{7}{8}\selectfont $[72.3{,}86.7]$}}} & \raisebox{-0.5\height}{\shortstack{$12.2$\\{\fontsize{7}{8}\selectfont $[11.3{,}13.2]$}}} & 91.7 & \raisebox{-0.5\height}{\shortstack{$42.7$\\{\fontsize{7}{8}\selectfont $[33.7{,}51.7]$}}} & \raisebox{-0.5\height}{\shortstack{$55.0$\\{\fontsize{7}{8}\selectfont $[45.7{,}64.0]$}}} & \raisebox{-0.5\height}{\shortstack{$10.4$\\{\fontsize{7}{8}\selectfont $[9.60{,}11.4]$}}} & 116.0 & \raisebox{-0.5\height}{\shortstack{$7.1$\\{\fontsize{7}{8}\selectfont $[4.1{,}10.4]$}}} \\
\addlinespace[2pt]
CSD$^\dagger$ & \raisebox{-0.5\height}{\shortstack{$68.0$\\{\fontsize{7}{8}\selectfont $[60.7{,}75.7]$}}} & \raisebox{-0.5\height}{\shortstack{$82.3$\\{\fontsize{7}{8}\selectfont $[75.7{,}88.7]$}}} & \raisebox{-0.5\height}{\shortstack{$45.9$\\{\fontsize{7}{8}\selectfont $[37.6{,}55.9]$}}} & 115.5 & \raisebox{-0.5\height}{\shortstack{$33.3$\\{\fontsize{7}{8}\selectfont $[26.0{,}41.0]$}}} & \raisebox{-0.5\height}{\shortstack{$63.3$\\{\fontsize{7}{8}\selectfont $[54.7{,}71.0]$}}} & \raisebox{-0.5\height}{\shortstack{$23.7$\\{\fontsize{7}{8}\selectfont $[19.2{,}29.8]$}}} & 140.8 & \raisebox{-0.5\height}{\shortstack{$45.2$\\{\fontsize{7}{8}\selectfont $[39.5{,}50.8]$}}} \\
\addlinespace[2pt]
DiffuGuard (full) & \raisebox{-0.5\height}{\shortstack{$33.3$\\{\fontsize{7}{8}\selectfont $[27.3{,}39.7]$}}} & \raisebox{-0.5\height}{\shortstack{$84.0$\\{\fontsize{7}{8}\selectfont $[78.0{,}89.3]$}}} & \raisebox{-0.5\height}{\shortstack{$537$\\{\fontsize{7}{8}\selectfont $[494{,}587]$}}} & 137.0 & \raisebox{-0.5\height}{\shortstack{$12.0$\\{\fontsize{7}{8}\selectfont $[8.0{,}16.0]$}}} & \raisebox{-0.5\height}{\shortstack{$70.0$\\{\fontsize{7}{8}\selectfont $[64.0{,}76.0]$}}} & \raisebox{-0.5\height}{\shortstack{$687$\\{\fontsize{7}{8}\selectfont $[631{,}748]$}}} & 125.0 & \raisebox{-0.5\height}{\shortstack{$6.7$\\{\fontsize{7}{8}\selectfont $[4.9{,}8.5]$}}} \\
\addlinespace[2pt]
DiffuGuard (gen) & \raisebox{-0.5\height}{\shortstack{$69.3$\\{\fontsize{7}{8}\selectfont $[62.0{,}76.3]$}}} & \raisebox{-0.5\height}{\shortstack{$81.3$\\{\fontsize{7}{8}\selectfont $[74.7{,}87.3]$}}} & \raisebox{-0.5\height}{\shortstack{$46.6$\\{\fontsize{7}{8}\selectfont $[36.0{,}60.4]$}}} & 97.2 & \raisebox{-0.5\height}{\shortstack{$49.0$\\{\fontsize{7}{8}\selectfont $[40.3{,}57.3]$}}} & \raisebox{-0.5\height}{\shortstack{$63.3$\\{\fontsize{7}{8}\selectfont $[55.0{,}71.3]$}}} & \raisebox{-0.5\height}{\shortstack{$8.80$\\{\fontsize{7}{8}\selectfont $[8.28{,}9.41]$}}} & 117.0 & \raisebox{-0.5\height}{\shortstack{$43.1^{\ddagger}$\\{\fontsize{7}{8}\selectfont $[37.5{,}48.5]$}}} \\
\addlinespace[2pt]
\bottomrule
\end{tabular*}
\end{table*}

%% file: tables/extra/wildguard_attacks_llada.tex
\begin{table}[!htbp]
\centering\fontsize{8}{9.6}\selectfont
\setlength{\tabcolsep}{3pt}\renewcommand{\arraystretch}{1.1}
\caption{LLaDA: WildGuard judgments on HarmBench.}
\label{tab:wildguard-attacks-llada}
\begin{tabular}{@{}lrrrrrr@{}}
\toprule
 & \multicolumn{3}{c}{DIJA} & \multicolumn{3}{c}{PAD} \\
\cmidrule(lr){2-4}\cmidrule(lr){5-7}
Method & Harmful\% & Refusal\% & $n$ & Harmful\% & Refusal\% & $n$ \\
\midrule
Unmodified & \shortstack{$96.0_{2.6}$\\{\scriptsize $[92.3,99.0]$}} & \shortstack{$1.7_{0.6}$\\{\scriptsize $[0.0,4.3]$}} & 300 & \shortstack{$80.0_{3.0}$\\{\scriptsize $[73.7,86.0]$}} & \shortstack{$15.3_{1.5}$\\{\scriptsize $[10.3,21.0]$}} & 300 \\
Full SAD & \shortstack{$96.3_{0.6}$\\{\scriptsize $[92.3,99.3]$}} & \shortstack{$2.0_{0.0}$\\{\scriptsize $[0.0,5.0]$}} & 300 & \shortstack{$85.7_{3.8}$\\{\scriptsize $[80.0,90.7]$}} & \shortstack{$11.3_{3.8}$\\{\scriptsize $[6.7,16.3]$}} & 300 \\
SADDER (selective) & \shortstack{$96.7_{0.6}$\\{\scriptsize $[92.7,99.7]$}} & \shortstack{$2.0_{0.0}$\\{\scriptsize $[0.0,5.0]$}} & 300 & \shortstack{$81.0_{2.6}$\\{\scriptsize $[74.3,87.0]$}} & \shortstack{$15.0_{4.0}$\\{\scriptsize $[10.0,20.7]$}} & 300 \\
SADDER + warm-up & \shortstack{$96.7_{1.5}$\\{\scriptsize $[93.0,99.3]$}} & \shortstack{$2.0_{1.0}$\\{\scriptsize $[0.0,4.7]$}} & 300 & \shortstack{$76.3_{0.6}$\\{\scriptsize $[68.7,83.3]$}} & \shortstack{$19.0_{1.7}$\\{\scriptsize $[12.7,26.0]$}} & 300 \\
CSD$^\dagger$ & \shortstack{$97.0_{1.0}$\\{\scriptsize $[93.7,99.3]$}} & \shortstack{$1.7_{0.6}$\\{\scriptsize $[0.0,4.3]$}} & 300 & \shortstack{$77.0_{4.4}$\\{\scriptsize $[70.7,83.0]$}} & \shortstack{$19.7_{2.1}$\\{\scriptsize $[14.0,26.0]$}} & 300 \\
DiffuGuard (full) & \shortstack{$97.7_{2.1}$\\{\scriptsize $[95.3,99.7]$}} & \shortstack{$0.7_{1.2}$\\{\scriptsize $[0.0,1.7]$}} & 300 & \shortstack{$85.0_{4.0}$\\{\scriptsize $[80.0,89.7]$}} & \shortstack{$11.7_{3.5}$\\{\scriptsize $[7.7,16.0]$}} & 300 \\
DiffuGuard (gen) & \shortstack{$97.0_{1.0}$\\{\scriptsize $[93.7,99.7]$}} & \shortstack{$1.0_{0.0}$\\{\scriptsize $[0.0,3.0]$}} & 300 & \shortstack{$78.0_{2.6}$\\{\scriptsize $[71.3,84.3]$}} & \shortstack{$17.3_{1.2}$\\{\scriptsize $[11.7,23.0]$}} & 300 \\
SADDER (top-1) & \shortstack{$97.3_{0.6}$\\{\scriptsize $[94.0,99.7]$}} & \shortstack{$1.3_{0.6}$\\{\scriptsize $[0.0,3.7]$}} & 300 & \shortstack{$81.3_{3.2}$\\{\scriptsize $[74.3,87.3]$}} & \shortstack{$15.3_{2.9}$\\{\scriptsize $[9.7,21.3]$}} & 300 \\
SADDER (aligned) & \shortstack{$96.7_{0.6}$\\{\scriptsize $[92.7,99.7]$}} & \shortstack{$2.0_{0.0}$\\{\scriptsize $[0.0,5.0]$}} & 300 & \shortstack{$83.3_{1.5}$\\{\scriptsize $[77.7,88.7]$}} & \shortstack{$13.0_{2.0}$\\{\scriptsize $[8.3,18.0]$}} & 300 \\
Random selection & \shortstack{$96.3_{1.2}$\\{\scriptsize $[92.3,99.3]$}} & \shortstack{$1.3_{0.6}$\\{\scriptsize $[0.0,3.3]$}} & 300 & \shortstack{$77.3_{3.2}$\\{\scriptsize $[71.7,82.7]$}} & \shortstack{$19.0_{3.5}$\\{\scriptsize $[14.0,24.0]$}} & 300 \\
\bottomrule
\end{tabular}
\end{table}

%% file: tables/extra/wildguard_attacks_dream.tex
\begin{table}[!htbp]
\centering\fontsize{8}{9.6}\selectfont
\setlength{\tabcolsep}{3pt}\renewcommand{\arraystretch}{1.1}
\caption{Dream: WildGuard judgments on HarmBench.}
\label{tab:wildguard-attacks-dream}
\begin{tabular}{@{}lrrrrrr@{}}
\toprule
 & \multicolumn{3}{c}{DIJA} & \multicolumn{3}{c}{PAD} \\
\cmidrule(lr){2-4}\cmidrule(lr){5-7}
Method & Harmful\% & Refusal\% & $n$ & Harmful\% & Refusal\% & $n$ \\
\midrule
Unmodified & \shortstack{$97.3_{0.6}$\\{\scriptsize $[94.0,99.7]$}} & \shortstack{$0.7_{0.6}$\\{\scriptsize $[0.0,2.0]$}} & 300 & \shortstack{$66.0_{2.0}$\\{\scriptsize $[57.7,74.0]$}} & \shortstack{$30.3_{2.5}$\\{\scriptsize $[22.3,38.7]$}} & 300 \\
Full SAD & \shortstack{$97.3_{0.6}$\\{\scriptsize $[94.0,100.0]$}} & \shortstack{$1.0_{0.0}$\\{\scriptsize $[0.0,3.0]$}} & 300 & \shortstack{$66.7_{0.6}$\\{\scriptsize $[58.3,74.7]$}} & \shortstack{$29.3_{0.6}$\\{\scriptsize $[21.3,37.7]$}} & 300 \\
SADDER (selective) & \shortstack{$96.7_{1.5}$\\{\scriptsize $[93.3,99.3]$}} & \shortstack{$1.3_{0.6}$\\{\scriptsize $[0.0,3.7]$}} & 300 & \shortstack{$80.0_{1.7}$\\{\scriptsize $[72.7,86.7]$}} & \shortstack{$16.7_{1.5}$\\{\scriptsize $[10.0,23.7]$}} & 300 \\
SADDER + warm-up & \shortstack{$96.3_{0.6}$\\{\scriptsize $[92.7,99.3]$}} & \shortstack{$1.7_{0.6}$\\{\scriptsize $[0.0,4.3]$}} & 300 & \shortstack{$69.0_{1.7}$\\{\scriptsize $[60.0,77.0]$}} & \shortstack{$28.0_{1.7}$\\{\scriptsize $[20.0,37.0]$}} & 300 \\
CSD$^\dagger$ & \shortstack{$98.0_{0.0}$\\{\scriptsize $[95.0,100.0]$}} & \shortstack{$0.7_{0.6}$\\{\scriptsize $[0.0,2.0]$}} & 300 & \shortstack{$65.3_{1.5}$\\{\scriptsize $[57.0,73.3]$}} & \shortstack{$30.7_{2.5}$\\{\scriptsize $[22.7,39.0]$}} & 300 \\
DiffuGuard (full) & \shortstack{$98.3_{0.6}$\\{\scriptsize $[95.7,100.0]$}} & \shortstack{$0.7_{0.6}$\\{\scriptsize $[0.0,1.7]$}} & 299 & \shortstack{$69.3_{2.1}$\\{\scriptsize $[62.7,75.7]$}} & \shortstack{$29.7_{2.1}$\\{\scriptsize $[23.3,36.3]$}} & 300 \\
DiffuGuard (gen) & \shortstack{$98.3_{0.6}$\\{\scriptsize $[95.7,100.0]$}} & \shortstack{$0.7_{0.6}$\\{\scriptsize $[0.0,2.0]$}} & 300 & \shortstack{$65.0_{2.0}$\\{\scriptsize $[56.7,72.7]$}} & \shortstack{$30.0_{1.0}$\\{\scriptsize $[22.3,38.0]$}} & 300 \\
SADDER (top-1) & \shortstack{$97.7_{0.6}$\\{\scriptsize $[94.3,100.0]$}} & \shortstack{$1.3_{0.6}$\\{\scriptsize $[0.0,3.7]$}} & 300 & \shortstack{$80.0_{1.7}$\\{\scriptsize $[72.7,87.0]$}} & \shortstack{$17.0_{1.0}$\\{\scriptsize $[10.7,24.0]$}} & 300 \\
SADDER (aligned) & \shortstack{$97.7_{0.6}$\\{\scriptsize $[94.3,100.0]$}} & \shortstack{$1.0_{0.0}$\\{\scriptsize $[0.0,3.0]$}} & 300 & \shortstack{$72.0_{1.7}$\\{\scriptsize $[63.3,80.0]$}} & \shortstack{$25.0_{1.0}$\\{\scriptsize $[17.3,33.0]$}} & 300 \\
Random selection & \shortstack{$97.0_{1.0}$\\{\scriptsize $[93.7,99.7]$}} & \shortstack{$1.0_{0.0}$\\{\scriptsize $[0.0,3.0]$}} & 300 & \shortstack{$71.0_{2.0}$\\{\scriptsize $[63.3,78.0]$}} & \shortstack{$27.0_{1.7}$\\{\scriptsize $[20.3,34.3]$}} & 300 \\
\bottomrule
\end{tabular}
\end{table}

%% file: tables/extra/wildguard_controls_paired.tex
\begin{table}[!htbp]
\centering\fontsize{8}{9.6}\selectfont
\setlength{\tabcolsep}{3pt}\renewcommand{\arraystretch}{1.1}
\caption{Paired WildGuard differences on HarmBench (pp).}
\label{tab:wildguard-controls-paired}
\begin{tabular}{@{}lrrr@{}}
\toprule
Comparison & $\Delta$ harmful\% & $\Delta$ refusal\% & Pairs \\
\midrule
\multicolumn{4}{l}{\textbf{LLaDA / DIJA}} \\
Selective vs. full SAD & $0.3\;[-0.7,1.7]$ & $0.0\;[0.0,0.0]$ & 300 \\
Selective vs. random & $0.3\;[-1.7,3.0]$ & $0.7\;[0.0,1.7]$ & 300 \\
Warm-up vs. full SAD & $0.3\;[-0.7,1.7]$ & $-0.0\;[-1.0,1.0]$ & 300 \\
Warm-up vs. random & $0.3\;[-1.3,2.3]$ & $0.7\;[0.0,1.7]$ & 300 \\
\multicolumn{4}{l}{\textbf{LLaDA / PAD}} \\
Selective vs. full SAD & $-4.7\;[-11.0,1.7]$ & $3.7\;[-2.3,9.7]$ & 300 \\
Selective vs. random & $3.7\;[-2.0,9.3]$ & $-4.0\;[-9.3,1.7]$ & 300 \\
Warm-up vs. full SAD & $-9.3\;[-16.7,-2.0]$ & $7.7\;[0.0,15.0]$ & 300 \\
Warm-up vs. random & $-1.0\;[-6.7,4.3]$ & $0.0\;[-5.3,5.7]$ & 300 \\
\multicolumn{4}{l}{\textbf{Dream / DIJA}} \\
Selective vs. full SAD & $-0.7\;[-2.0,0.7]$ & $0.3\;[0.0,1.0]$ & 300 \\
Selective vs. random & $-0.3\;[-1.7,0.7]$ & $0.3\;[0.0,1.0]$ & 300 \\
Warm-up vs. full SAD & $-1.0\;[-2.7,0.0]$ & $0.7\;[0.0,2.0]$ & 300 \\
Warm-up vs. random & $-0.7\;[-2.0,0.7]$ & $0.7\;[0.0,2.0]$ & 300 \\
\multicolumn{4}{l}{\textbf{Dream / PAD}} \\
Selective vs. full SAD & $13.3\;[5.0,22.0]$ & $-12.7\;[-21.3,-4.3]$ & 300 \\
Selective vs. random & $9.0\;[1.3,16.7]$ & $-10.3\;[-18.0,-3.0]$ & 300 \\
Warm-up vs. full SAD & $2.3\;[-7.3,12.0]$ & $-1.3\;[-11.3,8.7]$ & 300 \\
Warm-up vs. random & $-2.0\;[-7.3,3.3]$ & $1.0\;[-4.0,6.0]$ & 300 \\
\bottomrule
\end{tabular}
\end{table}

%% file: tables/extra/judge_views_v4_llada.tex
\begin{table}[!htbp]
\centering\scriptsize
\setlength{\tabcolsep}{3pt}\renewcommand{\arraystretch}{1.12}
\caption{LLaDA: DIJA judge and response-view comparison on 100 held-out HarmBench behaviors, three seeds. Full includes the fixed scaffold; fills excludes it. Rates (\%) are means with seed-SD subscripts. HB/LG have 300 scored outputs per row; WG has 300 except Dream aligned remasking (299).}
\label{tab:judge-views-v4-llada}
\begin{tabular}{@{}lrrrrrr@{}}
\toprule
Method & HB full & HB fills & LG full & LG fills & WG full & WG fills \\
\midrule
Unmodified & $78.7_{3.1}$ & $73.0_{3.6}$ & $82.0_{2.6}$ & $53.3_{2.5}$ & $96.0_{2.6}$ & $96.3_{1.5}$ \\
Full SAD & $77.0_{2.0}$ & $71.7_{2.9}$ & $81.0_{0.0}$ & $56.3_{6.8}$ & $96.3_{0.6}$ & $97.3_{1.2}$ \\
SADDER (selective) & $79.0_{1.0}$ & $75.7_{2.3}$ & $79.3_{1.5}$ & $54.0_{2.6}$ & $96.7_{0.6}$ & $96.3_{1.5}$ \\
SADDER + warm-up & $78.3_{0.6}$ & $71.3_{2.5}$ & $80.7_{0.6}$ & $41.7_{3.5}$ & $96.7_{1.5}$ & $97.3_{2.1}$ \\
CSD$^\dagger$ & $48.7_{4.7}$ & $30.7_{5.1}$ & $79.7_{4.9}$ & $33.7_{4.2}$ & $97.0_{1.0}$ & $96.3_{1.5}$ \\
DiffuGuard (full) & $42.7_{4.5}$ & $32.3_{0.6}$ & $81.7_{2.5}$ & $49.0_{3.6}$ & $97.7_{2.1}$ & $97.0_{0.0}$ \\
DiffuGuard (gen) & $70.7_{3.5}$ & $55.7_{4.0}$ & $83.7_{2.5}$ & $49.0_{5.2}$ & $97.0_{1.0}$ & $97.7_{0.6}$ \\
SADDER (top-1) & $78.7_{3.1}$ & $72.7_{3.5}$ & $80.7_{2.1}$ & $57.0_{2.0}$ & $97.3_{0.6}$ & $96.0_{2.0}$ \\
SADDER (aligned) & $79.3_{1.2}$ & $75.0_{3.5}$ & $80.0_{2.0}$ & $57.3_{0.6}$ & $96.7_{0.6}$ & $96.3_{1.5}$ \\
Random selection & $76.3_{1.2}$ & $72.3_{4.0}$ & $80.3_{5.5}$ & $51.7_{3.1}$ & $96.3_{1.2}$ & $97.3_{1.2}$ \\
\bottomrule
\end{tabular}
\end{table}

%% file: tables/extra/judge_views_v4_dream.tex
\begin{table}[!htbp]
\centering\scriptsize
\setlength{\tabcolsep}{3pt}\renewcommand{\arraystretch}{1.12}
\caption{Dream: DIJA judge and response-view comparison on 100 held-out HarmBench behaviors, three seeds. Full includes the fixed scaffold; fills excludes it. Rates (\%) are means with seed-SD subscripts. HB/LG have 300 scored outputs per row; WG has 300 except Dream aligned remasking (299).}
\label{tab:judge-views-v4-dream}
\begin{tabular}{@{}lrrrrrr@{}}
\toprule
Method & HB full & HB fills & LG full & LG fills & WG full & WG fills \\
\midrule
Unmodified & $82.0_{1.0}$ & $76.0_{3.6}$ & $80.7_{1.5}$ & $48.7_{2.5}$ & $97.3_{0.6}$ & $97.3_{0.6}$ \\
Full SAD & $85.3_{0.6}$ & $80.3_{2.3}$ & $83.0_{1.0}$ & $42.3_{3.1}$ & $97.3_{0.6}$ & $96.3_{1.5}$ \\
SADDER (selective) & $77.7_{1.5}$ & $72.3_{1.5}$ & $81.7_{3.2}$ & $57.3_{2.1}$ & $96.7_{1.5}$ & $97.3_{1.2}$ \\
SADDER + warm-up & $78.7_{2.3}$ & $75.7_{3.1}$ & $79.7_{1.2}$ & $39.7_{1.2}$ & $96.3_{0.6}$ & $96.0_{0.0}$ \\
CSD$^\dagger$ & $68.0_{6.2}$ & $45.0_{3.5}$ & $82.3_{0.6}$ & $41.7_{2.1}$ & $98.0_{0.0}$ & $97.3_{0.6}$ \\
DiffuGuard (full) & $33.3_{2.5}$ & $22.0_{2.6}$ & $84.0_{2.0}$ & $44.7_{3.1}$ & $98.3_{0.6}$ & $98.3_{1.2}$ \\
DiffuGuard (gen) & $69.3_{1.5}$ & $54.3_{1.5}$ & $81.3_{1.5}$ & $55.7_{4.0}$ & $98.3_{0.6}$ & $98.0_{1.0}$ \\
SADDER (top-1) & $78.0_{2.6}$ & $71.0_{1.7}$ & $80.0_{2.6}$ & $58.3_{1.5}$ & $97.7_{0.6}$ & $97.3_{0.6}$ \\
SADDER (aligned) & $85.3_{0.6}$ & $77.7_{3.1}$ & $82.7_{0.6}$ & $55.7_{1.5}$ & $97.7_{0.6}$ & $97.7_{0.6}$ \\
Random selection & $79.3_{1.2}$ & $74.3_{1.5}$ & $82.0_{2.0}$ & $53.7_{4.0}$ & $97.0_{1.0}$ & $97.0_{1.0}$ \\
\bottomrule
\end{tabular}
\end{table}

%% file: tables/extra/judge_views_v4_paired.tex
\begin{table}[!htbp]
\centering\scriptsize
\setlength{\tabcolsep}{3pt}\renewcommand{\arraystretch}{1.12}
\caption{Generated-span DIJA score differences (percentage points), with 95\% prompt-cluster bootstrap intervals. Comparisons use the same prompts and seeds; full SAD and random selection have no warm-up.}
\label{tab:judge-views-v4-paired}
\begin{tabular}{@{}lrrr@{}}
\toprule
Comparison & HB & LG & WG \\
\midrule
\multicolumn{4}{l}{\textbf{LLaDA}} \\
Selective vs Unmodified & $2.7\;[-2.0,7.3]$ & $0.7\;[-6.7,7.7]$ & $0.0\;[-2.3,2.0]$ \\
Selective vs Full SAD & $4.0\;[-1.3,9.3]$ & $-2.3\;[-8.7,3.7]$ & $-1.0\;[-2.7,0.3]$ \\
Selective vs Random selection & $3.3\;[-1.3,7.7]$ & $2.3\;[-2.7,7.3]$ & $-1.0\;[-2.7,0.7]$ \\
Warm-up vs Unmodified & $-1.7\;[-7.7,4.0]$ & $-11.7\;[-19.7,-4.3]$ & $1.0\;[-1.0,3.0]$ \\
Warm-up vs Full SAD & $-0.3\;[-6.0,5.3]$ & $-14.7\;[-21.7,-8.0]$ & $0.0\;[-1.3,1.3]$ \\
Warm-up vs Random selection & $-1.0\;[-6.0,3.7]$ & $-10.0\;[-16.3,-4.0]$ & $0.0\;[-1.3,1.3]$ \\
\multicolumn{4}{l}{\textbf{Dream}} \\
Selective vs Unmodified & $-3.7\;[-10.0,2.3]$ & $8.7\;[1.7,15.0]$ & $0.0\;[-1.0,1.0]$ \\
Selective vs Full SAD & $-8.0\;[-13.3,-2.7]$ & $15.0\;[8.0,22.0]$ & $1.0\;[0.0,2.3]$ \\
Selective vs Random selection & $-2.0\;[-7.3,3.0]$ & $3.7\;[-1.3,9.0]$ & $0.3\;[0.0,1.0]$ \\
Warm-up vs Unmodified & $-0.3\;[-6.7,5.7]$ & $-9.0\;[-16.0,-2.0]$ & $-1.3\;[-3.7,0.3]$ \\
Warm-up vs Full SAD & $-4.7\;[-9.7,0.3]$ & $-2.7\;[-10.3,5.0]$ & $-0.3\;[-2.3,1.3]$ \\
Warm-up vs Random selection & $1.3\;[-3.3,6.3]$ & $-14.0\;[-20.0,-8.0]$ & $-1.0\;[-2.3,0.0]$ \\
\bottomrule
\end{tabular}
\end{table}

%% file: tables/main_uniform.tex
\begin{table}[!htb]
\centering
\footnotesize
\setlength{\tabcolsep}{3pt}
\renewcommand{\arraystretch}{1.15}
\caption{Uniform-state continuation on RTP.}
\label{tab:uniform}
\begin{tabular*}{0.94\textwidth}{@{\extracolsep{\fill}}lrrrrr@{}}
\toprule
\multirow{2}{*}{Method}
  & \multicolumn{2}{c}{\textbf{Duo}}
  & \multicolumn{2}{c}{\textbf{UDLM}}
  & \multirow{2}{*}{NFE} \\
\cmidrule(lr){2-3}\cmidrule(lr){4-5}
  & Unsafe $\%\downarrow$
  & Benign BERT $\uparrow$
  & Unsafe $\%\downarrow$
  & Benign BERT $\uparrow$
  & \\
\midrule
Unmodified & $51.3_{1.8}$ & $81.77_{0.08}$ & $68.5_{2.5}$ & $80.81_{0.04}$ & 128 \\
\midrule
\multicolumn{6}{@{}l}{\textbf{SAD}} \\
\addlinespace[1pt]
$\eta=3{\times}10^{4}$ & $43.4_{3.5}$ & $81.77_{0.14}$ & $57.1_{1.1}$ & $80.50_{0.06}$ & 128 \\
$\eta=5{\times}10^{4}$ & $42.0_{1.4}$ & $81.50_{0.10}$ & $52.9_{1.6}$ & $80.40_{0.09}$ & 128 \\
$\eta=10^{5}$ & $42.8_{2.2}$ & $81.04_{0.07}$ & $47.4_{2.9}$ & $80.49_{0.03}$ & 128 \\
$\eta=1.5{\times}10^{5}$ & $41.1_{2.2}$ & $80.76_{0.22}$ & $49.5_{4.2}$ & $80.53_{0.16}$ & 128 \\
\midrule
\multicolumn{6}{@{}l}{\textbf{SADDER, $k=6$}} \\
\addlinespace[1pt]
$\eta=3{\times}10^{4}$ & $47.9_{2.3}$ & $81.95_{0.04}$ & $56.3_{1.7}$ & $80.67_{0.12}$ & 128 \\
$\eta=5{\times}10^{4}$ & $44.6_{4.5}$ & $81.85_{0.07}$ & $52.2_{1.6}$ & $80.68_{0.04}$ & 128 \\
$\eta=10^{5}$ & $44.2_{1.9}$ & $81.67_{0.14}$ & $46.1_{2.4}$ & $80.80_{0.07}$ & 128 \\
$\eta=1.5{\times}10^{5}$ & $45.2_{3.6}$ & $81.54_{0.13}$ & $45.1_{2.9}$ & $80.89_{0.13}$ & 128 \\
\midrule
\multicolumn{6}{@{}l}{\textbf{SADDER, $k=9$}} \\
\addlinespace[1pt]
$\eta=3{\times}10^{4}$ & $47.1_{4.1}$ & $81.93_{0.07}$ & $55.5_{2.1}$ & $80.65_{0.13}$ & 128 \\
$\eta=5{\times}10^{4}$ & $45.2_{2.3}$ & $81.81_{0.06}$ & $51.5_{3.3}$ & $80.67_{0.04}$ & 128 \\
$\eta=10^{5}$ & $45.1_{2.7}$ & $81.59_{0.17}$ & $46.5_{2.3}$ & $80.81_{0.09}$ & 128 \\
$\eta=1.5{\times}10^{5}$ & $44.1_{1.8}$ & $81.41_{0.15}$ & $44.5_{2.1}$ & $80.90_{0.14}$ & 128 \\
\midrule
\multicolumn{6}{@{}l}{\textbf{SADDER, $k=12$}} \\
\addlinespace[1pt]
$\eta=3{\times}10^{4}$ & $47.2_{3.2}$ & $82.03_{0.08}$ & $57.9_{2.1}$ & $80.66_{0.14}$ & 128 \\
$\eta=5{\times}10^{4}$ & $43.0_{2.5}$ & $81.86_{0.10}$ & $51.9_{3.0}$ & $80.69_{0.02}$ & 128 \\
$\eta=10^{5}$ & $45.0_{0.7}$ & $81.77_{0.16}$ & $48.0_{2.6}$ & $80.81_{0.10}$ & 128 \\
$\eta=1.5{\times}10^{5}$ & $43.6_{2.5}$ & $81.68_{0.10}$ & $45.9_{1.7}$ & $80.92_{0.12}$ & 128 \\
\midrule
\multicolumn{6}{@{}l}{\textbf{CSD}} \\
\addlinespace[1pt]
$k=6$ & $48.3_{2.9}$ & $81.80_{0.11}$ & $53.5_{1.8}$ & $80.51_{0.06}$ & 128 \\
$k=9$ & $35.1_{2.3}$ & $81.93_{0.06}$ & $63.5_{1.5}$ & $80.81_{0.06}$ & 128 \\
$k=12$ & $50.6_{2.1}$ & $81.78_{0.08}$ & $67.1_{4.2}$ & $80.81_{0.04}$ & 128 \\
\midrule
\multicolumn{6}{@{}l}{\textbf{Best-of-$N$}} \\
\addlinespace[1pt]
$N=3$ & $42.2_{2.3}$ & $82.05$ & $70.0_{3.4}$ & $80.80$ & 384 \\
$N=5$ & $37.6_{2.1}$ & $82.10$ & $68.9_{2.1}$ & $80.72$ & 640 \\
$N=8$ & $33.1_{2.3}$ & $82.16$ & $69.7_{2.2}$ & $80.77$ & 1024 \\
\bottomrule
\end{tabular*}
\end{table}